\documentclass[sigconf]{acmart}
\setcitestyle{numbers,sort&compress}

\usepackage{tikz}
\usepackage{filecontents}
\usepackage{footnote}
\usepackage{listings}

\usepackage{subfigure}
\usepackage{hyperref}       
\usepackage{url}            
\usepackage{booktabs}       
\usepackage{microtype}      
\usepackage{graphicx}%
\usepackage{float}
\usepackage{soul}
\usepackage{caption}
\usepackage{multirow}
\usepackage{bbm}
\usepackage{bm}
\usepackage{mathrsfs}
\usepackage{xcolor}
\usepackage{xspace}
\usetikzlibrary{trees, shapes, positioning}
\usepackage[framemethod=tikz]{mdframed}

\usepackage{thmtools,thm-restate}
\usepackage{dsfont}
\usepackage[ruled,vlined]{algorithm2e}
\usepackage{algpseudocode}  
\usepackage{pifont}
\usepackage[flushleft]{threeparttable}
\usepackage{makecell}
\usepackage{lipsum}
\usepackage{balance}
\usepackage{amsmath}
\usepackage{mathtools}
\usepackage{amsthm}
\usepackage{tabularx}
\usepackage{forest}
\usepackage{enumitem}
\setlist{noitemsep}
\usepackage{totcount}
\regtotcounter{section}

\definecolor{firebrick}{rgb}{0.7, 0.13, 0.13}
\definecolor{darkblue}{rgb}{0,0,0.55}
\definecolor{grey}{rgb}{0.8,0.8,0.8}



\usepackage{framed}

\definecolor{formalshade}{rgb}{0.95,0.95,0.97}
\definecolor{skyblue}{rgb}{0.529, 0.808, 0.922}
\definecolor{deepblue}{rgb}{0.14,0.22,0.52}
\usepackage[most]{tcolorbox}

\newtcolorbox{formal}[2][]{
  fontupper=\small,
  colback=formalshade,         
  colframe=skyblue!60!deepblue, 
  coltitle=deepblue,           
  colbacktitle=skyblue!15!white,
  fonttitle=\bfseries,
  left=2mm, right=2mm, top=1mm, bottom=1mm,
  sharp corners=west,           
  rounded corners=east,         
  boxrule=1pt,
  borderline west={2pt}{0pt}{deepblue}, 
  enhanced, 
  breakable,
  title=#2,
  #1
}

\definecolor{formal2shade}{rgb}{1.00,0.99,0.94} 
\definecolor{sunyellow}{rgb}{1.00,0.93,0.60}    
\definecolor{deepyellow}{rgb}{0.90,0.75,0.40} 

\newtcolorbox{formal2}[2][]{
  colback=formal2shade,   
  opacityback=0.96,
  colframe=sunyellow!50!deepyellow,
  coltitle=deepyellow,
  colbacktitle=sunyellow!10!white,
  fonttitle=\bfseries,
  left=2mm, right=2mm, top=1mm, bottom=1mm,
  sharp corners=west,
  rounded corners=east,
  boxrule=0.7pt,
  borderline west={2pt}{0pt}{deepyellow},
  enhanced,
  breakable,
  title=#2,
  #1
}

\DeclareRobustCommand{\pie}[1]{%
  \tikz[baseline=-0.5ex]{
    \draw (0,0) circle (1ex);
    \fill (0,-1ex) arc (-90:(#1-90):1ex) -- (0,-1ex) -- cycle;
  }%
}

\newcommand{\MethodName}{un\-learn\-abil\-ity\xspace}
\newcommand{\cmark}{\ding{51}}%
\newcommand{\xmark}{\ding{55}}%
\newcommand{\questionbullet}{\includegraphics[width=1em]{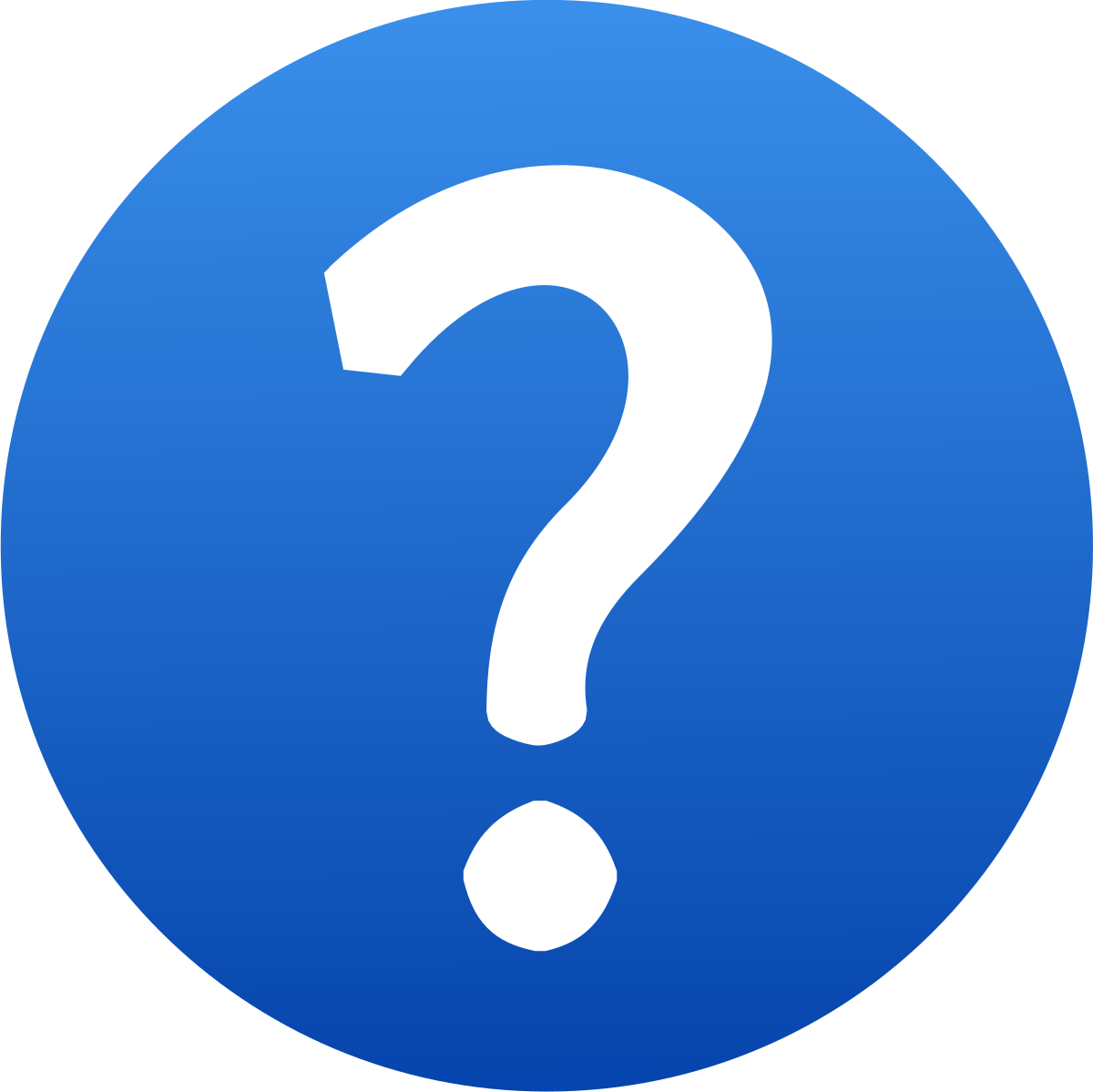}}

\def\eg{\emph{e.g.,}\xspace}

\def\ie{\emph{i.e.,}\xspace}

\def\eqref#1{equation~\ref{#1}}

\def\1{\bm{1}}

\DeclareMathAlphabet{\mathsfit}{\encodingdefault}{\sfdefault}{m}{sl}
\SetMathAlphabet{\mathsfit}{bold}{\encodingdefault}{\sfdefault}{bx}{n}

\def\gD{{\mathcal{D}}}

\def\gL{{\mathcal{L}}}
\def\gM{{\mathcal{M}}}

\def\gU{{\mathcal{U}}}

\def\sD{{\mathbb{D}}}

\newcommand{\N}{\mathcal{N}}

\def\eg{\emph{e.g.,}\xspace}

\def\ie{\emph{i.e.,}\xspace}

\DeclareUnicodeCharacter{0301}{\hspace{-1ex}\'{ }}
\graphicspath{{./figures/}}

\setcopyright{none}
\renewcommand\footnotetextcopyrightpermission[1]{}

\begin{document}

\title{SoK: Unlearnability and Unlearning for Model Dememorization}

\author{Mengying Zhang}
\authornote{Both authors contributed equally to this research.}
\affiliation{%
  \institution{RMIT University}
  \country{Australia}
}

\author{Derui Wang}
\authornotemark[1]
\affiliation{%
  \institution{CSIRO}
  \country{Australia}
}

\author{Ruoxi Sun}
\affiliation{%
  \institution{CSIRO}
  \country{Australia}
}

\author{Xiaoyu Xia}
\affiliation{%
  \institution{RMIT University}
  \country{Australia}
}

\author{Shuang Hao}
\affiliation{%
  \institution{University of Texas at Dallas}
  \country{USA}
}

\author{Minhui Xue}
\affiliation{%
  \institution{CSIRO and Adelaide University}
  \country{Australia}
}

\renewcommand{\shortauthors}{Trovato et al.}
\begin{abstract}
Advanced model dememorization methods, including availability poisoning (\MethodName) and machine unlearning, are emerging as key safeguards against data misuse in machine learning (ML).
At the training stage, \MethodName embeds imperceptible perturbations into data before release to reduce learnability. 
At the post-training stage, unlearning removes previously acquired information from models to prevent unauthorized disclosure or use.
While both defenses aim to preserve the right to withhold knowledge, their vulnerabilities and shared foundations remain unclear.
Specifically, both \MethodName and unlearning suffer from issues such as \textit{shallow dememorization}, leading to falsely claimed data learnability reduction or forgetting in the presence of weight perturbations.
Moreover, input perturbations may affect the effectiveness of downstream unlearning, while unlearning may inadvertently recover domain knowledge hidden by \MethodName.
This interplay calls for deeper investigation.
Finally, there is a lack of formal guarantees to provide theoretical insights into current defenses against shallow dememorization.
In this Systematization of Knowledge, we present the first integrated analysis of model dememorization approaches leveraging \MethodName and unlearning.
Our contributions are threefold:
\textit{(i)} a unified taxonomy of \MethodName and scalable unlearning methods;
\textit{(ii)} an empirical evaluation revealing the robustness, interplay, and shallow dememorization of leading methods; and
\textit{(iii)} the first theoretical guarantee on dememorization depth for models processed through certified unlearning.
These results lay the foundation for unifying dememorization mechanisms across the ML lifecycle to achieve a deeper \textit{immemor} state for sensitive knowledge.
\end{abstract}

\maketitle
\pagestyle{plain}

\section{Introduction}
Publicly available data fuels the rapid progress of machine learning (ML) models. From early breakthroughs in deep neural networks (DNNs) to recent advances in large language models (LLMs) and large vision models (LVMs), abundant training data has been a key driver of success~\cite{achiam2023gpt,guo2025deepseek,esser2024scaling,jiang2023mistral}. 
Yet a pressing need is emerging to safeguard sensitive knowledge from being learned illicitly and to erase mislearned knowledge from trained models.
Regulations across jurisdictions increasingly recognize this. 
The EU's GDPR~\cite{gdpr_art17} and AI Act~\cite{eu_ai_act_2024} enforce data minimization and a ``right to be forgotten'', while the United States' CCPA~\cite{ccpa_doj}/CPRA~\cite{cpra_cppa_regs} grant deletion and opt-out rights, supplemented by the NIST AI Risk Management Framework. 
These regimes highlight a common principle that knowledge associated with training data must be protected at release and deletable post hoc.
To address this principle, censorship can be applied both in the input and the output of an ML model as a control mechanism.
However, traditional censorship methods, such as content blocking, fall short when confronted with the complexity of high-dimensional data and large non-convex models~\cite{glukhov2024llm,ablove2026characterizing}. 
This fact directly motivates the need for advanced dememorization mechanisms such as data availability poison (referred to as \textit{\MethodName} in this paper) and machine unlearning.
In this Systematization of Knowledge (SoK), we examine how these advanced dememorization methods can jointly safeguard data-associated knowledge.

\noindent\textbf{Unlearnability.~}
Motivated by the challenges posed by the learnability of public data, a line of work has emerged to render data unlearnable for ML models. 
Early efforts incorporated adversarial examples to disrupt learning~\cite{fowl2021adversarial,yuan2021neural}, while subsequent approaches introduced training-error-minimizing noise to conceal knowledge within datasets~\cite{huang2021unlearnable,liu2024stable,wang2025provably}. 
Concurrently, shortcut learning has been exploited for constructing unlearnable datasets~\cite{yu2022availability,wu2023one}.
Throughout this paper, we refer to these methods collectively as \emph{\MethodName} methods. 
The central idea is to drive models toward states that generalize poorly on clean data. 
However, these approaches are often fragile under commonly used learning algorithms and data augmentation techniques, and lack strong theoretical guarantees for quantifying data learnability.
A recent line of work proposes certifiable frameworks for learnability~\cite{wang2025provably}, deriving quantile upper bounds (\ie certified learnability) on achievable clean-domain utility for models trained on perturbed data within a specified hypothesis subspace. 
Nonetheless, key challenges remain. 
First, the certification may be loose, with certified learnability significantly underestimating achievable utility. 
Second, beyond identifying certifiable model sets, it remains unclear how to guide training dynamics to converge to these sets without explicit control over the training process.

\noindent\textbf{Unlearning.~}
Machine unlearning can be performed either exactly or approximately. 
Exact unlearning typically requires retraining models from scratch~\cite{bourtoule2021machine}, making it computationally expensive and often impractical in real-world settings.
In contrast, approximate unlearning aims to remove learned information more efficiently through techniques such as fine-tuning~\cite{warnecke2023machine,golatkar2020eternal}, pruning~\cite{jia2023model}, or direct weight modification~\cite{cheng2023gnndelete,jia2023model}. 
Some approximate methods further provide certification guarantees, ensuring that the unlearned model is statistically indistinguishable from one never trained on the deleted data, typically via Newton updates combined with noisy fine-tuning~\cite{guo2020certified,chien2024certified,zhang2024towards,koloskova2025certified}. 
Nevertheless, unlearning remains vulnerable, as seemingly forgotten content can be recovered or unintentionally reproduced~\cite{shumailov2024ununlearning,hu2025unlearning,yoon2026rethinking}.

\noindent\textbf{Challenges and motivations.~}
The primary motivation of this study is to verify the \textit{immemor} of a model with respect to restricted knowledge, rather than achieving a shallow form of dememorization that can be easily circumvented. 
Since knowledge memorization is encoded in the learned parameters, it is natural to link the refutation of memorization to analysis in parameter space. 
However, several non-negligible challenges must be addressed before a reliable verification framework for models ``not learning'' or ``forgetting'' becomes tangible.

First, both \MethodName and unlearning suffer from robustness issues under adaptive attacks and even benign uncertainties.
In particular, there are no formal guarantees that models do not implicitly retain knowledge from \MethodName-protected data, nor that the dememorization achieved through unlearning is genuine. 
Second, the interplay between \MethodName and unlearning, crucial to the effectiveness and robustness of dememorization, remains underexplored.
As an upstream intervention in the ML lifecycle, \MethodName can influence downstream unlearning, as models trained on unlearnable data may exhibit distinct properties.
This concern is reinforced by recent findings on the difficulty of unlearning poisoned samples or residual knowledge in large models~\cite{pawelczyk2025machine,hsu2025unseen,yoon2026rethinking}, highlighting the need to understand how different unlearning methods respond to \MethodName approaches and their degree of compatibility. 
Third, both \MethodName and unlearning may yield only shallow dememorization, where success is measured by metrics tied to the current model state, without accounting for the potential recovery of memorized knowledge in nearby parameter regions. 
Such shallow dememorization contradicts the goal of achieving deep \textit{immemor} state (\ie a truly unremembering state) and therefore warrants further theoretical and empirical investigation.
Guided by these challenges, we identify three central research questions (RQs):
\begin{itemize}[leftmargin=*]
    \item[\questionbullet] \textbf{RQ1}: What is the current research landscape of \MethodName and unlearning methods?
    \item[\questionbullet] \textbf{RQ2}: What synergistic effects arise between \MethodName and unlearning approaches toward dememorization?
    \item[\questionbullet] \textbf{RQ3}: How can we provide robustness guarantees against shallow \MethodName and shallow unlearning?
\end{itemize}

We address these questions as follows. 
First, we present a systematic review of \MethodName and unlearning methods to organize the research landscape. 
We unify both paradigms within a model dememorization taxonomy and reveal their interconnections through empirical analysis and theoretical insights.
We then assess whether an ML model is \textit{immemor} with respect to censored knowledge using memory recovery mechanisms. 
Next, we derive formal guarantees against shallow dememorization by bounding model recoverability under weight perturbations. 
Finally, we show that these bounds extend to models statistically indistinguishable from the \textit{immemor} model, enabling efficient application to models obtained via certified unlearning. 
Together, these findings provide a holistic perspective on \MethodName and unlearning, and offer guidance for future research.

\noindent\textbf{Contributions.~}
The core contributions of this SoK paper are as follows:
\begin{itemize}[leftmargin=*]
    \item We systematically review 187 works on model dememorization based on \MethodName and unlearning, providing a comprehensive overview of the research landscape (RQ1).
    \item Through large-scale experiments in vision and text domains, we show that both \MethodName and unlearning methods are prone to shallow dememorization (RQ2).
    \item We establish non-vacuous formal guarantees on dememorization depth under certified unlearning, providing a principled foundation for achieving a deeper \textit{immemor} state (RQ3).
\end{itemize}
An overview of the model dememorization framework is presented in Figure~\ref{fig_overview}.

\begin{figure*}[t]
    \centering
    \includegraphics[width=\linewidth]{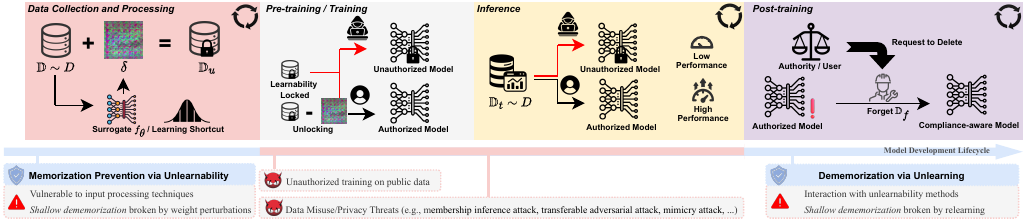}
    \caption{An overview of the model dememorization framework within the ML model development lifecycle. \MethodName and unlearning safeguard the upstream and downstream stages, respectively, yet each faces distinct robustness challenges, and their interaction remains insufficiently understood.}
    \label{fig_overview}
\end{figure*}
\section{Preliminaries}\label{sec:preliminaries}

\subsection{Concepts and Notations}
We begin by reviewing the key concepts and notations underlying both \MethodName and unlearning.

\subsubsection{\MethodName}
Given a protectee dataset $\sD\sim\gD$ sampled from a domain $\gD$, the defender computes a set of perturbations $\delta:=\{\delta_i\}_{i=1}^{|\sD|}$ and construct an unlearnable dataset $\sD_u:=\{\sD_i \oplus \delta_i\}_{i=1}^{|\sD|}$.
Here, $\oplus$ denotes a perturbation operator. 
As most \MethodName methods employ additive perturbations, $\oplus$ typically corresponds to addition.

Perturbative \MethodName methods aim to render $\sD_u$ unlearnable to unauthorized models and learning algorithms in the wild.
Formally, let $\gM(\cdot)$ denote a utility metric and $\sD_t \sim \gD$ a test dataset drawn from the same domain.
A stochastic learning algorithm $\Gamma(\cdot)$ trains a $\theta$-parameterized model $f_{\theta}$ on $\sD_u$, yielding parameters $\theta \in \Theta$, where $\Theta$ denotes the parameter space.
With high probability $1-\alpha$,
\begin{equation}\label{eq:learnability}
    \begin{aligned} 
        & \Pr[\gM(f_{\theta}; \sD_t) \leq \tau] \geq 1-\alpha \\
        & \text{s.t.}\, \theta=\Gamma(\sD_u),
    \end{aligned} 
\end{equation} 
where $\tau$ is a small scalar upper bounding the utility. 
The defender aims to minimize $\alpha$ and $\tau$ across all $\Gamma(\cdot)$, thereby enforcing strong unlearnability of $\sD_u$.

These methods share similarities with clean-label poisoning attacks~\cite{shafahi2018poison}. 
However, while clean-label poisoning targets instance-level manipulation, \MethodName methods focuses on protecting datasets against unauthorized learners.

\subsubsection{Unlearning}
Unlearning can be formally defined in terms of $(\epsilon, \zeta)$-unlearning.
In this formulation, the desideratum is captured through the \emph{indistinguishability} of model parameters.
\begin{restatable}
[\textit{Indistinguishability between models}]{definition}{indistinguishability}\label{def:indistinguishability}
Given a model $f_{\theta}$ parameterized by $\theta$, where $\theta$ is learned from a dataset $\sD$ via a stochastic learning algorithm $\theta = \Gamma(\sD)$, an unlearning algorithm $\gU$ aims to remove the information associated with a forgetting dataset $\sD_f$ from $\theta$, as if $\sD_f$ had never been used in training.
Let $f_{\hat{\theta}}$ denote the model trained on the retain set, where $\hat{\theta} = \Gamma(\sD \setminus \sD_f)$. 
The models $f_{\theta}$ and $f_{\hat{\theta}}$ are said to be $(\epsilon,\zeta)$--indistinguishable if, for any subset $S$ in the parameter space:
\begin{equation}\label{eq:indistinguishability}
    \left\{
    \begin{aligned}
       & \Pr[ \gU(\Gamma(\sD), \sD, \sD_f) \in S] \leq e^{\epsilon} \Pr[ \Gamma(\sD \setminus \sD_{f}) \in S] + \zeta, \\\nonumber
       & \Pr[ \Gamma(\sD \setminus \sD_{f}) \in S]  \leq e^{\epsilon} \Pr[ \gU(\Gamma(\sD), \sD, \sD_f) \in S] + \zeta.
    \end{aligned}
    \right.
\end{equation}
\end{restatable}
\noindent
When $\epsilon = 0$ and $\zeta = 0$, $(\epsilon, \zeta)$-unlearning reduces to exact unlearning, as $\Pr[ \gU(\Gamma(\sD), \sD, \sD_f) \in S] = \Pr[ \Gamma(\sD \setminus \sD_{f}) \in S]$.

In practice, unlearning can be achieved either by retraining the model on $\sD \setminus \sD_f$ or by approximating this effect through updates to $\theta$.
For non-convex models such as DNNs, certified unlearning can be obtained via noisy fine-tuning on the retain set $\sD_r := \sD \setminus \sD_f$, which provides differential privacy (DP) guarantees for forgetting $\sD_f$.

\begin{figure}[t]
    \centering
    \includegraphics[width=.655\linewidth]{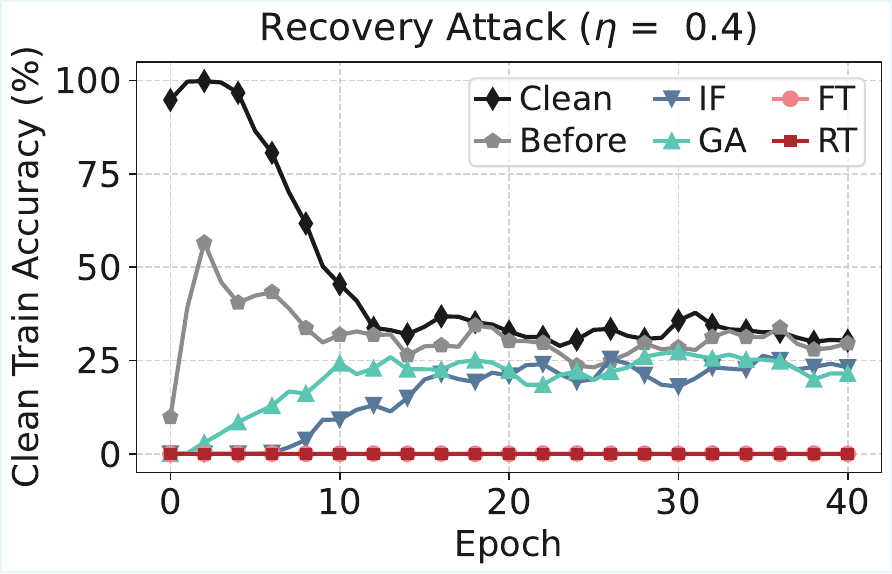}
    \caption{Shallow \MethodName and shallow unlearning revealed by recovery attacks under an $\ell_2$ parametric perturbation magnitude ($\eta$) of $0.4$.
    ``Clean'': model trained on clean CIFAR-10.
    ``Before'': pre-unlearning mode trained with class 0 replaced by UE-s~\cite{huang2021unlearnable}. 
    ``IF'', ``GA'', ``FT'', and ``RT'' represent the unlearned models obtained by applying influence function unlearning~\cite{izzo2021approximate, koh2017understanding}, gradient ascent~\cite{thudi2022unrolling, graves2021amnesiac}, fine-tuning~\cite{golatkar2020eternal}, and retraining methods, respectively, to unlearn the UE-s of Class 0. 
    We measure the accuracy on the corresponding clean CIFAR-10 samples of Class 0 during recovery to assess knowledge recovery under hints.  
    }

    \label{fig:Recovery_UE_eta04}
\end{figure}

\subsection{Toward Deeper \textit{Immemor}}
Both \MethodName and unlearning are developed toward the resection of knowledge from the memory of ML models.
Despite significant progress in each line of work, it remains unclear how interventions at different stages can be unified within a common analytical framework, and how to rigorously verify that undesired knowledge is not retained.

Current evaluation practices predominantly assess performance on data drawn from the corresponding domain, showing degraded task accuracy as evidence of successful forgetting. 
However, the quality of forgetting should not be judged solely by the performance of the model in its current state, but also by its ability to relearn with minimal guidance during a recovery process.
In practice, evaluating only current performance often results in shallow \MethodName or shallow unlearning.

To illustrate this, we take models trained on \MethodName data, apply standard unlearning methods to remove target knowledge, and then perform recovery attacks. 
Results (Figure~\ref{fig:Recovery_UE_eta04}) reveal several insights. 
First, models trained on UE-s (“Before”) exhibit low accuracy on clean data, indicating that \MethodName protection is effective. 
Second, the performance gap before and after unlearning confirms that unlearning can remove information leaked from \MethodName-protected data.
However, under recovery attacks, all methods except fine-tuning deviate significantly from the retraining baseline, suggesting that many unlearning approaches remain shallow and vulnerable to recovery.
These observations highlight that unlearning should be evaluated not only at a single post-unlearning state, but also across nearby parameter states.
Since both \MethodName and unlearning share this underlying principle, we advocate for a unified analysis of these paradigms in parameter space.
\subsection{Threat Model for Dememorization Framework}
\noindent\textbf{Defender.~}
In the \MethodName setting, the defender is responsible for applying \MethodName perturbations to the data prior to release. 
The defender’s objective is to minimize the utility upper bound $\tau$ over the protected dataset $\sD_u$.
In practice, $\tau$ can typically only be evaluated empirically over finite models and learning algorithms, as the functional spaces of $\Theta$ and $f_{\theta}$ are unknown. 
Nevertheless, the defender may seek to certify the learnability of \MethodName-protected data by deriving bounds on $\tau$ that hold with probability $1-\alpha$. 
Downstream, since \MethodName may not fully suppress all domain knowledge associated with $\sD$, the defender may require, upon detecting misuse, that unauthorized learners remove any information of $\sD$ leaked through $\sD_u$.

\noindent\textbf{Adversary.~}
The adversary aims to exploit the released data by modeling its distribution through training machine learning models. 
Specifically, the adversary collects $\sD_u$, or a subset thereof, to train a model $f_{\theta}$ without authorization. 
The adversary has full control over the model architecture, initialization, and training procedure, and may employ standard techniques such as data augmentation. 
However, the presence of \MethodName perturbations may degrade performance. 
To mitigate this, the adversary may apply unlearning techniques to remove the influence of \MethodName-protected data and recover performance.
This makes it critical to design \MethodName methods that are unlearning-aware.



\newcolumntype{P}[1]{>{\RaggedRight\arraybackslash}p{#1}}

\definecolor{root-color}{HTML}{FFE6CC}
\definecolor{pdlc-color}{HTML}{F8CECC}
\definecolor{unlearning-color}{HTML}{E1D5E7}
\definecolor{head-color}{HTML}{FFFFFF}
\definecolor{head-color2}{HTML}{FFFFF0}
\definecolor{method-color}{HTML}{DAE8FC}
\definecolor{line-color}{HTML}{4B74B2}
\definecolor{edge-color}{HTML}{7F7F7F}

\begin{figure*}[t]
\centering
\resizebox{\linewidth}{!}{
\begin{forest}
for tree={
    grow=east,
    reversed=true,
    anchor=base west,
    parent anchor=east,
    child anchor=west,
    base=center,
    font=\large,
    rectangle,
    draw=line-color,
    rounded corners,
    align=center,
    text centered,
    minimum width=5em,
    edge+={edge-color, line width=1.5pt},
    s sep=3pt,
    inner xsep=2pt,
    inner ysep=3pt,
    line width=.75pt,
},
where level=1{
    draw=line-color,
    text width=12em,
    font=\normalsize,
}{},
where level=2{
    draw=line-color,
    text width=15em,
    font=\normalsize,
}{},
where level=3{
    draw=line-color,
    minimum width=15em,
    font=\normalsize,
}{},
where level=4{
    draw=line-color,
    minimum width=55em,
    font=\normalsize,
}{},
[, phantom,
    [, phantom,
        [\textbf{Stage}, for tree={fill=head-color},
            [\textbf{Category}, for tree={fill=head-color},
                [\textbf{Mechanism}, for tree={fill=head-color}]
            ]
        ]
    ][\textbf{Model}\\\textbf{Dememorization}, for tree={fill=root-color},
        [\textbf{Unlearnability}\\\textit{Upstream Prevention}, for tree={fill=pdlc-color},
            [\textbf{Error}, for tree={fill=pdlc-color},
                [\textbf{E-Max}\\
                \cite{shen2019tensorclog,huang2020metapoison,yuan2021neural,fowl2021adversarial,zeng2023narcissus,van2023anti,wang2024efficient,liu2024countering,liu2024metacloak,liu2024game,zhao2024unlearnable}, for tree={fill=method-color}]
                [\textbf{E-Min}\\
                \cite{huang2021unlearnable,liu2021going,tao2021better,zhang2023unlearnable,fu2022robust,wen2023adversarial,fang2024re,ren2023transferable,zhang2025segue,chen2023self,liu2024stable,he2024sharpness,wang2025provably}\\
                \cite{peng2022learnability,sun2024medical,zhang2024mitigating,zhang2025safespeech,meerza2024harmonycloak,li2023make,liu2024multimodal,sun2024unseg,ye2024ungeneralizable,wang2024efficient,wu2025temporal,wu2025myopia,gong2025armor,ma2025t2ue,li2025versatile,liu2026perturbation,li2026priors}, for tree={fill=method-color}]
            ]
            [\textbf{Distribution}, for tree={fill=pdlc-color},
                [\textbf{Shortcut}\\
                \cite{sandoval2022autoregressive,yu2022availability,sadasivan2023cuda,gokul2024poscuda,wu2023one,liu2026expshield,shahid2025towards,wang2024unlearnable}, for tree={fill=method-color}] 
            ]
            [\textbf{Representation}, for tree={fill=pdlc-color},
                [\textbf{Feature Dissimilarization}\\
                \cite{wen2023adversarial,chen2024one,meng2024semantic}, for tree={fill=method-color}]
                [\textbf{Feature Collision}\\
                \cite{zhu2019transferable,wen2023adversarial,shan2020fawkes,shan2023glaze,shan2024nightshade,zhu2026unlearnable}, for tree={fill=method-color}]
            ]
        ]
        [\textbf{Unlearning}\\\textit{Downstream Intervention}, for tree={fill=unlearning-color},
            [\textbf{Exact}\\\textbf{Unlearning}, for tree={fill=unlearning-color},
                [\textbf{Retraining}\\
                \cite{bourtoule2021machine,aldaghri2021coded,yan2022arcane,ullah2023adaptive,dukler2023safe,chowdhury2025towards,xia2025edge,yu2025split,quan2025efficient}, for tree={fill=method-color}]
            ]
            [\textbf{Approximate}\\\textbf{Unlearning}, for tree={fill=unlearning-color},
                [\textbf{Non-certifiable}\\
                \cite{graves2021amnesiac,thudi2022unrolling,yao2024large,sepahvand2025selective,kim2025unlearning,zhu2026decoupling,zhao2026don,tang2026sharpness,di2026unlearning,cha2025towards,alberti2025data,liao2026explainable,li2026llm,cheng2026machine,liu2026randomized,asif2026ofmu,golatkar2020eternal,lu2022quark,tarun2023deep,kurmanji2023towards,yoon2024few,shen2024label,li2024machine,huang2024unified,li2024single,park2024direct,zhang2024negative,zhuoyi2025adversarial}\\
                \cite{di2025adversarial,georgiev2025attribute,chen2025score,feng2025controllable,gao2025large,ding2025unified,wang2025llm,khalil2025not,zhou2025decoupled,chen2023boundary,block2025machine,kawamura2025approximate,entesari2025constrained,khalil2025coun,lee2025distillation,zhong2025dualoptim,zhou2025efficient,sui2025elastic,hu2025falcon,siddiqui2025dormant,ren2025keeping,shen2025llm,wang2025machine,lee2025ruago,fan2025simplicity,hsu2025unseen,tan2025wisdom,zade2026attention,chen2026clue,kim2026cooccurring,lee2026continual,lang2026downgrade,yang2026erase,yu2026falw,deng2026forget,li2026knowledge,di2026label}\\
                \cite{newatia2026mitigating,cheng2026remaining,gu2026towards,facchiano2026video,scholten2026model,gong2025trajdeleter,ye2025reinforcement,bui2026cure,cheng2023gnndelete,jia2023model,jia2024wagle,fan2024salun,buarbulescu2024each,biswas2025cure,dong2025machine,schoepf2026redirection,rezaeirevisiting,zhang2025rule,zaradoukas2026reinforcement,zhang2024defensive,li2023ultrare,zhao2024makes,li2025towards,ji2024reversing,zhong2025unlearning,spartalis2025lotus,chen2022graph,wang2023inductive,pawelczyk2024context,liu2024large,wang2026dragon,chen2026robust,zhu2025llm}, for tree={fill=method-color}]
                [\textbf{Convex-certifiable}\\
                \cite{guo2020certified,ullah2021machine,neel2021descent,nguyen2020variational,sekhari2021remember,warnecke2023machine,gupta2021adaptive,liu2023certified,ahmed2025towards,basaran2025certified,chien2024certified,pandey2025gaussian,allouah2025utility,allouah2026distributional}, for tree={fill=method-color}]
                [\textbf{Non-convex-certifiable}\\
                \cite{zhang2024towards,wei2025provable,chien2024langevin,mehta2022deep,golatkar2021mixed,koloskova2025certified,mu2025rewind,qiao2025hessian,yi2025scalable,zhang2022prompt,yang2025feature}, for tree={fill=method-color}]
            ]
            ]
        ]
    ]
]
\end{forest}
}
\caption{The taxonomy of model dememorization.}
\label{fig:taxonomy}
\end{figure*}
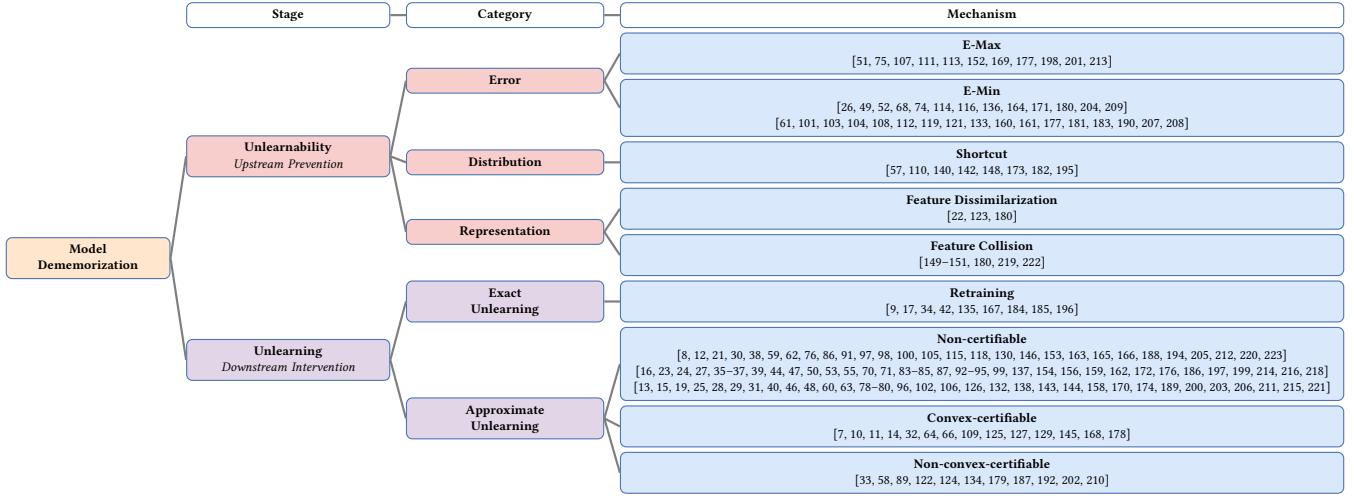
\begin{table*}[t]
    \centering
    \caption{Summary of the properties of major \MethodName methods (ranked within each category by year). \pie{360} indicates that a method is specifically designed to achieve the property, \pie{180} denotes partial capability without deliberate design for it, and \pie{0} signifies that the property does not apply. FR-Oriented refers to methods aimed at robustness in the feature space (\eg against adversarial data augmentation or adversarial training), whereas PR-Oriented refers to methods designed to withstand parametric attacks, in which attackers directly manipulate model parameters. Scalability refers to the largest dataset to which a method can be applied.}
    \label{tab:summary}
    \resizebox{\linewidth}{!}{%
    \begin{tabular}{ccrcccccccc}
    \toprule[1pt]
    \midrule[0.3pt]
             &           Perturbation Type  &  Method                                       & Gradient-Free &   Label-Free   &    Task        &    Data Modality     & Sample-/Class-Wise &   FR-Oriented   &  PR-Oriented  &  Scalability \\
\midrule[0.3pt]
\multirow{39}{*}{\rotatebox[origin=c]{90}{Error}}
&       \multirow{10}{*}{Error-maximizing}  & TC (2019)~\cite{shen2019tensorclog}              &  \pie{0}   &     \pie{0}    & Classification &       Image          &         S          &      \pie{0}     &   \pie{0}    & ImageNet \\
&                                           & Metapoison (2020)~\cite{huang2020metapoison}     &  \pie{0}   &     \pie{0}    & Classification &       Image          &         S          &      \pie{0}     &   \pie{180}  & CIFAR-10 \\
&                                           & NTGA (2021)~\cite{yuan2021neural}                &  \pie{0}   &     \pie{0}    & Classification &       Image          &         S          &      \pie{180}   &   \pie{0}    & ImageNet \\
&                                           & TAP (2021)~\cite{fowl2021adversarial}            &  \pie{0}   &     \pie{0}    & Classification &       Image          &         S          &      \pie{180}   &   \pie{0}    & ImageNet \\
&                                           & Anti-DreamBooth (2023)~\cite{van2023anti}        &  \pie{0}   &        -       &  Text-to-Image &       Image          &         S          &      \pie{180}   &   \pie{0}    & CelebA-HQ \\
&                                           & AAP (2024)~\cite{wang2024efficient}              &  \pie{0}   &     \pie{0}    & Classification &       Image          &         S          &      \pie{180}   &   \pie{0}    & ImageNet \\
&                                           & InMark (2024)~\cite{liu2024countering}           &  \pie{0}   &        -       & Text-to-Image  &       Image          &         S          &      \pie{360}   &   \pie{0}    & WikiArt / VGGFace2 \\
&                                           & EUDP (2024)~\cite{zhao2024unlearnable}           &  \pie{0}   &        -       &Image Generation&       Image          &         S          &      \pie{180}   &   \pie{0}    & CIFAR-10 \\
&                                           & MetaCloak (2024)~\cite{liu2024metacloak}         &  \pie{0}   &        -       &  Text-to-Image &       Image          &         S          &      \pie{360}   &   \pie{0}    & CelebA-HQ \\
&                                           & GUE (2024)~\cite{liu2024game}                    &  \pie{0}   &     \pie{0}    & Classification &       Image          &         S          &      \pie{180}   &   \pie{0}    & CIFAR-10/100 \\
\cmidrule(lr){2-11}
&       \multirow{29}{*}{Error-minimizing}  & UE (2021)~\cite{huang2021unlearnable}            &  \pie{0}   &     \pie{0}    & Classification &       Image          &       S / C        &      \pie{180}   &   \pie{0}    & ImageNet \\
&                                           & GrayAugs (2021)~\cite{liu2021going}              &  \pie{0}   &     \pie{0}    & Classification &       Image          &       S / C        &      \pie{360}   &   \pie{0}    & ImageNet \\
&                                           & Hypocritical (2021)~\cite{tao2021better}         &  \pie{0}   &     \pie{0}    & Classification &       Image          &       S / C        &      \pie{180}   &   \pie{0}    & ImageNet \\
&                                           & RUE (2022)~\cite{fu2022robust}                   &  \pie{0}   &     \pie{0}    & Classification &       Image          &         S          &      \pie{360}   &   \pie{0}    & ImageNet \\
&                                           & LearnabilityLock (2022)~\cite{peng2022learnability}&  \pie{0}   &     \pie{0}    & Classification &       Image          &         C          &      \pie{180}   &   \pie{0}    & ImageNet \\
&                                           & UC (2023)~\cite{zhang2023unlearnable}            &  \pie{0}   &   \pie{360}    & Classification &       Image          &         -          &      \pie{180}   &   \pie{0}    & ImageNet \\
&                                           & TUE (2023)~\cite{ren2023transferable}            &  \pie{0}   &     \pie{0}    & Classification &       Image          &       S / C        &      \pie{180}   &   \pie{0}    & CIFAR-10/100 \\
&                                           & SEP (2023)~\cite{chen2023self}                   &  \pie{0}   &     \pie{0}    & Classification &       Image          &         S          &      \pie{180}   &   \pie{180}  & ImageNet \\
&                                           & UT  (2023)~\cite{li2023make}                     &  \pie{0}   &     \pie{0}    & Classification/Question Answering & Text &      S          &      \pie{180}   &   \pie{0}    & AG-News / SQuAD \\
&                                           & $\text{M}^3$ (2024)~\cite{fang2024re}            &  \pie{0}   &     \pie{0}    & Classification &       Image          &         S          &      \pie{360}   &   \pie{0}    & ImageNet \\
&                                           & SALM (2024)~\cite{sun2024medical}                &  \pie{0}   &     \pie{0}    & Classification &       Image          &         S          &      \pie{180}   &   \pie{0}    & MedMNIST-v2 \\
&                                           & SEM (2024)~\cite{liu2024stable}                  &  \pie{0}   &     \pie{0}    & Classification &       Image          &         S          &      \pie{360}   &   \pie{0}    & ImageNet \\
&                                           & SAPA (2024)~\cite{he2024sharpness}               &  \pie{0}   &     \pie{0}    & Classification &       Image          &         S          &      \pie{180}   &   \pie{180}  & CIFAR-10/100 \\
&                                           & HarmonyCloak (2024)~\cite{meerza2024harmonycloak}&  \pie{0}   &        -       & Music Synthesis&      Audio (MIDI)    &         S          &      \pie{180}   &   \pie{0}    & Lakh MIDI \\
&                                           & POP (2024)~\cite{zhang2024mitigating}            &  \pie{0}   &        -       &Speech Synthesis&    Audio (Waveform)  &         S          &      \pie{180}   &   \pie{0}    & LibriTTS \\
&                                           & AUE (2024)~\cite{wang2024efficient}              &  \pie{0}   &     \pie{0}    & Classification &       Image          &         S          &      \pie{180}   &   \pie{0}    & ImageNet \\
&                                           & MEM (2024)~\cite{liu2024multimodal}              &  \pie{0}   &     \pie{0}    &Image/Text Retrieval&   Image          &         S          &      \pie{0}     &   \pie{0}    & MSCOCO/Flickr30k \\
&                                           & Unseg (2024)~\cite{sun2024unseg}                 &  \pie{0}   &     \pie{0}    &Image Segmentation&     Image          &         S          &      \pie{180}   &   \pie{0}    & Cityscapes / MSCOCO \\
&                                           & UGE (2024)~\cite{ye2024ungeneralizable}          &  \pie{0}   &     \pie{0}    & Classification &       Image          &         S          &      \pie{0}     &   \pie{0}    & Tiny-ImageNet \\
&                                           & \textsc{Armor} (2025)~\cite{gong2025armor}       &  \pie{0}   &     \pie{0}    & Classification &       Image          &       S / C        &      \pie{360}   &   \pie{0}    & ImageNet \\
&                                           & \textsc{Myopia} (2025)~\cite{wu2025myopia}       &  \pie{0}   &        -       & Text-to-Image  &       Image          &         S          &      \pie{180}   &   \pie{0}    & CelebA-HQ \\
&                                           & T2UE (2025)~\cite{ma2025t2ue}                    &  \pie{0}   &     \pie{0}    & Image/Text Retrieval&  Image          &         S          &      \pie{180}   &   \pie{0}    & Tiny-MSCOCO / Flickr30k \\
&                                           & Segue (2025)~\cite{zhang2025segue}               &  \pie{0}   &   \pie{360}    & Classification &       Image          &         S          &      \pie{180}   &   \pie{0}    & VGGFace2 \\
&                                           & PUE (2025)~\cite{wang2025provably}               &  \pie{0}   &     \pie{0}    & Classification &       Image          &         C          &      \pie{180}   &   \pie{360}  & ImageNet \\
&                                           & SafeSpeech (2025)~\cite{zhang2025safespeech}     &  \pie{0}   &        -       &Speech Synthesis&    Audio (Waveform)  &         S          &      \pie{360}   &   \pie{180}  & LibriTTS \\
&                                           & TempUE (2025)~\cite{zhang2025safespeech}         &  \pie{0}   &     \pie{0}    &Object Tracking &       Video          &         S          &      \pie{180}   &   \pie{0}    & YouTube-VOS 2019 \\
&                                           & VTG (2025)~\cite{li2025versatile}                &  \pie{0}   &     \pie{0}    & Classification &       Image          &         S          &      \pie{180}   &   \pie{0}    & PACS  \\  
&                                           & PIL (2026)~\cite{liu2026perturbation}            &  \pie{0}   &     \pie{0}    & Classification &       Image          &         S          &      \pie{180}   &   \pie{0}    & ImageNet  \\  
&                                           & BAIT (2026)~\cite{li2026priors}                  &  \pie{0}   &     \pie{0}    & Classification &       Image          &         C          &      \pie{180}   &   \pie{0}    & ImageNet  \\ 
\midrule[0.3pt]
\multirow{8}{*}{\rotatebox[origin=c]{90}{Distribution}}
&        \multirow{8}{*}{Shortcut}          & AR (2022)~\cite{sandoval2022autoregressive}      &  \pie{360} &   \pie{360}    & Classification &       Image          &         C          &      \pie{180}   &   \pie{180}  & CIFAR-10 \\
&                                           & LSP (2022)~\cite{yu2022availability}             &  \pie{360} &     \pie{0}    & Classification &       Image          &         C          &      \pie{180}   &   \pie{0}    & ImageNet \\
&                                           & OPS (2023)~\cite{wu2023one}                      &  \pie{360} &     \pie{0}    & Classification &       Image          &         C          &      \pie{180}   &   \pie{0}    & ImageNet \\
&                                           & CUDA (2023)~\cite{sadasivan2023cuda}             &  \pie{360} &     \pie{0}    & Classification &       Image          &         C          &      \pie{360}   &   \pie{0}    & ImageNet \\
&                                           & PosCUDA (2024)~\cite{gokul2024poscuda}           &  \pie{360} &     \pie{0}    & Classification &       Audio (Speech) &         C          &      \pie{180}   &   \pie{0}    & SpeechCommands \\
&                                           & UMT (2024)~\cite{wang2024unlearnable}            &  \pie{360} &     \pie{0}    &Classification/Segmentation&Point Cloud&         C          &      \pie{180}   &   \pie{0}    & S3DIS / KITTI \\
&                                           & RegText (2025)~\cite{shahid2025towards}          &  \pie{360} &     \pie{0}    & Text Classification/Generation & Text &         C          &      \pie{180}   &   \pie{0}    & Natural Instructions \\
&                                           & ExpShield (2026)~\cite{liu2026expshield}         &  \pie{360} &        -       &Text Generation/Vision-to-Language&Text&         S          &      \pie{180}   &   \pie{0}    & CC-News/IAPR-TC-12 \\
\midrule[0.3pt]
\multirow{8}{*}{\rotatebox[origin=c]{90}{Representation}}
&       \multirow{3}{*}{Feature Dissimilarization}
                                            & EntF-push (2023)~\cite{wen2023adversarial}       &  \pie{0}   &     \pie{0}    & Classification &       Image          &         C          &      \pie{360}   &   \pie{0}    & Tiny-ImageNet \\
&                                           & 14A (2024)~\cite{chen2024one}                    &  \pie{0}   &   \pie{360}    & Classification &       Image          &         S          &      \pie{180}   &   \pie{0}    & ImageNet \\
&                                           & DH (2024)~\cite{meng2024semantic}                &  \pie{0}   &   \pie{360}    & Classification &       Image          &         S          &      \pie{180}   &   \pie{0}    & ImageNet \\
\cmidrule(lr){2-11}
&       \multirow{5}{*}{Feature Collision}  
                                            & Fawkes (2020)~\cite{shan2020fawkes}              &  \pie{0}   &     \pie{0}    & Classification &       Image          &         S          &      \pie{180}   &   \pie{0}    & WebFace \\
&                                           & EntF-pull (2023)~\cite{wen2023adversarial}       &  \pie{0}   &     \pie{0}    & Classification &       Image          &         C          &      \pie{360}   &   \pie{0}    & Tiny-ImageNet \\
&                                           & GLAZE (2023)~\cite{shan2023glaze}                &  \pie{0}   &        -       &  Text-to-Image &       Image          &         S          &      \pie{180}   &   \pie{0}    & WikiArt / Original Art \\
&                                           & NightShade (2024)~\cite{shan2024nightshade}      &  \pie{0}   &        -       &  Text-to-Image &       Image          &         S          &      \pie{180}   &   \pie{0}    & Conceptual Caption \\
&                                           & MI-UE (2026)~\cite{zhu2026unlearnable}           &  \pie{0}   &     \pie{0}    & Classification &       Image          &         S          &      \pie{180}   &   \pie{0}    & ImageNet \\
    \midrule[0.3pt]
    \bottomrule[1pt]
    \end{tabular}
    }
\end{table*}

\section{Taxonomy and Literature Collation}\label{sec:tax_and_overview}

\subsection{Unifying Taxonomy of \MethodName and Unlearning}
Our taxonomy encompasses both \MethodName and unlearning methods, as illustrated in Figure~\ref{fig:taxonomy}. 
For \MethodName, we categorize existing approaches into three groups: \textit{error-based \MethodName}, \textit{distribution-based \MethodName}, and \textit{representation-based \MethodName}. 
For unlearning, we classify prior work into \textit{exact} and \textit{approximate} methods. 
Exact unlearning is typically achieved via \textit{retraining}, while approximate unlearning is further divided into \textit{non-certifiable}, \textit{convex-certifiable}, and \textit{non-convex-certifiable} methods, based on the strength of their guarantees.

\subsection{Literature Collation}
To ensure this SoK covers representative literature, we conducted a comprehensive survey of peer-reviewed articles from leading AI, machine learning, and security venues. 
Papers were collected from major databases, including DBLP and Google Scholar, as well as official websites of leading venues such as ICML, ICLR, NeurIPS, CVPR, ICCV, IEEE S\&P, USENIX Security, NDSS, ACM CCS, and other relevant conferences.
We combined keyword-based searches using terms such as ``unlearnable,'' ``shortcut,'' ``unlearning,'' ``forget,'' ``data deletion,'' and ``data removal'' with GPT-5.4-assisted abstract analysis to identify relevant work.
Starting from this core set, we performed manual title and abstract screening, together with citation chaining, to identify additional relevant papers that did not explicitly contain these keywords. 
We included only articles with publicly available full texts at the time of writing to ensure accurate assessment.
From this pool, we manually reviewed and selected 187 highly relevant works, forming a comprehensive body of literature on scalable \MethodName and unlearning that reflects the state of the art.

We further provide systematic reviews of \MethodName and unlearning methods in Section~\ref{sec:input_methods} and Section~\ref{sec:output_methods}, respectively.
\section{Memorization Prevention via \MethodName}\label{sec:input_methods}
In this section, we delve into the detailed methods for generating \MethodName perturbations.
For ease of comparison, we summarize the key properties of different \MethodName methods in Table~\ref{tab:summary}.

\subsection{Error-based \MethodName}
The first prevailing category of \MethodName methods perturbs data samples based on training error (\eg cross-entropy), which measures the discrepancy between the data distribution and the learned distribution. 
By maximizing or minimizing this error on a surrogate model, the defender can induce mislearning or suppress effective learning, thereby concealing information in the data.

\subsubsection{Error-Maximizing Perturbations}
\paragraph{\textit{Mechanism}}
Error-maximizing (E-Max) noise shares a similar objective with inference-time adversarial examples. 
Given the dataset $\sD$, E-Max learns perturbations $\delta=\{\delta_i\}_{i=1}^{|\sD|}$ that maximize an error function $\gL(\cdot)$, while a surrogate model $f_{\theta}$ minimizes this error by updating $\theta$ on the perturbed data. 
This yields the following bi-level optimization problem:
\begin{equation}\small\label{eq:E-Max_def}
    \begin{aligned}
        & \delta^* = \max_{\delta} \min_{\theta} \sum_{(x_i,y_i)\in \sD} \gL (\theta; x_i+\delta_i, y_i), \\
        & \text{s.t.}\, \|\delta_i\|_p \leq B,\, \forall i\in\{1,2,...,|\sD|\},
    \end{aligned}
\end{equation}
where $B$ denotes the perturbation budget for preserving data utility. 
Some methods instead minimize an adversarial loss between $f_{\theta}(x_i+\delta_i)$ and an adversarial target~\cite{huang2020metapoison,zeng2023narcissus}.
We also categorize them as E-Max, since minimizing such a loss effectively maximizes the classification error.

E-Max noise is typically more effective when generated sample-wise. 
Since Equation~\ref{eq:E-Max_def} relies on a surrogate model for gradient computation, its transferability to unseen models is crucial in practice. 
To date, E-Max noise has mainly been applied to protecting image data in classification~\cite{yuan2021neural,fowl2021adversarial,liu2024game} and image generation tasks~\cite{van2023anti,liu2024countering,liu2024metacloak}.

\paragraph{\textit{Key literature}}
E-Max perturbations constitute the earliest line of work in \MethodName. 
TensorClog (TC) maximizes training loss by inducing gradient vanishing, thereby impeding effective training~\cite{shen2019tensorclog}. 
MetaPoison improves transferability across models and training settings by crafting perturbations with an ensemble of surrogate models~\cite{huang2020metapoison}. 
Neural Tangent Generalization Attack (NTGA) applies E-Max perturbations to training data to degrade model generalization~\cite{yuan2021neural}. 
E-Max noise has also shown strong effectiveness in Targeted Adversarial Poisoning (TAP)~\cite{fowl2021adversarial}. 
Using a surrogate trained on public out-of-distribution (POOD) data, Narcissus generates E-Max perturbations as clean-label poisons against image classifiers~\cite{zeng2023narcissus}. 
AAP extends E-Max noise to make data unlearnable for contrastive learning~\cite{wang2024efficient}, while GUE formulates unlearnable data generation as a game between a generator and a classifier, where the generator is optimized to maximize classification error~\cite{liu2024game}.

Beyond classification, E-Max perturbations have been adopted to protect images against unauthorized text-to-image (T2I) personalization and mimicry. 
InMark uses influence functions to identify salient pixels and applies gradient ascent to make images unlearnable for T2I models~\cite{liu2024countering}. 
MetaCloak maximizes denoising error in diffusion models to protect images from mimicry attacks~\cite{liu2024metacloak}. 
Anti-DreamBooth maximizes the conditional diffusion loss so that protected images cannot be used for malicious T2I personalization~\cite{van2023anti}.
Finally, Enhanced Unlearnable Diffusion Perturbation (EUDP) accounts for the varying importance of diffusion timesteps and adaptively optimizes E-Max noise on the most influential steps~\cite{zhao2024unlearnable}.

\subsubsection{Error-Minimizing Perturbations}
\paragraph{\textit{Mechanism}}
Compared with E-Max noise, error-minimizing (E-Min) perturbations often provide stronger protection in class-wise settings, where all samples from the same class share a universal perturbation. 
Given perturbations $\delta$, E-Min noise is obtained by:
\begin{equation}\small
    \begin{aligned}
        & \delta^* = \min_{\delta} \min_{\theta} \sum_{(x_i,y_i)\in \sD} \gL (\theta; x_i+\delta_i, y_i), \\
        & \text{s.t.}\, \|\delta_i\|_p \leq B,\, \forall i\in\{1,2,...,|\sD|\}.
    \end{aligned}
\end{equation}
For samples $x_i$ from the same class, the corresponding $\delta_i$ can be shared, yielding class-wise \MethodName noise.

Similar to E-Max, E-Min methods optimize perturbations through gradient-based procedures. 
However, they have been applied more broadly across data modalities, including images~\cite{huang2021unlearnable,fu2022robust}, text~\cite{li2023make,zhou2024making}, and audio~\cite{meerza2024harmonycloak,zhang2025safespeech}, as well as across tasks such as classification~\cite{huang2021unlearnable,wen2023adversarial}, question answering~\cite{li2023make}, speech synthesis~\cite{zhang2024mitigating,zhang2025safespeech}, and music synthesis~\cite{meerza2024harmonycloak}.

\paragraph{\textit{Key literature}}
E-Min perturbations were first introduced for generating Unlearnable Examples (UEs)~\cite{huang2021unlearnable}.
Subsequent work has mainly improved their robustness against input transformations, data augmentations, and alternative training paradigms.
GrayAugs enhances robustness to grayscale filtering~\cite{liu2021going}, while Hypocritical perturbations, also studied under delusive attacks, minimize classification loss on a static surrogate to evaluate defenses~\cite{tao2021better}. 
Unlearnable Cluster (UC) further constructs label-agnostic UEs by minimizing distances to incorrect cluster centers rather than labels~\cite{zhang2023unlearnable}. 
To defend against adversarially trained learners, Robust Unlearnable Examples (RUEs) incorporate adversarial perturbations into the E-Min optimization, making the resulting noise more robust to adversarial training~\cite{fu2022robust}. 
Similar objectives are pursued by the two-stage Min-Max-Min optimization framework ($\text{M}^3$)~\cite{fang2024re}, Segue~\cite{zhang2025segue}, and \textsc{Armor}, which improves robustness against data augmentation~\cite{gong2025armor}. 
Stable Error-Minimizing noise (SEM) further targets pixel corruptions by optimizing E-Min perturbations over random input perturbations within an adversarial training framework~\cite{liu2024stable}.

Another line of work improves transferability and stability. 
Transferable Unlearnable Examples (TUE) enhance the generalization of E-Min perturbations across training settings and datasets~\cite{ren2023transferable}. 
Segue employs a generator to produce UEs robust to adversarial training, JPEG compression, and data augmentations~\cite{zhang2025segue}, while Versatile Transferable Generator (VTG) improves cross-domain transferability through adversarial domain augmentation~\cite{li2025versatile}. 
To mitigate stochasticity in training, Self-Ensemble Protection (SEP) uses training checkpoints as an ensemble when generating \MethodName perturbations~\cite{chen2023self}, and Sharpness-Aware Data Poisoning Attack (SAPA) adopts a sharpness-aware loss to improve stability under training uncertainty~\cite{he2024sharpness}. 
Provably Unlearnable Examples (PUEs) instead strengthen robustness from a certification perspective by reducing the certified learnability of data~\cite{wang2025provably}.

Other methods address efficiency, reversibility, or broader deployment settings. 
LearnabilityLock uses an adversarial invertible transformation to generate reversible E-Min perturbations~\cite{peng2022learnability}. 
Ungeneralizable Examples (UGEs) leverage CLIP encoders to introduce feature-level contradictions, inducing misleading correlations that hinder generalization~\cite{ye2024ungeneralizable}. 
Sparsity-Aware Local Masking (SALM) improves efficiency with sparse E-Min noise~\cite{sun2024medical}, while Perturbation-Induced Linearization (PIL) accelerates perturbation generation using a linear surrogate classifier~\cite{liu2026perturbation}. 
BAIT further targets pretrained classifiers by rendering images unlearnable under prior knowledge~\cite{li2026priors}. 
Beyond images, E-Min perturbations have been extended to speech audio (POP~\cite{zhang2024mitigating}, SafeSpeech~\cite{zhang2025safespeech}), music (HarmonyCloak~\cite{meerza2024harmonycloak}), text (UT~\cite{li2023make}), video (TempUE~\cite{wu2025temporal}), and multimodal data (MEM~\cite{liu2024multimodal}, T2UE~\cite{ma2025t2ue}). 
They have also been applied beyond classification, including segmentation (Unseg~\cite{sun2024unseg}), contrastive learning (AUE~\cite{wang2024efficient}), and text-to-image generation (\textsc{Myopia}~\cite{wu2025myopia}).

\subsection{Distribution-based \MethodName}
Methods in this category generate \MethodName perturbations in a gradient-free manner. 
Instead of relying on surrogate models, they derive perturbations directly from the distribution of data features and their influence. 
Shortcuts constitute the main line of work under this category.

\subsubsection{Shortcut}
\paragraph{\textit{Mechanism}}
Shortcuts are local data patterns that induce preferential learning in models trained on the protected data. 
In this SoK, we define shortcuts as gradient-free, data-centric \MethodName perturbations derived from the distributional information of the protectee dataset. 
Typically, a fitness function is used to identify patterns that are easily learned by models and exploit them as shortcut signals.

The key principle is that overparameterized models, such as DNNs, tend to memorize associations between simple, frequent, and label-consistent input patterns and their corresponding labels. 
Based on this observation, shortcut-based \MethodName methods have been developed for images~\cite{wu2023one}, text~\cite{shahid2025towards}, and audio~\cite{gokul2024poscuda}, mainly in classification tasks. 
Extending such shortcuts to generative tasks remains challenging.

\paragraph{\textit{Key literature}}
Shortcuts between inputs and outputs can be established by inserting special patterns into data features. 
In our taxonomy, the defining property of shortcuts is that their construction is \emph{gradient-free}. 
Unlike E-Max and E-Min perturbations, shortcuts exploit distributional learning preferences: models tend to capture correlations between certain input patterns and their associated outputs. 
These shortcuts may naturally occur in data or be deliberately engineered~\cite{geirhos2020shortcut,brown2023detecting}. 
Their model-agnostic and stealthy nature makes them effective for reducing data learnability.

Auto-Regressive (AR) poisoning generates shortcuts in a data- and model-agnostic manner~\cite{sandoval2022autoregressive}. 
Linearly-Separable Perturbation (LSP) is based on the conjecture that \MethodName perturbations for different classes can be made linearly separable, thereby serving as effective shortcuts~\cite{yu2022availability}. 
CUDA shows that convolution-induced blurring can act as a shortcut for learning~\cite{sadasivan2023cuda}, and PosCUDA extends this idea to audio protection~\cite{gokul2024poscuda}. 
One-Pixel Shortcut (OPS) uses limited pixel flipping to conceal domain knowledge~\cite{wu2023one}. 
It identifies pixels that induce the largest distributional drift and maximizes their values across the dataset to make the data unlearnable. Since OPS is model-agnostic, it is robust to variations in model architectures and parameters.
For text data, RegText constructs textual shortcuts against LLMs in classification tasks~\cite{shahid2025towards}, while ExpShield inserts random tokens as shortcuts to maximize token generation error~\cite{liu2026expshield}. 
Beyond images and text, Unlearnable Multi-Transformation (UMT) applies 3D transformations to convert point clouds into UEs~\cite{wang2024unlearnable}.

\subsection{Representation-based \MethodName }
Latent representations can also guide the design of \MethodName perturbations. 
Methods in this category optimize perturbations by either separating the feature representations of protected data from those of the original data, or inducing feature collisions with a target sample in the latent space.

\subsubsection{Feature Dissimilarization}
\paragraph{\textit{Mechanism}}
Given a protectee sample $x_i$, feature dissimilarization perturbs it with $\delta_i$ to push its latent representation $\Phi(x_i+\delta_i)$ away from the original representation $\Phi(x_i)$, where $\Phi(\cdot)$ denotes a feature extractor. 
The general objective is:
\begin{equation}\small
\begin{aligned}
    & \delta^* = \max_{\delta} \sum_{x_i\in \sD} \text{Dist}(\Phi(x_i+\delta_i), \Phi(x_i)), \\
    & \text{s.t.}\, \|\delta_i\|_p \leq B,\, \forall i\in\{1,2,...,|\sD|\},
\end{aligned}
\end{equation}
where $\text{Dist}(\cdot)$ measures the distance between latent representations. 
Unlike E-Min and E-Max optimization, this objective can be label-free, making feature dissimilarization particularly suitable for unlabeled data protection. 
Existing methods have mainly focused on image classification tasks~\cite{chen2024one,meng2024semantic}.

\paragraph{\textit{Key literature}}
Feature dissimilarization reduces data learnability by manipulating representations rather than raw inputs. 
Specifically, it operates in the feature space of feature extractors instead of directly optimizing pixel-level losses.

Entangled Feature (EntF-push) uses feature dissimilarization to improve the robustness of \MethodName noise against adversarial training~\cite{wen2023adversarial}. 
14A employs a generator to maximize the conceptual gap between perturbed and original images in the embedding space, producing semantically misaligned yet natural-looking samples~\cite{chen2024one}. 
Deep Hiding (DH) further embeds semantic images into protectee samples via invertible neural networks, concealing sensitive information while injecting structured feature-space noise~\cite{meng2024semantic}.

\subsubsection{Feature Collision}
\paragraph{\textit{Mechanism}}
In contrast to feature dissimilarization, feature collision minimizes the distance between the perturbed representation $\Phi(x_i+\delta_i)$ and the representation $\Phi(T)$ of a target image $T$. 
The target can be selected from an open domain where privacy or copyright concerns do not arise. 
The defender optimizes:
\begin{equation}\small
\begin{aligned}
    & \delta^* = \min_{\delta} \sum_{x_i\in \sD} \text{Dist}(\Phi(x_i+\delta_i), \Phi(T)), \\
    & \text{s.t.}\, \|\delta_i\|_p \leq B,\, \forall i\in\{1,2,...,|\sD|\}.
\end{aligned}
\end{equation}

Feature collision was originally introduced for clean-label poisoning attacks~\cite{shafahi2018poison,zhu2019transferable} and has later been adapted to protect image data, particularly in T2I tasks~\cite{shan2023glaze}. 
Unlike feature dissimilarization, this approach requires specifying a target class or target domain to induce the collision. 
Although precise labels are not always necessary, rough category knowledge is required.
Hence, these methods are not label-free.

\paragraph{\textit{Key literature}}
Feature collision methods perturb protected images so that their latent features resemble those of target images. 
This idea was first proposed as a clean-label poisoning attack against classifiers~\cite{shafahi2018poison} and later extended to Entangled Feature Pulling (EntF-pull) for crafting perturbations robust to adversarial training~\cite{wen2023adversarial}. 
MI-UE further optimizes feature relations by minimizing inter-class cosine similarity while maximizing intra-class cosine similarity~\cite{zhu2026unlearnable}.

Several feature-collision methods have also been developed for generative settings. 
Fawkes protects online facial images from unauthorized face recognition by moving the features of cloaked photos toward those of target identities~\cite{shan2020fawkes}. 
As a result, models trained on cloaked images fail to recognize clean images of the same users. 
GLAZE protects artistic styles from unauthorized T2I mimicry by selecting a public target style, stylizing the protected artwork accordingly, and optimizing a style cloak that aligns the original latent features with those of the stylized target~\cite{shan2023glaze}. 
Similarly, Nightshade defends against unauthorized T2I models by injecting prompt-specific poisoning through feature collision~\cite{shan2024nightshade}.

\begin{formal}{Remark on \MethodName as an upstream defense.}
Unlearnability acts as an upstream defense in the ML lifecycle and is therefore vulnerable to downstream data manipulations and variations in learning or unlearning algorithms. 
Its effectiveness can be weakened in both feature and parameter spaces: preprocessing defenses (\eg Shortcut Squeezing~\cite{liu2023image}, DiffPure~\cite{nie2022diffusion}, IMPRESS~\cite{cao2023impress}, IMPRESS++~\cite{honig2025adversarial}) and adversarial training~\cite{madry2018towards} can remove or suppress protective perturbations, while Deep Feature Reweighting~\cite{sandoval2023can} and recovery attacks~\cite{wang2025provably} can restore model performance. 
These limitations underscore the need for stronger robustness theory and more rigorous evaluation against downstream threats~\cite{wang2025provably,ye2025far}.
\end{formal}
\section{Dememorization via Unlearning}\label{sec:output_methods}
This section systematically reviews unlearning methods as output-level dememorization. 
We focus on approaches that scale to models comparable to those addressed by \MethodName, such as DNNs, while also covering convex-certifiable unlearning methods for completeness.

\subsection{Exact Unlearning}
As the name suggests, exact unlearning removes the forget set by retraining the model solely on the retain set. 
Retraining is therefore the principal approach in this category.

\subsubsection{Retraining}
\paragraph{\textit{Mechanism}}
Exact unlearning requires removing the forget set as if it had never participated in training. 
Formally, for any measurable subset $S$ of the hypothesis space, an unlearning algorithm $\gU$ satisfies
\begin{equation}\small
    \begin{aligned}
        \Pr[ \gU(\Gamma(\sD), \sD, \sD_f) \in S] = \Pr[ \Gamma(\sD \setminus \sD_{f}) \in S].
    \end{aligned}
\end{equation}
The most direct instantiation of $\gU(\cdot)$ is full retraining. 
Since this is computationally expensive, existing methods often adopt divide-and-conquer strategies, partitioning the dataset into shards and training separate constituent models on them~\cite{bourtoule2021machine}. 
These models are then aggregated at inference time. 
Subsequent work reduces unlearning cost through optimized partitioning~\cite{aldaghri2021coded,yan2022arcane} or parameter-efficient retraining~\cite{dukler2023safe,chowdhury2025towards}.

\paragraph{\textit{Key literature}}
Retraining (\textit{RT}) is the most straightforward exact-unlearning approach. 
The representative method is Sharded, Isolated, Sliced, and Aggregated (SISA) training~\cite{bourtoule2021machine}, which partitions the training set into disjoint shards and trains independent constituent models on them. 
During inference, all constituents process the input and their outputs are aggregated. 
Since each constituent depends only on its own shard, unlearning a data point only requires retraining the constituent trained on the shard containing that point.

SISA, however, faces an inherent scalability trade-off.
Increasing the number of shards reduces the data per constituent and lowers the cost of each unlearning request, but it also increases the number of models that must be trained and maintained. 
Conversely, using fewer shards reduces the overall training and storage burden but increases the retraining cost for each deletion. This trade-off limits scalability as models and datasets grow.
To mitigate this issue, later work improves shard partitioning~\cite{aldaghri2021coded,yan2022arcane,ullah2023adaptive,xia2025edge,yu2025split} or uses adapters to reduce retraining cost~\cite{dukler2023safe,chowdhury2025towards,quan2025efficient}. 
Nevertheless, the fundamental cost--scalability dilemma remains, motivating approximate unlearning as a more practical alternative.

\subsection{Approximate Unlearning}
Approximate unlearning is generally more scalable than exact unlearning and has become the dominant paradigm for large models such as LLMs and LVMs. 
However, its scalability comes at the cost of weaker guarantees on deletion quality. 
In this section, we categorize approximate unlearning into three types: non-certifiable, convex-certifiable, and non-convex-certifiable unlearning. 
Given the breadth of this literature, we focus primarily on methods that are scalable to large models.

\subsubsection{Non-Certifiable Unlearning}
\paragraph{\textit{Mechanism}}
Many approximate unlearning methods improve scalability by forgoing formal deletion guarantees. 
This category includes most unlearning techniques for large-scale models, such as LLMs and LVMs.
These methods are typically metric-driven. 
Given a performance metric $\gM(\cdot)$, the objective can be written as:
\begin{equation}\small
    \begin{aligned}
        \max_{\gU} \gM(\gU(\Gamma(\sD), \sD, \sD_f); \sD_r) - \gM(\gU(\Gamma(\sD), \sD, \sD_f); \sD_f),
    \end{aligned}
\end{equation}
where $\sD_r = \sD \setminus \sD_f$ denotes the retain set. 
Existing methods rely on heuristic updates, such as reversing gradients on the forget set~\cite{graves2021amnesiac,thudi2022unrolling}, fine-tuning on the retain set~\cite{lu2022quark,li2024machine,gao2025large}, or directly modifying model weights~\cite{cheng2023gnndelete,jia2023model}, to degrade performance on $\sD_f$ while preserving performance on $\sD_r$.

\paragraph{\textit{Key literature}}
Gradient ascent (\textit{GA}) is a widely used non-certifiable unlearning strategy~\cite{graves2021amnesiac,thudi2022unrolling,yao2024large,sepahvand2025selective,kim2025unlearning,zhu2026decoupling,zhao2026don,tang2026sharpness,di2026unlearning}. 
It heuristically reverses learning on selected data by applying gradients that increase the forget-set loss. 
GA has also been extended to LLM and LVM unlearning~\cite{zhang2024negative,cha2025towards,alberti2025data,liao2026explainable,li2026llm,cheng2026machine,liu2026randomized,asif2026ofmu}.

Fine-tuning (\textit{FT}) is another dominant scalable approach~\cite{golatkar2020eternal,lu2022quark,tarun2023deep,kurmanji2023towards,yoon2024few,shen2024label,li2024machine,huang2024unified,li2024single,park2024direct,zhuoyi2025adversarial,di2025adversarial,georgiev2025attribute,chen2025score,feng2025controllable,gao2025large,ding2025unified,wang2025llm,khalil2025not,zhou2025decoupled,chen2023boundary,block2025machine,kawamura2025approximate,entesari2025constrained,khalil2025coun,lee2025distillation,zhong2025dualoptim,zhou2025efficient,sui2025elastic,hu2025falcon,siddiqui2025dormant,ren2025keeping,shen2025llm,wang2025machine,lee2025ruago,fan2025simplicity,hsu2025unseen,tan2025wisdom,zade2026attention,chen2026clue,kim2026cooccurring,lee2026continual,lang2026downgrade,yang2026erase,yu2026falw,deng2026forget,li2026knowledge,di2026label,newatia2026mitigating,cheng2026remaining,gu2026towards,facchiano2026video,scholten2026model}. 
It induces catastrophic forgetting of $\sD_f$ while retaining knowledge from $\sD_r$.
We also include retain- or forget-set-free gradient descent methods in this category, as they similarly optimize heuristic unlearning objectives. 
FT has further been applied to remove environmental information from reinforcement learning agents~\cite{gong2025trajdeleter,ye2025reinforcement}.

Beyond GA and FT, non-certifiable unlearning has explored influence functions~\cite{bui2026cure}, direct weight editing~\cite{cheng2023gnndelete,jia2023model,jia2024wagle,fan2024salun,buarbulescu2024each,biswas2025cure,dong2025machine,schoepf2026redirection,rezaeirevisiting}, reinforcement learning~\cite{zhang2025rule,zaradoukas2026reinforcement}, adversarial training~\cite{zhang2024defensive}, meta-unlearning~\cite{li2023ultrare,zhao2024makes,li2025towards}, model-assisted unlearning~\cite{ji2024reversing,zhong2025unlearning,zhu2025llm}, and output tempering~\cite{spartalis2025lotus}. 
For graph-structured data, partial retraining has been used as an approximate unlearning strategy~\cite{chen2022graph,wang2023inductive}. 
For LLMs, prompt engineering has also been studied as a lightweight mechanism for inducing unlearning-like behavior~\cite{pawelczyk2024context,liu2024large,wang2026dragon,chen2026robust}.

\subsubsection{Convex-Certifiable Unlearning}
\paragraph{\textit{Mechanism}}
Convex-certifiable unlearning provides formal $(\epsilon,\zeta)$-unlearning guarantees, as in Definition~\ref{def:indistinguishability}, for models trained with convex losses. 
Most methods in this category use influence-function-based Newton updates to remove the effect of $\sD_f$. 
Let $\theta^*=\Gamma(\sD)$ and $\theta_r^*=\Gamma(\sD_r)$. 
Then,
\begin{equation}\small
    \theta_r^* \approx \theta^* - H^{-1}\bigtriangledown_{\theta} \gL(\theta^*; \sD_f),
\end{equation}
where $\gL(\theta^*; \sD_f)$ denotes the loss on the forget set and $H$ is the Hessian of the training objective. 
These guarantees critically rely on convexity assumptions~\cite{guo2020certified,warnecke2023machine}. 
Recent extensions address settings such as adaptive deletion~\cite{gupta2021adaptive} and data-free unlearning~\cite{ahmed2025towards}.

\paragraph{\textit{Key literature}}
Certified unlearning was introduced through $(\epsilon,\delta)$--certified data removal~\cite{guo2020certified}, inspired by DP~\cite{dwork2006differential}. 
We denote this notion as $(\epsilon,\zeta)$-unlearning to avoid notation conflict. 
Related formulations include total variation stability for measuring indistinguishability before and after unlearning~\cite{ullah2021machine}, reservoir-sampling-based convex optimization for long sequences of deletion requests~\cite{neel2021descent}, and Bayesian unlearning for Gaussian processes and logistic regression~\cite{nguyen2020variational}. 
Checkpointing has also been used to improve certified unlearning efficiency for convex objectives~\cite{sekhari2021remember}.

Feature--Label Unlearning (FLU) provides an efficient alternative to instance-level unlearning by perturbing features, modifying labels, and using influence functions to update model weights~\cite{warnecke2023machine}. 
When the loss is strictly convex and twice differentiable, FLU achieves $(\epsilon,\zeta)$--certified unlearning. 
Beyond standard deletion, certified methods have been extended to adaptive, data-free, and other complex settings~\cite{gupta2021adaptive,liu2023certified,ahmed2025towards,basaran2025certified}. 
More recently, noisy gradient descent has been used to obtain certified unlearning under convexity assumptions~\cite{chien2024certified,pandey2025gaussian,allouah2025utility}, while distributional unlearning has been studied for selective data removal~\cite{allouah2026distributional}.

Despite their formal guarantees, convex-certifiable methods scale poorly to DNNs because their assumptions, especially convexity, do not hold for modern non-convex models.

\subsubsection{Non-Convex-Certifiable Unlearning}
\paragraph{\textit{Mechanism}}
Certifiable unlearning in non-convex settings is highly desirable, as non-convexity is a key barrier introduced by modern DNNs. 
Overcoming this barrier would make certified unlearning applicable to large-scale deep models in practical deployments. 
Existing methods typically use noisy fine-tuning to provide $(\epsilon,\zeta)$-unlearning guarantees for models with non-convex losses. 
Such guarantees are often derived from distribution-level bounds, such as R\'enyi DP~\cite{chien2024langevin,koloskova2025certified}. 
With probability at least $p$, one can bound
\begin{equation}\small
    D [ \gU(\Gamma(\sD), \sD, \sD_f) \mid\mid \Gamma(\sD \setminus \sD_{f}) ] \leq \tau_{p},
\end{equation}
where $D[\cdot]$ denotes a divergence measure (\eg R\'enyi divergence, Kullback--Leibler divergence, or total variation), and $\tau_p$ is the corresponding computable upper bound. 
Under suitable noise scales and gradient truncation mechanisms, this divergence can be bounded after sufficiently many noisy fine-tuning steps.

\paragraph{\textit{Key literature}}
Methods in this category primarily rely on DP techniques. 
Recent work extends inverse-Hessian approximation to compute approximate Newton updates for certified unlearning in DNNs~\cite{zhang2024towards,wei2025provable}, and uses projected noisy gradient descent (PNGD) to obtain certified unlearning with DP guarantees~\cite{chien2024langevin}. 
Other approaches approximate Newton updates through sparse updates on salient parameters under second-order sufficiency conditions~\cite{mehta2022deep}, or approximate non-convex losses via mixed linear approximations~\cite{golatkar2021mixed}. 
However, these methods often depend on smooth losses or positive-definite Hessians, assumptions that may be violated by modern activation functions and loss formulations.

To relax these assumptions, recent studies introduce stochastic post-processing and privacy-amplification frameworks for DNN unlearning~\cite{koloskova2025certified,mu2025rewind}. 
These methods combine noisy fine-tuning with tailored gradient operations to enforce DP-style guarantees, avoiding restrictive smoothness assumptions and improving applicability to large-scale non-convex models. 
Further advances reduce the cost of Hessian estimation by attaching unlearning statistics to training samples and leveraging Hessian--vector products~\cite{qiao2025hessian}. 
Certified unlearning has also been extended to graph-structured data~\cite{yi2025scalable}, while connections between randomized smoothing and gradient quantization have been explored to support efficient certified deletion requests~\cite{zhang2022prompt}. 
Beyond instance-level deletion, feature-level certified unlearning has been studied by shuffling targeted features and fine-tuning on the modified dataset to remove their effects~\cite{yang2025feature}.

\begin{formal}{Remark on unlearning as a downstream defense.}
Table~\ref{tab:unlearning_taxonomy_counts} summarizes the trends across unlearning paradigms. 
Fine-tuning (\textit{FT}) dominates approximate unlearning, whereas gradient ascent (\textit{GA}) is mainly used in non-certifiable methods and influence functions (\textit{IF}) are more common in certified unlearning. 
However, the interaction between unlearning and \MethodName remains largely underexplored, leaving open how they can be unified into a robust dememorization framework spanning both training and post-training stages.
\begin{table}[H]
\centering
\caption{Number of papers in each unlearning category.}
\label{tab:unlearning_taxonomy_counts}
\resizebox{\linewidth}{!}{%
\begin{tabular}{ccccccccccccc}
\toprule[1pt]
\midrule[0.3pt]
Exact
& \multicolumn{12}{c}{Approximate} \\
\cmidrule(lr){2-13}
& \multicolumn{4}{c}{Non-certifiable}
& \multicolumn{4}{c}{Conv-certifiable}
& \multicolumn{4}{c}{Non-conv-certifiable} \\
\cmidrule(lr){2-5} \cmidrule(lr){6-9} \cmidrule(lr){10-13}
RT
& FT & GA & IF & Other
& FT & GA & IF & Other
& FT & GA & IF & Other \\
\midrule
9
& 55 & 17 & 1 & 25
& 5 & 0 & 7 & 2
& 4 & 0 & 3 & 4 \\
\midrule[0.3pt]
\bottomrule[1pt]
\end{tabular}%
}
\end{table}
\end{formal}

In realistic ML lifecycles, models are not necessarily trained on clean data. 
Models trained on \MethodName-protected data may later receive deletion requests.
For example, a model trained on a protected medical dataset may subsequently require unlearning under right-to-erasure obligations. 
Such cross-stage interactions can cause residual leakage and shallow dememorization. 
To answer RQ2, we evaluate the interplay between \MethodName and unlearning and reveal their shared vulnerabilities. 
In particular, we show that shallow dememorization is pervasive in both paradigms.
\section{Experiments}\label{sec:interaction}
\subsection{Efficacy of \MethodName}
To evaluate the resistance of \MethodName methods to unauthorized learning, we apply six types of \MethodName noise on CIFAR-10, CIFAR-100~\cite{krizhevsky2009learning}, and ImageNet with 100 randomly selected classes~\cite{russakovsky2015imagenet}. 
These include sample-wise methods (UE-s, TUE, TAP) and class-wise methods (UE-c, PUE, OPS). 
We convert the full CIFAR-10/100 training sets into \MethodName-protected datasets and use 20\% of ImageNet for efficiency. 
ResNet-18~\cite{he2016deep}, DenseNet-121~\cite{huang2017densely}, and ViT-Tiny~\cite{dosovitskiy2020image} are trained as unauthorized classifiers with data augmentation and evaluated on clean test sets, with clean-trained models as baselines (see Appendix~\ref{sec:Experiment_setting} for details).

As shown in Table~\ref{tab:test_acc_on_PDLC}, existing methods provide limited resistance under certain conditions. 
OPS consistently allows substantial clean knowledge extraction across datasets and architectures. 
UE-s and TUE are effective on smaller datasets but degrade on larger ones such as ImageNet. 
Across architectures, resistance is weakest against ViT-Tiny. 
Figure~\ref{fig_test_vit_no_aug_bar} further shows that data augmentation enables models trained on \MethodName-protected datasets to retain high clean accuracy. 
PUE remains relatively robust on ResNet-18 and DenseNet-121, but is less effective on ViT-Tiny, while TAP is intrinsically ineffective against ViT even without augmentation. 
Additional results on poisoning ratios are reported in Figure~\ref{fig_test_various_train_percent} and Appendix~\ref{sec:additional_results}.

Overall, current \MethodName methods do not reliably prevent unauthorized learning. 
Their effectiveness is sensitive to model architecture, image resolution, data augmentation, and poisoning ratio, leaving substantial clean knowledge retained in the trained model. 
This motivates a downstream unlearning stage to remove residual information and assess whether unlearning can complement \MethodName.

\begin{table}[t]
    \centering
    \caption{The test accuracy (\%) of classifiers trained on \MethodName data.}
    \resizebox{1\linewidth}{!}{%
    \begin{tabular}{c|ccc|ccc|ccc}
    \toprule[1pt]
    \midrule[0.3pt]
     \multirow{1}{*}{Dataset}& \multicolumn{3}{c}{CIFAR10} & \multicolumn{3}{c}{CIFAR100} & \multicolumn{3}{c}{ImageNet} \\  
    
       Perturbation  & RN-18 & DN-121 & ViT-Tiny  & RN-18 & DN-121 & ViT-Tiny  & RN-18 & DN-121 & ViT-Tiny \\     
    \midrule
        UE-c & 13.27 & 17.18 & 23.41 & 3.84 & 5.14 & 13.00 & 47.26 & 46.70 & 29.64 \\  
        PUE  & 10.02 & 10.01 & 75.33 & 2.52 & 2.10 & 23.72 & 14.52 & 16.58 & 16.90 \\  
        OPS  & 32.88 & 33.37 & 22.93 & 14.02 & 15.27 & 29.12 & 56.86 & 59.20 & 30.14 \\  
    
         UE-s & 20.61 & 18.29 & 25.37 & 9.07 & 9.14 & 20.25 & 17.70 & 20.96 & 19.42 \\  
        TUE  & 10.07 & 10.18 & 35.09 & 1.00 & 1.03 & 31.52 & 56.16 & 59.62 & 31.46 \\  
        TAP  & 9.10  & 9.86  & 35.13 & 7.05 & 6.73 & 20.41 & 6.38  & 6.80  & 17.80 \\  
    \midrule
   Clean & 94.50 & 95.35 & 77.57 & 70.86 & 74.02 & 45.53 & 63.42 & 67.10 & 37.52 \\
    \midrule[0.3pt]
    \bottomrule[1pt]
    \end{tabular}
    }
    \label{tab:test_acc_on_PDLC}
\end{table}

\begin{figure}[t]
    \centering
    \includegraphics[width=1\linewidth]{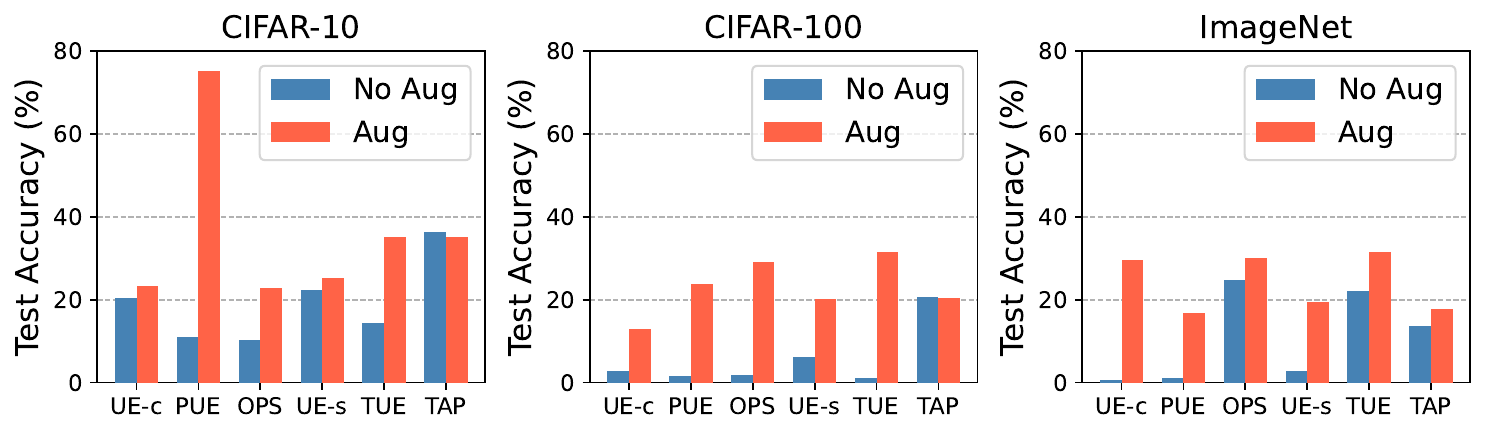}
    \caption{The test accuracy (\%) of ViT-Tiny trained on \MethodName data with and without augmentation.}
    \label{fig_test_vit_no_aug_bar}
\end{figure}

\subsection{Unlearnability-Unlearning Interaction}
\label{subsec:interplay}
This section investigates the interaction between \MethodName and unlearning. 
We examine how the two mechanisms affect each other when applied sequentially, and whether unlearning can effectively remove residual knowledge introduced by \MethodName-protected data.

\subsubsection{Experimental Design}
\paragraph{\textit{Forgetting targets}} 
We evaluate five forgetting targets: 1) clean knowledge leaked from \MethodName-protected samples, 2) \MethodName-protected samples, 3) clean samples, 4) \MethodName-protected classes, and 5) clean classes. 
These targets cover both \emph{subset-level} and \emph{class-level} unlearning.

In the subset-level setting, each target class contains 90\% \MethodName-protected data and 10\% clean data, and unlearning is evaluated on either subset. 
We denote the \MethodName training set, its clean counterpart, the clean training set, and the test set by $D_u$, $D_{uc}$, $D_c$, and $D_t$, respectively.

In the class-level setting, classes 0--3 are \MethodName-protected ($D_u$), while classes 4--9 are clean ($D_c$).
We unlearn each class from classes 0--7 separately and use classes 8--9 as the observation set ($D_{tr}$). 
For a \MethodName-protected class, $D_u$, $D_{uc}$, and $D_{ut}$ denote the forget training set, its clean counterpart, and the corresponding test set, respectively. 
For a clean class, Acc$D_c$ and Acc$D_{ct}$ denote the training and test accuracies after unlearning.

\paragraph{\textit{Unlearning Methods}} 
We compare three representative unlearning families: IF~\cite{izzo2021approximate,koh2017understanding}, FT~\cite{golatkar2020eternal}, and GA~\cite{thudi2022unrolling,graves2021amnesiac}. 
FT is run for 10 epochs and GA for 5 epochs. 
Full retraining (RT) on the retain set is used as the baseline.

\paragraph{\textit{Evaluation Metrics}} 
We evaluate \textit{unlearning effectiveness} and \textit{privacy leakage} using classification accuracy (\textit{Acc}) and membership inference attack (\textit{MIA}) success rate. 
Successful unlearning requires the accuracy on forgetting targets to be no higher than that of the retrained model, indicating removal of target knowledge, while maintaining accuracy on retained data and test data close to the retrained baseline. 
In addition, the MIA success rate on clean counterparts of \MethodName-protected data should be no higher than that of the retrained model after unlearning.

\paragraph{\textit{Experimental Settings}}
We apply unlearning to ResNet-18 models trained on CIFAR-10 with different \MethodName variants, including UE-s, TUE, OPS, and PUE. 
All models are trained for 100 epochs using SGD with learning rate 0.01, momentum 0.9, and weight decay $5\times10^{-4}$. 
Since exact Hessian inversion for IF is computationally expensive, we approximate $H^{-1}\nabla_\theta \mathcal{L}(\theta^*;\mathcal{D}_f)$ using first-order WoodFisher~\cite{singh2020woodfisher}, with the update scale controlled by $\alpha$.

Results are reported under two hyperparameter scales, small ($\bigtriangledown$) and large ($\bigtriangleup$). 
For IF, the scale is determined by $\alpha$; for FT and GA, it is determined by the learning rate. 
Detailed hyperparameters are provided in Tables~\ref{tab:Experimental Settings for Subset-Level Unlearning}--\ref{tab:Experimental Settings for Class-Level Unlearning}.

\begin{figure}[t]
    \centering
    \begin{subfigure}
        \centering
        \includegraphics[width=0.5\linewidth]{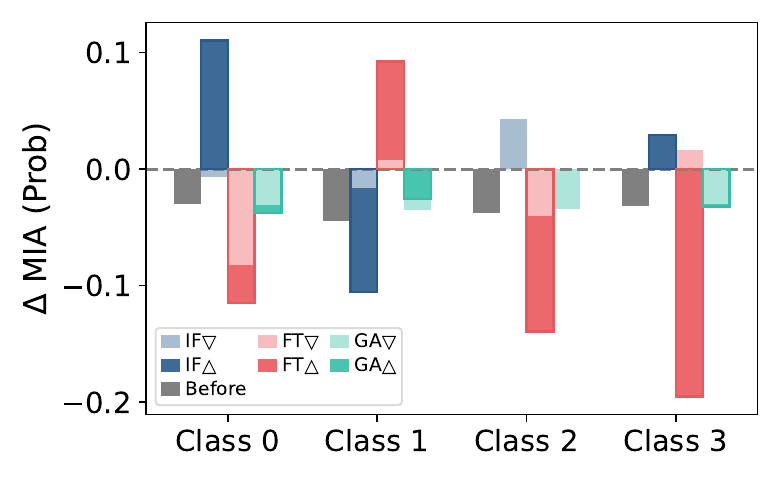}
        \label{11}
    \end{subfigure}
    \hspace{-1em} 
    \begin{subfigure}
        \centering
        \includegraphics[width=0.5\linewidth]{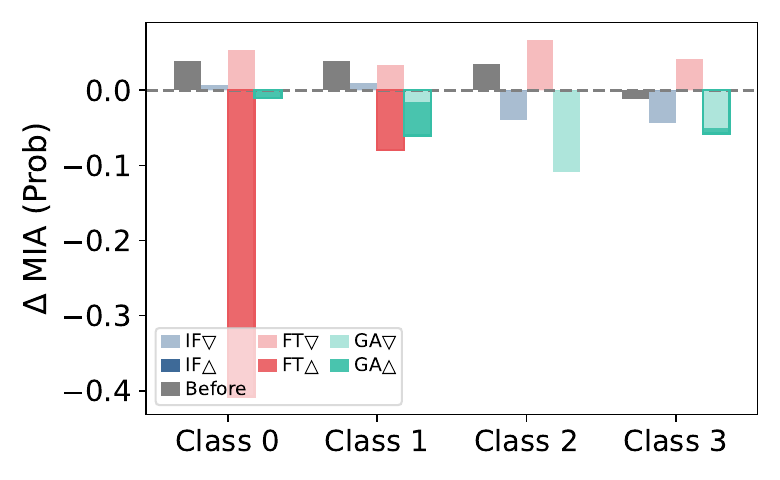}
        \label{22}
    \end{subfigure}
    \caption{MIA on UE-s (left:class-level; right:subset-level).}
\label{fig_MIA_UE}
\end{figure}
\begin{figure}[t]
    \centering
    \begin{subfigure}
        \centering
        \includegraphics[width=0.5\linewidth]{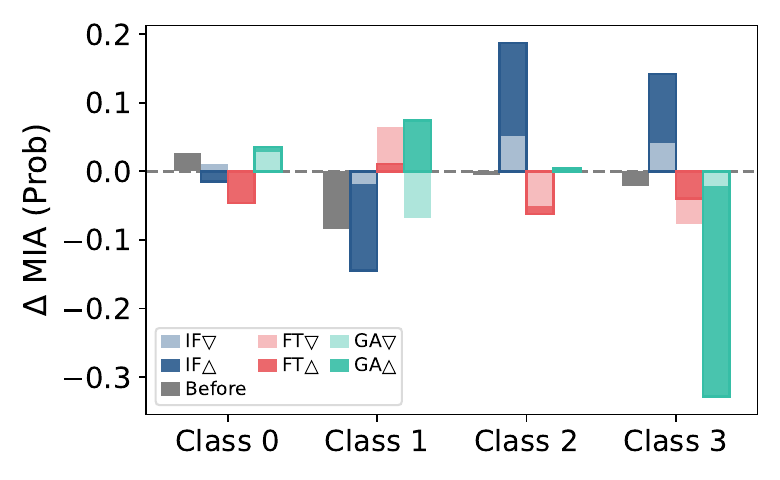}
        \label{11}
    \end{subfigure}
    \hspace{-1em}  
    \begin{subfigure}
        \centering
        \includegraphics[width=0.5\linewidth]{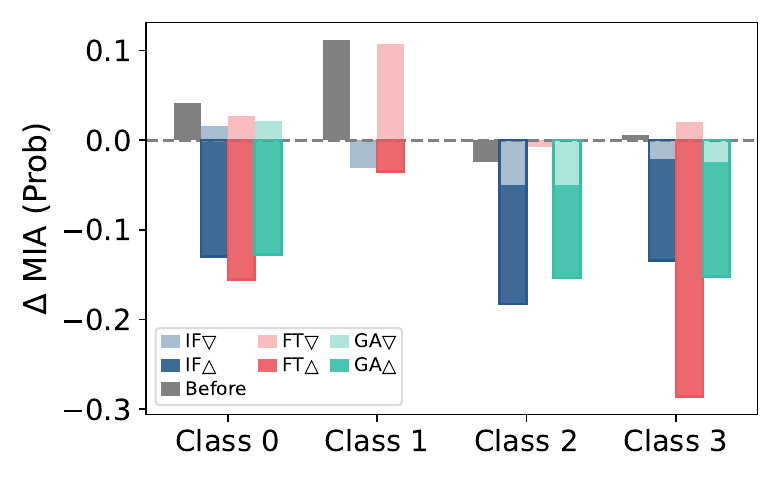}
        \label{22}
    \end{subfigure}
    \caption{MIA on OPS (left:class-level; right:subset-level).}
\label{fig_MIA_OPS}
\end{figure}

\begin{figure}[t]
    \centering
    \begin{subfigure}
        \centering
        \includegraphics[width=0.5\linewidth]{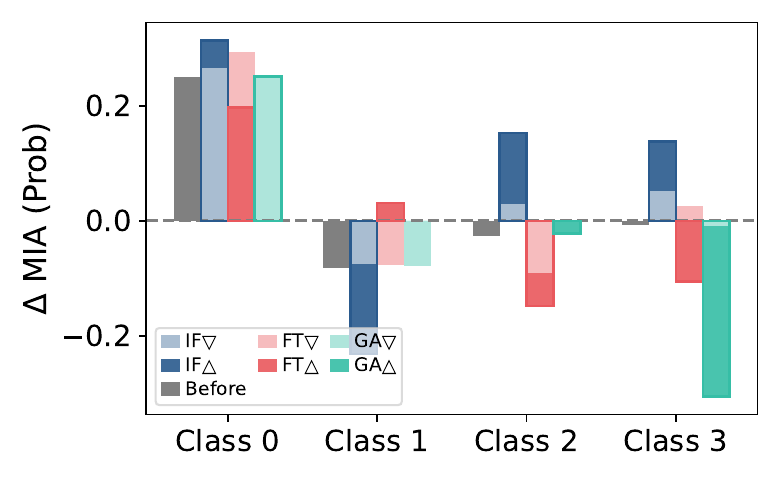}
        \label{11}
    \end{subfigure}
    \hspace{-1em} 
    \begin{subfigure}
        \centering
        \includegraphics[width=0.5\linewidth]{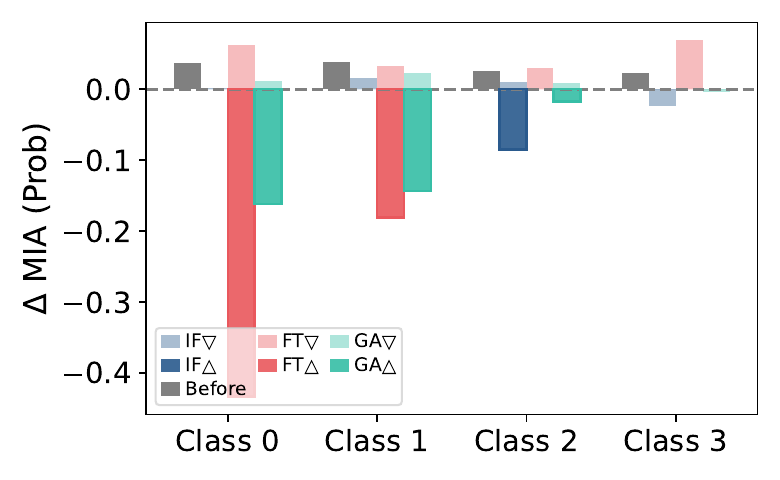}
        \label{22}
    \end{subfigure}
    \caption{MIA on PUE (left:class-level; right:subset-level).}
\label{fig_MIA_PUE}
\end{figure}

\begin{figure}[t]
    \centering
    \begin{subfigure}
        \centering
        \includegraphics[width=0.5\linewidth]{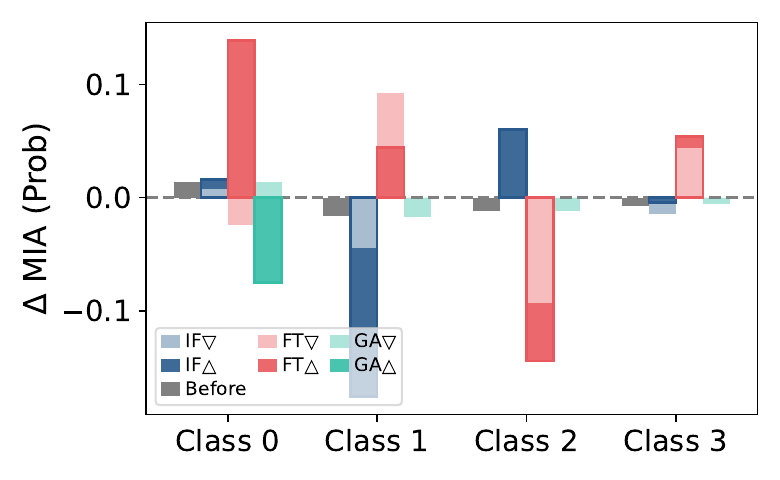}
        \label{11}
    \end{subfigure}
    \hspace{-1em} 
    \begin{subfigure}
        \centering
        \includegraphics[width=0.5\linewidth]{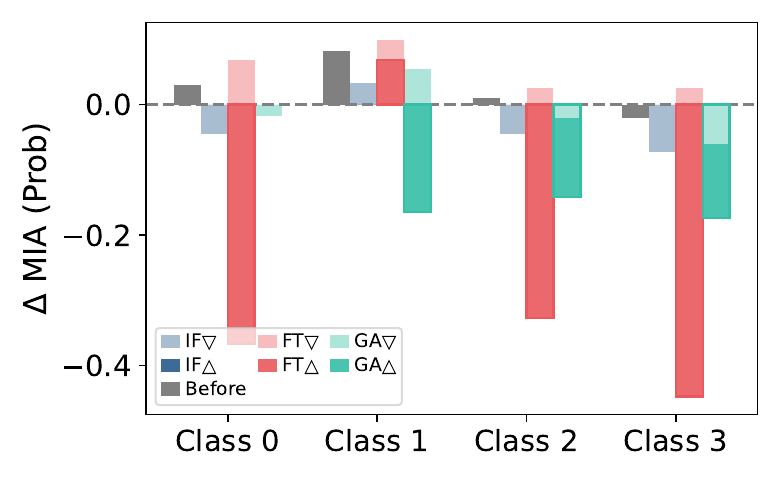}
        \label{22}
    \end{subfigure}
    \caption{MIA on TUE (left:class-level; right:subset-level).}
\label{fig_MIA_TUE}
\end{figure}


\begin{table}[h]
    \centering
    \caption{The accuracy (\%) of unlearned models trained on Regtext.}
    \resizebox{.8\linewidth}{!}{%
    \begin{tabular}{cccc|cccc}
\toprule[1pt]
\midrule[0.3pt]
 &  & Clean & Regtext & Retrain & GA & NPO & FT  \\
\midrule

& Task1        & 89.20 &78.80   &	69.30 &	68.70 &	65.10 &	75.60  \\
& Task2        & 99.20 &87.60  &	88.00 &	90.50{\color{red} \large$\uparrow$} &	91.20{\color{red} \large$\uparrow$} &	94.40{\color{red} \large$\uparrow$}  \\
& Task3        & 93.89 &76.08 &	66.07 &	79.58{\color{red} \large$\uparrow$} &	79.88{\color{red} \large$\uparrow$} &	80.68{\color{red} \large$\uparrow$}  \\
\midrule[0.3pt]
\bottomrule[1pt]               
    \end{tabular}
    }
    \label{tab:regtext_unlearn}
\end{table}

\subsubsection{Empirical Findings and Analysis on Unlearning Effectiveness}\label{subsubsec:unlearning}
For subset-level unlearning, results are reported in Tables~\ref{tab:unlearn_acc_on_UE}--\ref{tab:unlearn_acc_on_OPS}. 
Due to space constraints, we report the first eight classes and highlight cases where unlearning matches the retraining baseline. 
The symbol {\color{red} $\uparrow$} indicates an increase in Acc${D_{uc}}$. 
The main findings are summarized below.

\paragraph{\textit{Finding I: Forgetting clean knowledge leaked through \MethodName}}
With properly chosen hyperparameters, IF and GA effectively remove leaked clean knowledge, as reflected by reduced Acc${D_{uc}}$, while preserving utility on $D_c$ and $D_t$. 
In contrast, FT is highly sensitive to the learning rate.
An improper choice may fail to forget, or even reinforce, the leaked knowledge by increasing Acc${D_{uc}}$.

\paragraph{\textit{Finding II: Forgetting \MethodName-protected data}}
By examining Acc${D_u}$, we find that none of the three unlearning methods removes \MethodName-protected data as effectively as RT without degrading performance on the remaining sets. 
Forgetting \MethodName-protected data requires stronger intervention than forgetting their leaked clean knowledge, leading to substantial class-wise variation and a challenging trade-off between unlearning effectiveness and model utility.

\paragraph{\textit{Finding III: Forgetting clean data}}
IF and FT effectively remove $D_c$ while preserving test performance, whereas GA is less reliable and varies considerably across classes under the same learning rate. 
In particular, for most classes in TUE, UE and PUE, forgetting $D_c$ with IF also reduces Acc${D_u}$, suggesting that the the injected perturbations overlap with the true class-discriminative features. 
In contrast, Acc${D_u}$ remains high for OPS, indicating that its pixel-level shortcut is more independent of class-discriminative features.

For class-level unlearning, results are reported in Tables~\ref{tab:unlearn_class_acc_on_UE}--\ref{tab:unlearn_class_acc_on_OPS}. 
We summarize the finding below.

\paragraph{\textit{Finding IV: Class-level unlearning differs between protected and clean classes}}
Protected classes generally require stronger intervention than clean classes for IF and GA, creating a trade-off between forgetting effectiveness and utility preservation. 
In contrast, FT shows more consistent behavior across both protected and clean classes.

Overall, the results reveal a nuanced interaction between \MethodName and unlearning.
When \MethodName-protected and clean data coexist within the same class, unlearning struggles to remove perturbation-induced features, as they may overlap with class-discriminative features. 
Regarding residual leakage of clean knowledge from \MethodName variants,
Improper unlearning, such as poorly tuned FT, may even reinforce and further memorize residual knowledge and increase unauthorized exposure. 
At the class level, FT is more effective than IF and GA, consistently inducing stronger dememorization across both protected and clean classes under the same hyperparameter scale.
For completeness, we provide a theoretical justification of the unlearnability--unlearning interplay in Appendix~\ref{sec:interplay_theory}.

\subsubsection{Empirical Findings on Privacy Leakage of Clean Counterparts}
We report the \emph{gap} in probability-based MIA success rate between the unlearned model and the retrained baseline on clean counterparts of \MethodName-protected data (Figures~\ref{fig_MIA_UE}--\ref{fig_MIA_TUE}); positive values indicate residual privacy leakage. 
We also report MIA results based on alternative signals, including correctness and entropy, as complementary metrics (Tables~\ref{tab:unlearn_subset_MIA_on_UE-s}--\ref{tab:unlearn_class_MIA_on_OPS}). The main findings are summarized below.

\paragraph{\textit{Finding VI: Counterproductive privacy behavior under unlearning}}
Across both subset-level and class-level settings, none of the three unlearning methods consistently reduce MIA success rates to the retraining baseline. 
More concerningly, unlearning often exhibits \emph{counterproductive} behavior: it can increase MIA success on clean counterparts, especially under FT and IF, making the samples more vulnerable to MIA. 
This suggests that removing \MethodName-protected data may unintentionally amplify the distinguishability of the corresponding clean samples, thereby exacerbating rather than mitigating privacy leakage.

\paragraph{\textit{Finding VIII: Mismatch between accuracy-based unlearning and privacy preservation}}
Combining the MIA results with the accuracy-based analysis reveals a clear mismatch between unlearning effectiveness and privacy preservation. 
A model may appear to forget clean counterparts of \MethodName-protected data in terms of accuracy, yet still retain or even amplify membership leakage, particularly under IF. 
Thus, accuracy alone is insufficient for evaluating unlearning in the presence of \MethodName, and privacy-oriented metrics such as MIA are necessary for a comprehensive assessment.


\subsection{Evaluation of Shallow Dememorization}\label{subsec:shallow_censorship}
This section examines the shallow nature of \MethodName and unlearning as dememorization techniques. 
Although these methods can partially mitigate unauthorized knowledge acquisition, the stability of the resulting \textit{immemor} state remains unclear. 
We use recovery attacks~\cite{wang2025provably}, which apply projected stochastic gradient descent to perturb model weights, to assess the robustness of this state under parameter perturbations. 
Attack details are provided in Appendix~\ref{sec:Experiment_setting}.

\subsubsection{Model Recovery}\label{subsubsec:recovery_attack}
We launch recovery attacks against six \MethodName noise variants, including class-wise methods (UE-c, PUE, OPS) and sample-wise methods (TAP, UE-s, TUE), to measure how much model utility can be restored.

As shown in Figure~\ref{fig:robustness_P-DLC}, all \MethodName methods are more sensitive to weight perturbations on smaller datasets (CIFAR-10 and CIFAR-100), where recovery attacks substantially improve accuracy as the perturbation radius increases, especially for methods such as TAP. 
On ImageNet, some methods, such as OPS and TUE, appear more robust under the same perturbation budget. 
However, this apparent robustness coincides with weaker learnability reduction, as indicated by their higher baseline accuracies. 
Overall, although dataset scale affects recovery effectiveness, \MethodName methods consistently permit performance recovery across settings, suggesting that they impose only a shallow obstruction to clean knowledge acquisition.

\begin{figure}[t]
    \centering
    \includegraphics[width=1\linewidth]{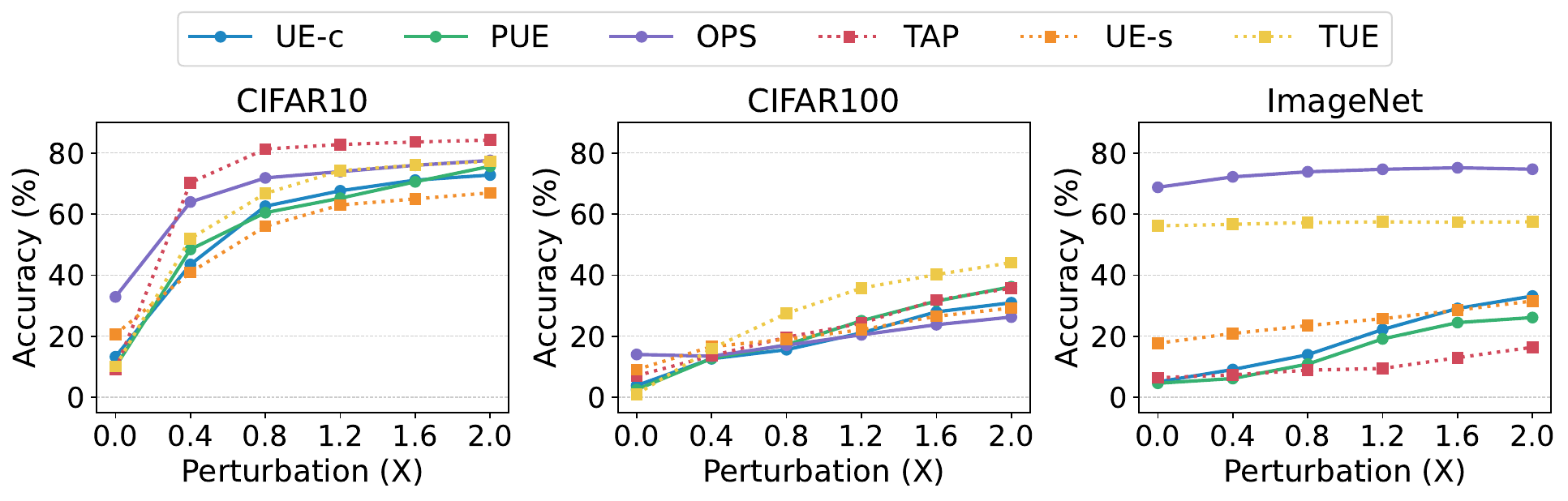}
    \caption{Parametric robustness of \MethodName perturbations gauged by recovery attacks. ``Accuracy'' denotes the clean test accuracy of recovered poisoned models whose weights have been perturbed within an $\ell_2$ norm bound of X.}
    \label{fig:robustness_P-DLC}
\end{figure}


\begin{figure*}[t]
    \centering
    
    \begin{subfigure}
        \centering
        \includegraphics[width=0.85\linewidth]{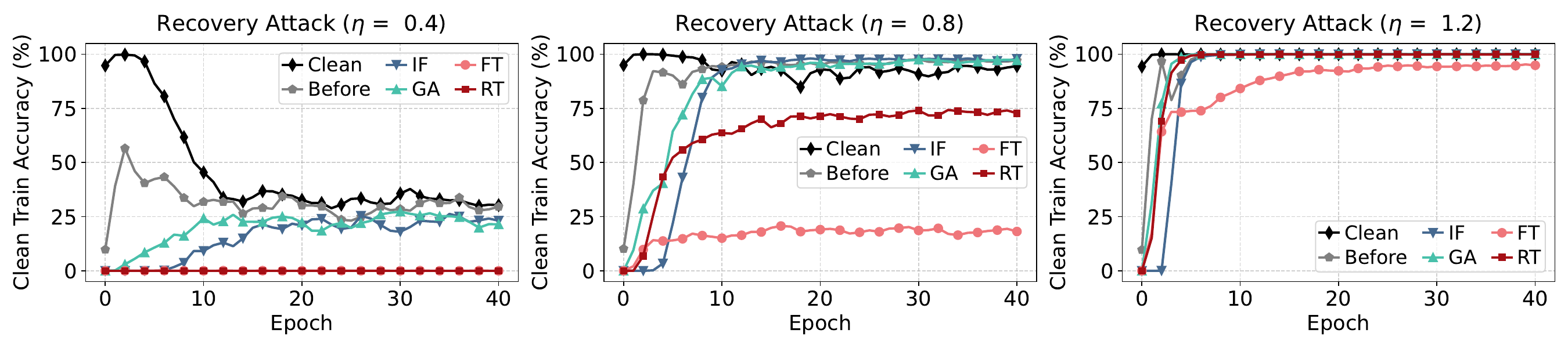}
        \label{fig_Recovery_UE_class0}
    \end{subfigure}
 \vspace{-0.2cm} 
    \begin{subfigure}
        \centering
        \includegraphics[width=0.85\linewidth]{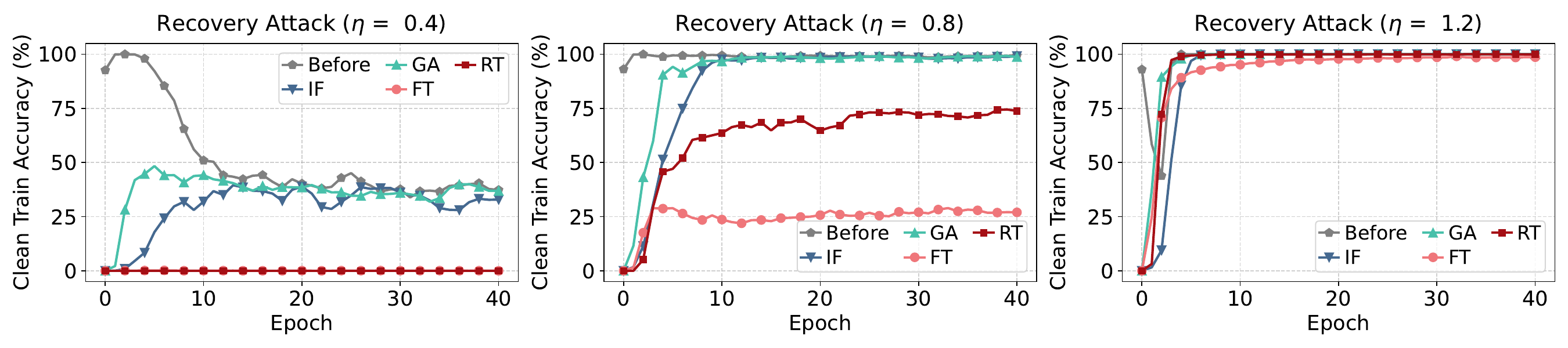}
        \label{fig_Recovery_UE_class5}
    \end{subfigure}
    
    \caption{Recovery attack against unlearned classifiers trained on UE-s from Table~\ref{tab:unlearn_class_acc_on_UE}. Top: Recovery attack against the unlearned model with class 0 (UE-s) removed, ```clean'' represents the pre-unlearning model where class 0 was trained with clean data. Bottom: Recovery attack against the unlearned model with class 5 (clean) removed.}
    \label{fig_Recovery_UE}
\end{figure*}

\begin{figure*}[t]
    \centering
    \includegraphics[width=0.85\linewidth]{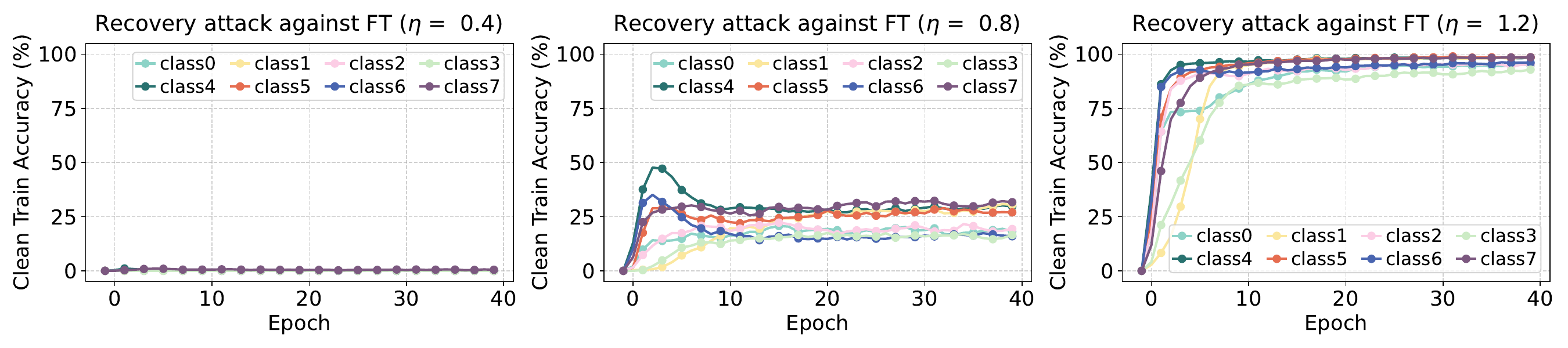}
    \caption{Recovery attack against unlearned classifiers using FT trained on UE-s from Table~\ref{tab:unlearn_class_acc_on_UE}.}
    \label{fig_Recovery_on_FT_UE}
\end{figure*}


\begin{figure*}[t]
    \centering
    \includegraphics[width=.75\linewidth]{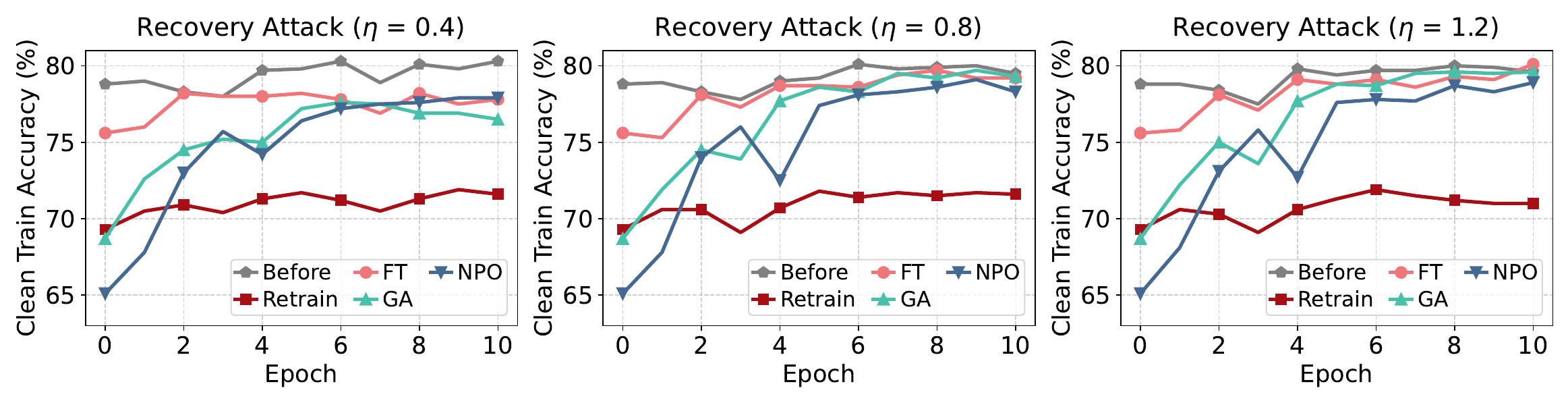}
    \caption{Recovery attack against unlearned models trained on Regtext.}
    \label{fig_Recovery_on_Regtext}
\end{figure*}

\subsubsection{Recovery after Unlearning}\label{subsubsec:unlearning_recover}
We conduct recovery attacks on classifiers trained with UE-s, TUE, PUE, and OPS after class-level unlearning, using the settings in Tables~\ref{tab:unlearn_class_acc_on_UE}--\ref{tab:unlearn_class_acc_on_OPS}. 
We compare them with the RT baseline to assess the depth and stability of the resulting \textit{immemor} state. Results are shown in Figures~\ref{fig_Recovery_UE} and~\ref{fig_Recovery_TUE}--\ref{fig_Recovery_OPS}, where $\eta$ denotes the $\ell_2$ parameter perturbation magnitude.

The results show that IF and GA induce only shallow unlearning for \MethodName-protected classes: recovery attacks restore target-class test accuracy close to the pre-unlearning state and near the level of clean models.
In contrast, FT yields a deeper unlearning effect, with performance remaining close to the RT baseline at $\eta=0.4$ and even lower than RT at $\eta=0.8$. 
Similar trends are observed when recovering forgotten clean classes, as shown at the bottom of Figures~\ref{fig_Recovery_UE} and~\ref{fig_Recovery_TUE}--\ref{fig_Recovery_OPS}. 
Overall, only FT consistently produces deeper forgetting.

We further analyze FT-unlearned classifiers where classes 1--7 are forgotten individually. 
As shown in Figures~\ref{fig_Recovery_on_FT_UE} and~\ref{fig_Recovery_on_FT_TUE}--\ref{fig_Recovery_on_FT_OPS}, FT maintains a deep unlearning effect across all categories, demonstrating stronger robustness to recovery attacks than GA or IF in class-level unlearning. 
This may be attributed to FT using more retain data and update steps, which stabilizes feature memory for remaining classes while inducing catastrophic forgetting of the target class. 
By contrast, the ``negative'' updates in GA and IF are more easily reversed by subsequent ``positive'' learning on target data.

FT remains effective for class-level unlearning but may expose unauthorized knowledge in subset-level settings (Section~\ref{subsubsec:unlearning}). 
Moreover, unlearning a \MethodName-protected class yields deeper forgetting than unlearning a clean class, as most lighter curves fall below darker ones at $\eta=0.8$. 
This suggests that combining \MethodName with unlearning can provide stronger protection than either alone.

\subsubsection{Shallow Dememorization in LLM}
\label{subsec:text_case_study}
We adopt \textsc{RegText}~\cite{shahid2025towards} as the unlearnable-data generator and follow its Natural Instructions (NI) ``Polarity''~\cite{wang2022super} setting, which is constructed from multiple sentiment- and toxicity-related tasks. 
We fine-tune a Mistral-v0.3-7b~\cite{jiang2023mistral7b} model on the \textsc{RegText}-perturbed NI Polarity training set for 10 epochs, and then apply three representative unlearning methods, namely FT, GA, and NPO~\cite{zhang2024negative}, to forget three target tasks. A retrained model on the retained data (RT) serves as the baseline for comparison. 

As summarized in Table~\ref{tab:regtext_unlearn}, all three unlearning methods yield higher post-unlearning accuracy than their pre-unlearning counterparts on Tasks~2 and 3, indicating a ``more memorized'' effect. 
Although Task~1 appears dememorized initially, the recovery curves in Figure~\ref{fig_Recovery_on_Regtext} show that the target information can be restored with modest weight perturbations, suggesting that the \textit{immemor} state induced by unlearning on \textsc{RegText} data is shallow.

Overall, this case study suggests that the interaction between unlearnability and unlearning is not specific to vision models. 
Even in the text domain, \textsc{RegText} can create residual learning signals that are difficult to eliminate, while standard unlearning methods may suffer from both counterproductive memorization and shallow forgetting. 

Based on our findings in Sections~\ref{subsec:interplay}--\ref{subsec:shallow_censorship}, we provide the following guidelines.
\begin{formal}{Guidelines for unlearning-\MethodName synergy.}
\begin{table}[H]
\centering
\caption{Unlearning-\MethodName synergy.}
\label{tab:unlearning_takeaway}
\resizebox{\linewidth}{!}{
\begin{tabular}{cccccccc}
\toprule[1pt]
\midrule[0.3pt]
\multirow{2}{*}{Paradigm} 
& \multicolumn{2}{c}{P-DLC as Poison} 
& \multicolumn{4}{c}{Leaked Clean Knowledge} 
& \multirow{2}{*}{Depth}
  \\
\cmidrule(lr){2-3} \cmidrule(lr){4-7}
& Subset (Acc) & Class (Acc) & Subset (Acc) & Subset (MIA) & Class (Acc) & Class (MIA) & \\
\midrule
IF & \xmark & \cmark & \cmark &\cmark & \cmark & \ding{115} & Shallow \\
GA & \xmark & $\triangle$  & \cmark &\cmark & $\triangle$ & \ding{115} & Shallow \\
FT & \xmark & \cmark & \ding{115} &\ding{115} & \cmark & \ding{115} & Deep (class) \\
\midrule[0.3pt]
\bottomrule[1pt]
\end{tabular}
}
\end{table}
Table~\ref{tab:unlearning_takeaway} summarizes the unlearning effectiveness and privacy risk of three major unlearning paradigms on \MethodName-protected data. 
Here, \cmark denotes effective and stable behavior; $\triangle$ denotes hyperparameter-sensitive effectiveness; \xmark denotes failure; and \ding{115} denotes counterproductive behavior, where accuracy or MIA increases after unlearning. 
Although each method addresses specific objectives, none is universally effective. 
FT generally achieves a deeper \textit{immemor} state in class-level unlearning but remains problematic in subset-level settings involving \MethodName. 
The key takeaway is: \textit{there exists a structural dilemma between eliminating residual knowledge leaked from \MethodName-protected data and achieving a deep \textit{immemor} state}.
\end{formal}

\section{Guarantees against Shallow Dememorization}
To better understand the salient signals underlying shallow dememorization and answer RQ3, we take a first step toward providing a guarantee for model dememorization under weight modifications.
We study the impact of a weight perturbation $\upsilon$ on the effectiveness of dememorization.
This connects naturally to certified learnability~\cite{wang2025provably}, leading to the following theorem.
\begin{restatable}
[\textit{Bounds of dememorization depth}]{theorem}{lcert}\label{theorem:censorship_bound}
Given a trained model $f_{\theta}$ with weights $\theta$, let $A_{\sD^*}(\theta)$ be a metric indicating the performance of $f_{\theta}$ on a target dataset $\sD^*$.
For any $\theta^*$ drawn from the parameter set $\hat{\Theta}:=\{\theta\ |\ \theta\sim\N(\hat{\theta}+\upsilon, \sigma^2I),\ \Vert \upsilon\Vert_2 \leq \eta\}$, with probability no less than $q$, there is:
\begin{equation}\small\label{eq:QPS_bound}
    \begin{aligned}
       & A_{\sD^*}(\theta^*) \leq \inf\ \{t\ |\ \Pr_{\kappa}[A_{\sD^*}(\hat{\theta}+\kappa) \leq t] \geq \overline{q}\}\, and ,\\
       & A_{\sD^*}(\theta^*) \geq \sup\ \{t\ |\ \Pr_{\kappa}[A_{\sD^*}(\hat{\theta}+\kappa) \leq t] \leq \underline{q}\},\ \forall\ \|\upsilon\|\leq\eta,
    \end{aligned}
\end{equation}
where $\overline{q}:=\Phi( \Phi^{-1}(q) + \frac{\eta}{\sigma})$, $\underline{q}:=\Phi( \Phi^{-1}(q) - \frac{\eta}{\sigma})$, and $\kappa\sim\N(0, \sigma^2I)$.
$\Phi(\cdot)$ is the standard Gaussian CDF and $\Phi^{-1}(\cdot)$ is the inverse of the CDF.
$\Vert\upsilon\Vert_2$ is the $\ell_2$ norm of the parameter shift $\upsilon$ from $\hat{\theta}$.
\end{restatable}
\noindent
See Definition 2 and Theorem 1 in~\cite{wang2025provably} for the proof.
Theorem~\ref{theorem:censorship_bound} states that any model with weights $\theta^*$ drawn from a univariate Gaussian centered at $\hat{\theta}+\upsilon$ admits an upper bound on its performance $A_{\sD^*}(\theta^*)$ with probability at least $q$.
This bound is obtained by specifying $q$ and $\eta$, and then computing the corresponding confidence interval of $A_{\sD^*}(\hat{\theta}+\kappa)$ via Monte Carlo.
This is a non-vacuous bound, and the resulting value quantifies the degree of shallow dememorization.
A smaller upper bound indicates that the target knowledge is harder to recover, implying a deeper \textit{immemor} state.

We further consider the case where a model $\theta'$ is $(\epsilon,\zeta)$--in\-dis\-tin\-guish\-able from $\hat{\theta}$ and analyze how the dememorization depth bounds transfer to $\theta'$.
This is especially important for certified unlearning, where stochastic learning algorithms yield distinct weight sets for each run.
Certifying dememorization depth directly on every $\theta'$ is costly.
Instead, \textit{certifying at $\hat{\theta}$ and transferring the bounds to all $(\epsilon,\zeta)$--indistinguishable neighbors offers a tractable solution.}
The following main theorem shows that this transferability is achievable.
\begin{restatable}
[\textit{Transfer of dememorization bounds between indistinguishable models}]{theorem}{certindistinguishablemodels}\label{theorem:censorship_bound_transfer}
If $\theta'$ and $\hat{\theta}$ are $(\epsilon,\zeta)$--indistinguishable, then the following bounds transfer with probability at least\\
$\max\{0,\, e^{-\epsilon}(\Pr_{\theta}[\Vert \theta - \hat{\theta} \Vert_2 \leq \eta ] -\zeta)\}$:
\begin{equation}\small\label{eq:transferred_QPS_bound}
    \begin{aligned}
       & A_{\sD^*}(\theta'^*) \leq \inf\ \{t\ |\ \Pr_{\kappa}[A_{\sD^*}(\hat{\theta}+\kappa) \leq t] \geq \overline{q}\}\, and ,\\
       & A_{\sD^*}(\theta'^*) \geq \sup\ \{t\ |\ \Pr_{\kappa}[A_{\sD^*}(\hat{\theta}+\kappa) \leq t] \leq \underline{q}\},\ \forall\ \|\upsilon\|\leq\eta,
    \end{aligned}
\end{equation}
where $\Vert \theta - \hat{\theta} \Vert_2 \leq \eta$ and $\theta'^* \sim\N(\theta'+\upsilon, \sigma^2I)$ with $\|\upsilon\| \leq \eta$.
\end{restatable}

\begin{proof}
For any indistinguishable $\theta$ and $\theta'$:
\begin{equation}\small\label{eq:indistinguishability}
    \Pr[\theta \in S] \leq e^{-\epsilon}\Pr[\theta' \in S] + \zeta.
\end{equation}
Let $S :=\{\theta: \Vert \theta - \hat{\theta} \Vert_2 \leq \eta\}$ be the $\ell_2$ norm ball centered at $\hat{\theta}$.
Substituting $S$ into Inequality~\ref{eq:indistinguishability}, there is:
\begin{equation}\small
    \Pr[\Vert \theta - \hat{\theta} \Vert_2 \leq \eta] \leq e^{\epsilon}\Pr[\Vert \theta' - \hat{\theta} \Vert_2 \leq \eta] + \zeta.
\end{equation}
Rearranging the above inequality gives:
\begin{equation}\small\label{eq:holding_cond}
    \Pr[\Vert \theta' - \hat{\theta} \Vert_2 \leq \eta ] \geq e^{-\epsilon} ( \Pr[\Vert \theta - \hat{\theta} \Vert_2 \leq \eta ] - \zeta).
\end{equation}
Note that if $\Vert \theta' - \hat{\theta} \Vert_2 \leq \eta$, the Gaussian around $\theta'$ satisfies the dememorization bounds of Equation~\ref{eq:transferred_QPS_bound} with probability at least $q$.
This completes the proof.
\end{proof}
\noindent
Theorem~\ref{theorem:censorship_bound_transfer} supports the dememorization depth certification of any $\theta'$ from certified unlearning based on a fixed $\hat{\theta}$, therefore significantly reducing the cost of shallow dememorization verification in certified unlearning.
\section{Conclusion}\label{sec:conclusion}
In this paper, we present an SoK for both \MethodName and unlearning within a unified model dememorization framework. 
We examine the problem of shallow dememorization and connect upstream \MethodName defenses with downstream unlearning through empirical analysis and parameter-space theory.
Our study reveals how \MethodName influences subsequent unlearning tasks, situating the current status of both fields and bridging two research areas that have thus far evolved largely in isolation.
Ultimately, we aim to highlight future research directions that strengthen the right to withhold knowledge from unauthorized ML models.

\section*{Ethics Considerations}
This work considers the ethical implications of our research on model dememorization using the principles outlined in the Menlo Report: Beneficence, Respect for Persons, Justice, and Respect for Law and Public Interest.

\noindent \textbf{Beneficence.~} 
Our research aims to advance model dememorization studies, minimizing risks to users by identifying and categorizing potential threats. We carefully consider both positive and negative potential impacts, such as improving defense mechanisms while mitigating the risks of data misuse by adversaries.

\noindent \textbf{Respect for persons.~} 
We prioritize transparency and accountability, ensuring our findings serve to empower stakeholders while avoiding harm. No human subjects were involved in this work, and no deceptive practices were employed.

\noindent \textbf{Justice.~}
Our methodology seeks to equitably benefit diverse stakeholders, including researchers, practitioners, and users. 

\noindent \textbf{Respect for law and public interest.~}
We adhere to all applicable laws and ethical standards in conducting our research, ensuring no violation of terms of service or legal frameworks. We disclose findings responsibly to avoid enabling adversarial actions.

\bibliographystyle{ACM-Reference-Format}
\bibliography{others,pdlc,unlearning}

\appendix
\section{Experiment Settings}\label{sec:Experiment_setting}
\noindent\textbf{Unlearnable data generation.~}
In our experiments, we adopt CIFAR-10, CIFAR-100, and ImageNet with 100 randomly selected categories as the datasets in our evaluation.
We adopt the default settings from the original papers to generate PUE, OPS (2 pixel), UE-c (class-wise), TUE, UE-s (sample-wise), and TAP noises.

\noindent\textbf{Assessing \MethodName efficacy against unauthorized learning.~}
We employ ResNet-18, DenseNet-121, and ViT-Tiny as unauthorized classifiers. ResNet-18 and DenseNet-121 are trained using the SGD optimizer, while ViT-Tiny is trained using AdamW. Except for the experiment in Figure~\ref{fig_test_vit_no_aug_bar}, which evaluates the test accuracy of ViT-Tiny trained on P-DLC data both with and without augmentation, all other training settings use data augmentation by default. The detailed training settings are summarized in Table~\ref{tab:Experimental training settings}, which are kept consistent for both the clean and P-DLC training setups.

\begin{table}[H]
    \centering
    \caption{Experimental training settings.}
    \resizebox{1\linewidth}{!}{%
    \begin{tabular}{cccccccc}
    \toprule[1pt]
    \midrule[0.3pt]
   Model   &      train Set    &  epochs    &  lr &  weight decay &  momentum  &    augment     \\
    \midrule

        \multirow{3}{*}{ResNet-18} 
                & CIFAR-10 & 60 & 0.1 &   $5\times10^{-4}$& 0.9 &  RandomCrop, RandomHorizontalFlip \\
                & CIFAR-100  & 100 & 0.1&   $5\times10^{-5}$& 0.9 & RandomCrop,RandomHorizontalFlip,RandomRotation \\
                & ImageNet & 100 & 0.1&  $5\times10^{-5}$& 0.9 & RandomResizedCrop,RandomHorizontalFlip,ColorJitter \\

    \midrule
         \multirow{3}{*}{DenseNet-121} 
                &  CIFAR-10 & 60 & 0.1 &   $5\times10^{-4}$& 0.9 &  RandomCrop, RandomHorizontalFlip \\
                & CIFAR-100  & 100 & 0.1&  $5\times10^{-5}$& 0.9 & RandomCrop,RandomHorizontalFlip,RandomRotation \\
                & ImageNet & 100 & 0.1&  $5\times10^{-5}$& 0.9 & RandomResizedCrop,RandomHorizontalFlip,ColorJitter \\

  \midrule
         \multirow{3}{*}{ViT-Tiny} 
                &  CIFAR-10 & 60 & $3\times10^{-4}$ & 0.05 & - &  RandomCrop, RandomHorizontalFlip \\
                &  CIFAR-100 & 100 & $3\times10^{-4}$ & 0.05 & - & RandomCrop,RandomHorizontalFlip,RandomRotation \\
                &  ImageNet & 150 &  $1\times10^{-3}$ & 0.1 & - & RandomResizedCrop,RandomHorizontalFlip,ColorJitter \\
    \midrule[0.3pt]
    \bottomrule[1pt]    
    \end{tabular}
    }
    \label{tab:Experimental training settings}
\end{table}

\noindent\textbf{Recovery of models trained on \MethodName-protected data.~}
We evaluate the empirical robustness of PUE, OPS, UE-c, TUE, UE-s, and TAP on ResNet-18 using recovery attacks across CIFAR-10, CIFAR-100, and ImageNet. 
For ImageNet training with P-DLC noise, we use 100\% of the first 100 classes as the class-wise P-DLC training set, and 20\% of the first 100 classes as the sample-wise P-DLC training set (for efficiency). 
For CIFAR-10, CIFAR-100, and ImageNet, we conduct recovery attacks using 10\% of their training sets. 
Attack effectiveness is then evaluated on the corresponding test sets of the three datasets. The detailed experimental settings and parameters are shown in Table~\ref{tab:Experimental settings for the recovery of P-DLC-trained models}.

\begin{table}[H]
    \centering
    \caption{Experimental settings for the recovery of P-DLC-trained models.}
    \resizebox{1\linewidth}{!}{%
    \begin{tabular}{cccccccc}
    \toprule[1pt]
    \midrule[0.3pt]
      &      train Set    &  epochs    &  lr &  weight decay &  momentum  &    augment     \\
    \midrule
         \multirow{3}{*}{recovery} 
                &   CIFAR-10 & 60 & 0.01 &  $5\times10^{-4}$ & 0.9 &  RandomCrop, RandomHorizontalFlip \\
                &  CIFAR-100 & 100 & 0.01&   $5\times10^{-5}$ & 0.9 & RandomCrop,RandomHorizontalFlip,RandomRotation \\
                &  ImageNet & 100 & 0.1&  $5\times10^{-5}$& 0.9 & RandomResizedCrop,RandomHorizontalFlip,ColorJitter \\
    \midrule[0.3pt]
    \bottomrule[1pt]    
    \end{tabular}
    }
    
    \label{tab:Experimental settings for the recovery of P-DLC-trained models}
\end{table}

\noindent\textbf{Recovery of unlearned models.~}
The recovery attack experiments are conducted on unlearned models obtained from the class-level unlearning experiments described in Section~\ref{subsection:Censorship Interaction}. We select classes for which the IF, RT, and FT methods all achieve effective forgetting while preserving model utility as recovery targets. The corresponding unlearned models serve as the attack targets. Recovery attacks are performed using 10\% of clean training data from the target class. For detailed parameters, see Table~\ref{tab:Experimental settings for the recovery of P-DLC-trained models}.

\begin{table}[H]
    \centering
    \small
    \caption{Experimental settings for subset-level unlearning.}
    \resizebox{1\linewidth}{!}{%
    \begin{tabular}{c|cc|cc|cc|cc|cc|cc}
    \toprule[1pt]
    \midrule[0.3pt]
    \multirow{3}{*}{P-DLC} 
        & \multicolumn{4}{c}{IF ($\alpha$)} 
        & \multicolumn{4}{c}{FT (lr)} 
        & \multicolumn{4}{c}{GA (lr)} \\
 
          & \multicolumn{2}{c}{$D_u$} & \multicolumn{2}{c|}{$D_c$} 
           & \multicolumn{2}{c}{$D_u$} & \multicolumn{2}{c|}{$D_c$}
           & \multicolumn{2}{c}{$D_u$} & \multicolumn{2}{c}{$D_c$} \\
    \cmidrule(lr){2-13}
          & $\bigtriangledown$ & $\bigtriangleup$ & $\bigtriangledown$ & $\bigtriangleup$
           & $\bigtriangledown$ & $\bigtriangleup$ & $\bigtriangledown$ & $\bigtriangleup$
           & $\bigtriangledown$ & $\bigtriangleup$ & $\bigtriangledown$ & $\bigtriangleup$ \\
    \midrule
        TUE  &  0.5 & 5 & 0.5 & 5 & $5\times10^{-3}$  &  $6\times10^{-2}$ &  $5\times10^{-3}$ & $2.5\times10^{-2}$  & $3\times10^{-5}$  &  $6\times10^{-5}$ & $3\times10^{-5}$  &  $1.5\times10^{-4}$ \\
    \midrule
        UE-s & 0.3 & 3  &  0.3 & 3 & $5\times10^{-3}$ & $6\times10^{-2}$ & $5\times10^{-3}$ & $2\times10^{-2}$ &  $3\times10^{-5}$ & $3.5\times10^{-5}$  & $3\times10^{-5}$  & $8\times10^{-5}$  \\
    \midrule
        PUE & 0.6 & 6 & 0.6 & 6  &  $5\times10^{-3}$ & $6\times10^{-2}$  &  $5\times10^{-3}$ & $2.5\times10^{-2}$ & $3\times10^{-5}$  &  $8\times10^{-5}$ & $3\times10^{-5}$  & $2\times10^{-4}$  \\
    \midrule
        OPS &   0.3  &  3 &   0.3  & 3  &  $5\times10^{-3}$ & $6\times10^{-2}$  &  $5\times10^{-3}$ & $2.5\times10^{-2}$  & $3\times10^{-5}$  & $1\times10^{-4}$  & $3\times10^{-5}$  & $1.5\times10^{-4}$  \\
    \midrule[0.3pt]
    \bottomrule[1pt]    
    \end{tabular}
    }
    \label{tab:Experimental Settings for Subset-Level Unlearning}
\end{table}

\begin{table}[H] 
    \centering
    \small
    \caption{Experimental settings for class-level unlearning.}
    \resizebox{0.65\linewidth}{!}{%
    \begin{tabular}{c|cc|cc|cc}
    \toprule[1pt]
    \midrule[0.3pt]
    \multirow{2}{*}{P-DLC} 
        & \multicolumn{2}{c|}{IF ($\alpha$)} 
        & \multicolumn{2}{c|}{FT (lr)} 
        & \multicolumn{2}{c}{GA (lr)} \\
    \cmidrule(lr){2-7}
        & $\bigtriangledown$ & $\bigtriangleup$ 
        & $\bigtriangledown$ & $\bigtriangleup$
        & $\bigtriangledown$ & $\bigtriangleup$ \\
    \midrule
        TUE   & 2 & 7 & $1.5\times10^{-2}$ & $3\times10^{-2}$ & $1.5\times10^{-5}$ & $1.1\times10^{-4}$   \\
    \midrule
        UE-s    & 2 & 7 & $1.5\times10^{-2}$ & $3\times10^{-2}$  & $1.1\times10^{-5}$  & $2.6\times10^{-5}$  \\
    \midrule
        PUE   & 2 & 7 & $1.5\times10^{-2}$ & $3\times10^{-2}$  &  $1.7\times10^{-5}$ & $4.8\times10^{-5}$   \\
    \midrule
        OPS   & 2 & 7 & $1.5\times10^{-2}$ & $3\times10^{-2}$  & $1.7\times10^{-5}$  & $6.7\times10^{-5}$  \\
    \midrule[0.3pt]
    \bottomrule[1pt]    
    \end{tabular}
    }
    \label{tab:Experimental Settings for Class-Level Unlearning}
\end{table}

\begin{figure*}[t]
    \centering
    
    \begin{subfigure}
        \centering
        \includegraphics[width=0.8\linewidth]{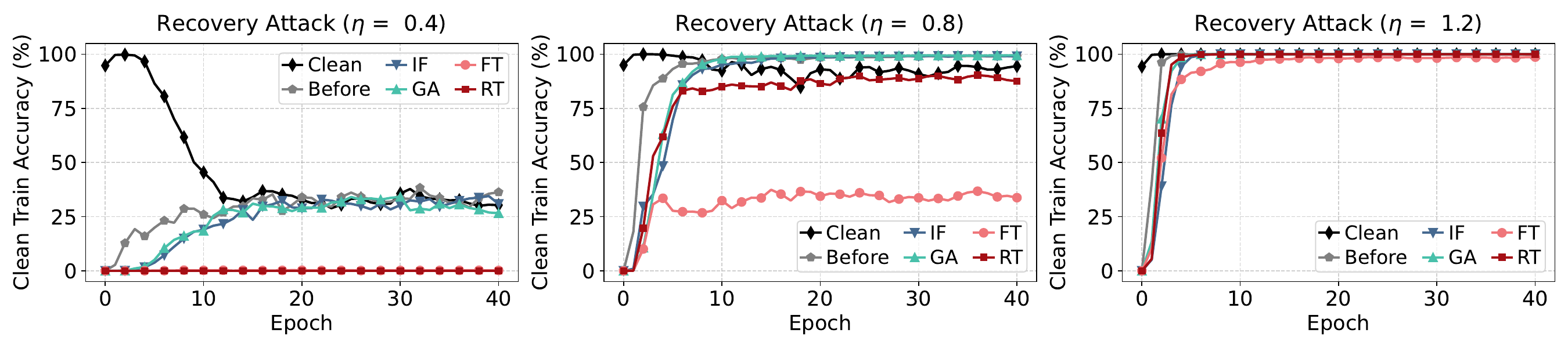}
        \label{fig_Recovery_TUE_class0}
    \end{subfigure}
    \vspace{-0.2cm} 
    \begin{subfigure}
        \centering
        \includegraphics[width=0.8\linewidth]{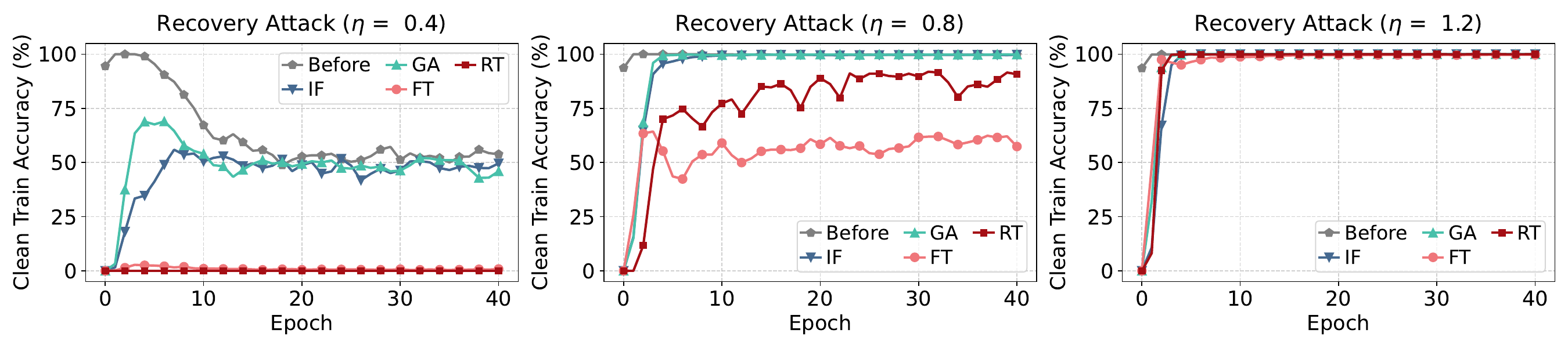}
        \label{fig_Recovery_TUE_class5}
    \end{subfigure}
    
    \caption{Recovery attacks against unlearned classifiers trained on TUE from Table~\ref{tab:unlearn_class_acc_on_TUE}. Top: Recovery attack against the unlearned model with class 0 (TUE) removed. ```clean'' represents the pre-unlearning model where class 0 was trained with clean data. Bottom: Recovery attack against the unlearned model with class 5 (clean) removed.}
    \label{fig_Recovery_TUE}
\end{figure*}

\begin{figure*}[t]
    \centering
    
    \begin{subfigure}
        \centering
        \includegraphics[width=0.8\linewidth]{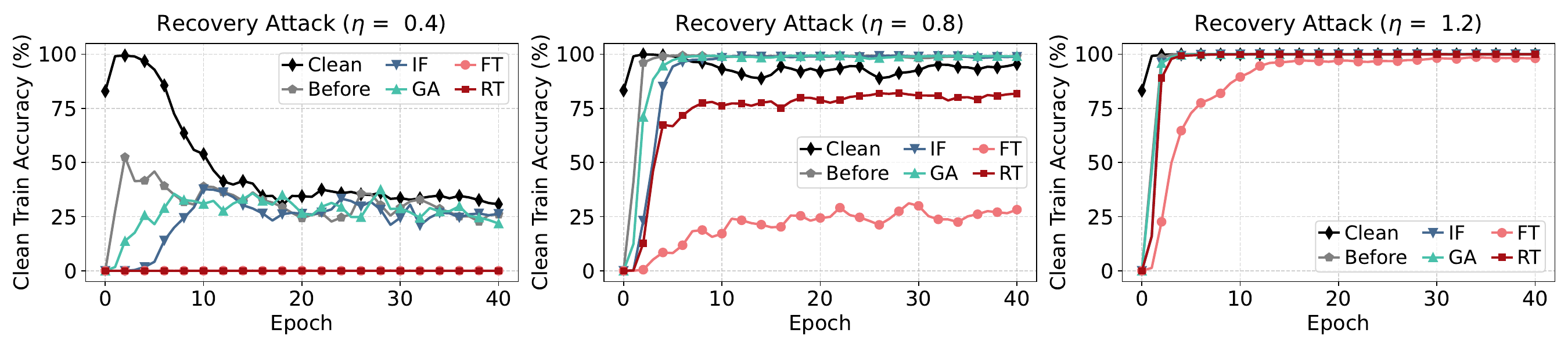}
        \label{fig_Recovery_PUE_class0}
    \end{subfigure}
    \vspace{-0.2cm}
    \begin{subfigure}
        \centering
        \includegraphics[width=0.8\linewidth]{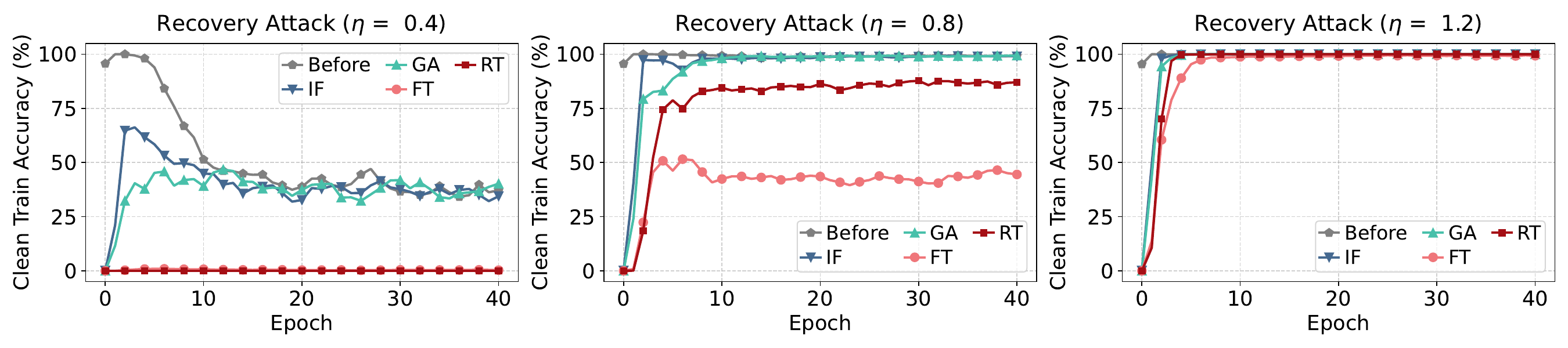}
        \label{fig_Recovery_PUE_class5}
    \end{subfigure}
    
    \caption{Recovery attacks against unlearned classifiers trained on PUE from Table~\ref{tab:unlearn_class_acc_on_PUE}. Top: Recovery attack against the unlearned model with class 3 (PUE) removed. ```clean'' represents the pre-unlearning model where class 3 was trained with clean data. Bottom: Recovery attack against the unlearned model with class 6 (clean) removed.}
    \label{fig_Recovery_PUE}
\end{figure*}

\begin{figure*}[t]
    \centering
    
    \begin{subfigure}
        \centering
        \includegraphics[width=0.8\linewidth]{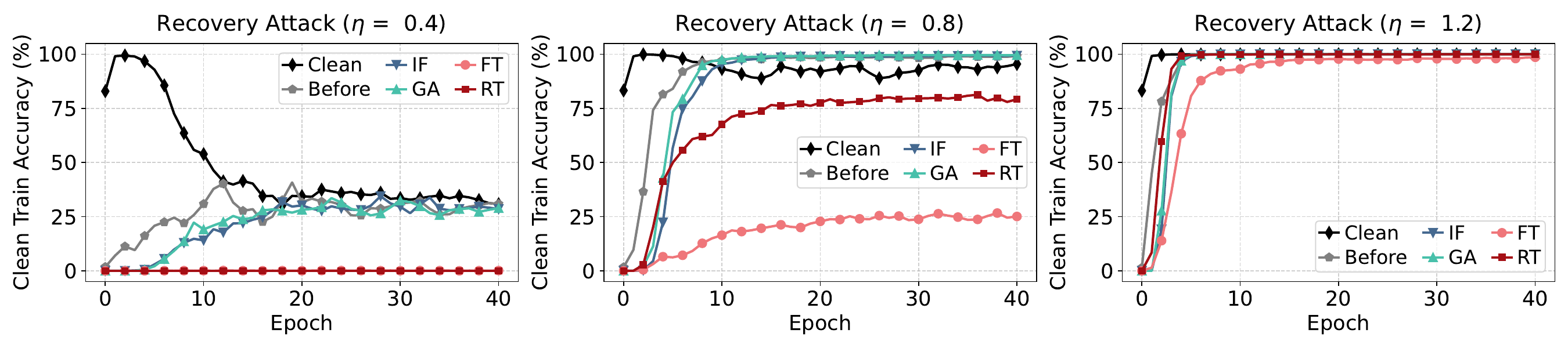}
        \label{fig_Recovery_OPS_class0}
    \end{subfigure}
    \vspace{-0.2cm}
    \begin{subfigure}
        \centering
        \includegraphics[width=0.8\linewidth]{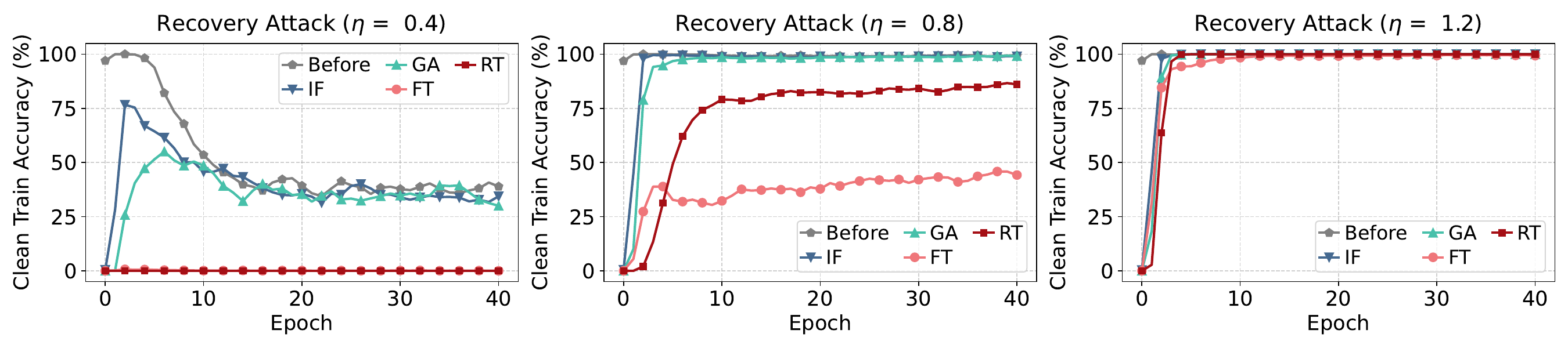}
        \label{fig_Recovery_OPS_class5}
    \end{subfigure}
    
    \caption{Recovery attacks against unlearned classifiers trained on OPS from Table~\ref{tab:unlearn_class_acc_on_OPS}. Top: Recovery attack against the unlearned model with class 3 (OPS) removed. ```clean'' represents the pre-unlearning model where class 3 was trained with clean data. Bottom: Recovery attack against the unlearned model with class 6 (clean) removed.}
    \label{fig_Recovery_OPS}
\end{figure*}

\begin{figure*}[t]
    \centering
    \includegraphics[width=0.8\linewidth]{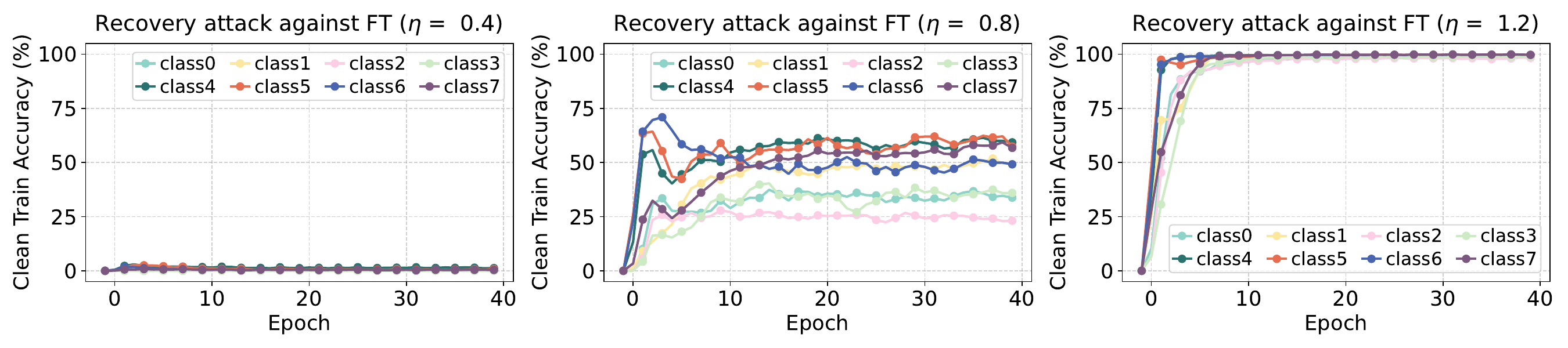}
    \caption{Recovery attack against unlearned classifiers using FT trained on TUE from Table~\ref{tab:unlearn_class_acc_on_TUE}.}
    \label{fig_Recovery_on_FT_TUE}
\end{figure*}

\begin{figure*}[t]
    \centering
    \includegraphics[width=0.8\linewidth]{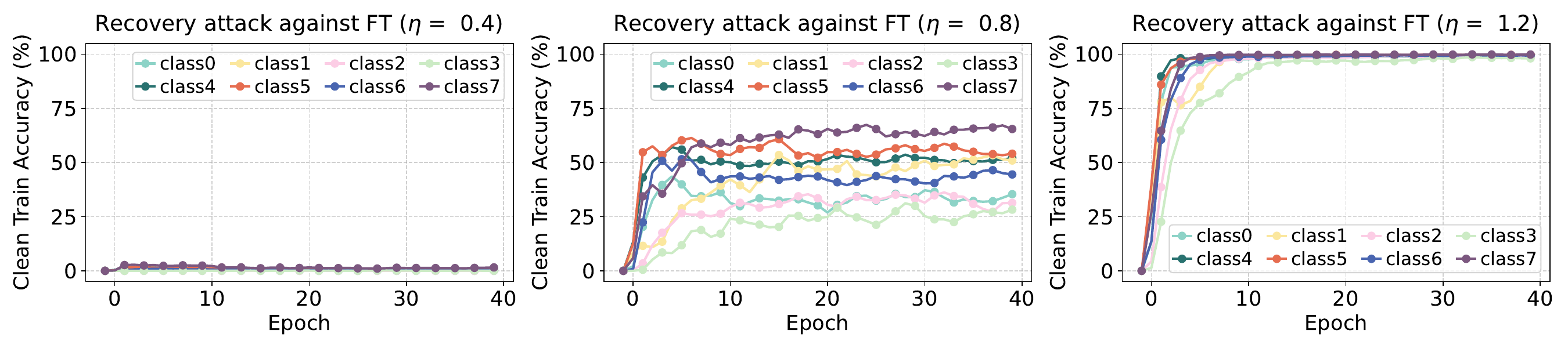}
    \caption{Recovery attack against unlearned classifiers using FT trained on PUE from Table~\ref{tab:unlearn_class_acc_on_PUE}.}
    \label{fig_Recovery_on_FT_PUE}
\end{figure*}

\begin{figure*}[t]
    \centering
    \includegraphics[width=0.8\linewidth]{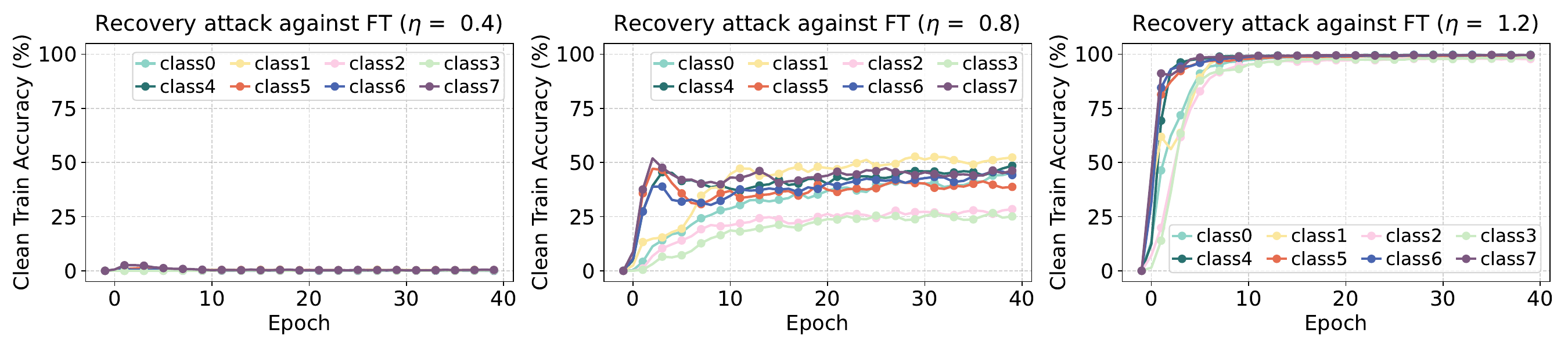}
    \caption{Recovery attack against unlearned classifiers using FT trained on OPS from Table~\ref{tab:unlearn_class_acc_on_OPS}.}
    \label{fig_Recovery_on_FT_OPS}
\end{figure*}

\section{Additional Experiment Results}\label{sec:additional_results}
This section presents additional experimental results.

\noindent\textbf{The impact of poisoning ratio.~}
We test the impact of poisoning ratio on the effectiveness of \MethodName. 
For class-wise \MethodName method, applying a single perturbation per class is appealing for practical deployment. 
However, Figure~\ref{fig_test_various_train_percent} shows that the amount of \MethodName data may affect unlearnability. 
For example, UE-c requires sufficient coverage of perturbed ImageNet data to achieve adequate protection, while PUE remains more robust under reduced perturbation coverage. 
In contrast, unauthorized models can absorb more clean knowledge during training as more OPS-protected data are included, since pixel-level perturbations on high-resolution images are easily weakened by data augmentation, such as random cropping.

\begin{figure}[t]
    \centering
    \includegraphics[width=.55\linewidth]{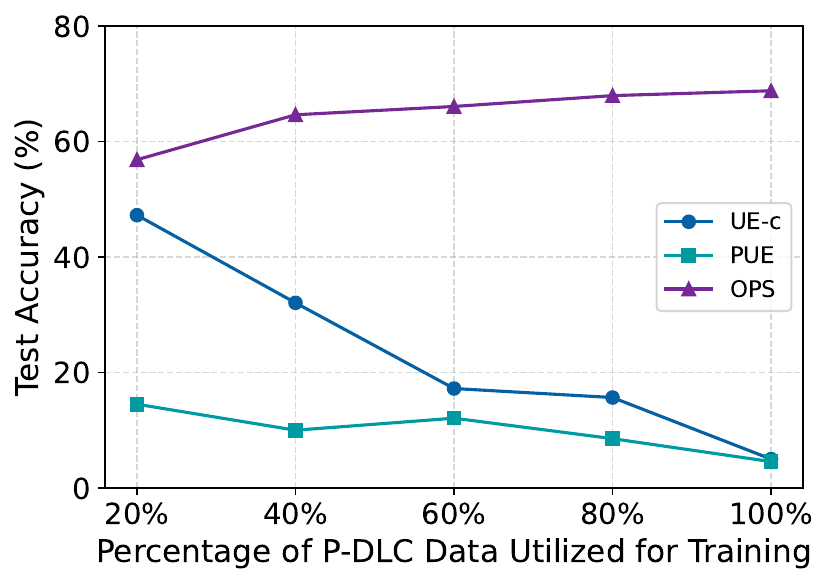}
    \caption{The test accuracy (\%) of ResNet-18 trained on varying amounts of ImageNet-based classwise P-DLC data.}
    \label{fig_test_various_train_percent}
\end{figure}

\noindent\textbf{Unlearnability-unlearning interaction.~}
Class-level unlearning results on models trained with UE-s, TUE, PUE,and OPS are shown in Tables~\ref{tab:unlearn_class_acc_on_UE}--~\ref{tab:unlearn_class_acc_on_OPS}, respectively. 
The corresponding subset-level unlearning results are reported in Tables~\ref{tab:unlearn_acc_on_UE}--~\ref{tab:unlearn_acc_on_OPS}.

\noindent\textbf{Evaluation of shallow dememorization.~} Additional results of recovery attacks on unlearned classifiers trained on TUE, PUE, and OPS are shown in Figures~\ref{fig_Recovery_TUE}--~\ref{fig_Recovery_OPS}, respectively. 
Figures~\ref{fig_Recovery_on_FT_TUE}--~\ref{fig_Recovery_on_FT_OPS} report the results of recovery attacks against unlearned classifiers using FT across classes 1--7.

\begin{table*}[t]
    \centering
    \caption{The accuracy (\%) of unlearned classifiers trained on UE-s (subset-level).}
    \resizebox{\linewidth}{!}{%
    \begin{tabular}{cccccc|cccc|cccc|cccc|cccc|cccc|cccc|cccc}
    \toprule[1pt]
    \midrule[0.3pt]
    \multirow{2}{*}{Forgetting Set}&\multirow{2}{*}{Method}& \multicolumn{4}{c}{Class 0} & \multicolumn{4}{c}{Class 1} & \multicolumn{4}{c}{Class 2}& \multicolumn{4}{c}{Class 3}& \multicolumn{4}{c}{Class 4} & \multicolumn{4}{c}{Class 5} & \multicolumn{4}{c}{Class 6} & \multicolumn{4}{c}{Class 7} \\  
    \cmidrule(lr){3-34}
      &    & $D_u$ & $D_{uc}$ & $D_c$ & $D_t$ & $D_u$ & $D_{uc}$ & $D_c$ & $D_t$& $D_u$ & $D_{uc}$ & $D_c$ & $D_t$& $D_u$ & $D_{uc}$ & $D_c$ & $D_t$& $D_u$ & $D_{uc}$ & $D_c$ & $D_t$& $D_u$ & $D_{uc}$ & $D_c$ & $D_t$& $D_u$ & $D_{uc}$ & $D_c$ & $D_t$& $D_u$ & $D_{uc}$ & $D_c$ & $D_t$ \\     
    \midrule
    \multirow{7}{*}{$D_u$}
    & RT &8.79 & 72.70 & 100 & 74.59 & 2.70 & 83.88 & 100 & 72.87 & 13.14 & 60.99 & 100 & 73.79 & 0.03 & 58.67 & 99.78 & 74.07 & 29.70 & 64.99 & 99.78 & 72.26 & 9.68 & 63.48 & 99.78 & 73.49 & 8.69 & 75.93 & 100 & 72.39 & 2.79 & 73.06 & 100 & 72.70 \\  
    \cmidrule(lr){2-34}
    & {IF$\bigtriangledown$} & 100 & \textbf{71.78} & \textbf{100} & \textbf{74.29} & 100 & \textbf{82.35} & \textbf{99.33} & \textbf{74.37} & 99.93 & \textbf{53.16} & \textbf{98} & \textbf{74.18} & 100 &  \textbf{52.72} &  \textbf{100} &  \textbf{74.41} & 99.90 & \textbf{61.43} & \textbf{99.50} & \textbf{73.98} & 99.90 & \textbf{52.54} & \textbf{97.78} & \textbf{74.24} & 99.98 & \textbf{72.30} & \textbf{99.78} & \textbf{73.94} & 100 & \textbf{73.06} & \textbf{100} & \textbf{74.50}  \\ 
    & {IF$\bigtriangleup$}   & 72.86 & 25.11 & 50.67 & 67.14 & 15.28 & 34.30 & 45.33 & 63.17 & 0.42 & 0 & 0 & 65.63 & 81.06 & 15.58 & 52.67 & 69.59 & 1.80 & 0.05 & 0 & 56.38 & 2.42 & 0.07 & 0 & 62.41 & 11.83 & 2.42 & 2.22 & 63.07 & 91.95 & 35.09 & 67.56 & 68.72 \\ 
    \cmidrule(lr){2-34}
    & {FT$\bigtriangledown$} & 99.93 & 74.47 & 100 & 73.89 & 99.95 & 85.31 & 100 & 73.70 & 99.98 & 66.57 & 99.78 & 74.11 & 99.98 & 58.32 & 100 & 73.86 & 100 & 72.40 & 99.78 & 73.53 & 99.85 & \textbf{60.44} & \textbf{98.89} & \textbf{74.02} & 99.98 & 83.26{\color{red} $\uparrow$} & 100 & 73.83 & 99.98 & 77.04{\color{red} $\uparrow$} & 100 & 74.19 \\ 
    & {FT$\bigtriangleup$}  &2.91 & 59.11 & 70.89 & 64.40 & 21.75 & 89.16{\color{red} $\uparrow$} & 93.33 & 64.13 & 13.26 & 40.07 & 41.56 & 62.71 & 0.03  & 47.24 & 57.11 & 66.99 & 27.06 & 43.01 & 52.44 & 61.70 & 0.15 & 0.30 & 0.44 & 56.85 & 16.17 & 73.98 & 83.11 & 65.68 & 7.65 & 72.82 & 82.22 & 67.98\\ 
    \cmidrule(lr){2-34}
    & {GA$\bigtriangledown$} & 99.98 & \textbf{70.49} & \textbf{100} & \textbf{74.31} & 99.85 & \textbf{78.15} & \textbf{99.11} & \textbf{74.07} & 91.98 & 25.58 & 58 & 72.02 & 100 & \textbf{51.26} & \textbf{100} & \textbf{74.44} & 98.12 & \textbf{43.11} & \textbf{88.67} & \textbf{71.73} & 95.28 & 28.1 & 74.67 & 71.78 & 99.93 & \textbf{67.19} & \textbf{98.44} & \textbf{73.57} & 100 & \textbf{73.98} & \textbf{100} & \textbf{74.55} \\ 
    & {GA$\bigtriangleup$}  & 99.93 & \textbf{68.82} & \textbf{100} & \textbf{74.31} & 99.95 & \textbf{69.68} & \textbf{96.89} & \textbf{73.39} & 0 & 0 & 0 & 19.21 & 100 & \textbf{49.38} & \textbf{99.33}  & \textbf{74.46}  & 0 & 0 & 0 &  17.28 & 0 & 0 & 0 & 32.55 & 98.84 & \textbf{55.85} & \textbf{93.11} & \textbf{72.44} & 100 & \textbf{73.11} & \textbf{100} & \textbf{74.48} \\ 
    \midrule
    \multirow{7}{*}{$D_c$}
    & RT    & 100 & 5.36 & 7.11 & 67.12 & 99.95 & 3.31 & 3.56 & 63.88 & 99.95 & 3.65 & 3.55 & 68.28 & 99.93 & 0.17 & 0.22 & 68.98 & 99.95 & 3.58 & 2 & 66.22 & 100 & 0.49 & 0.44 & 67.47 & 100 & 6.54 & 5.78 & 65.40 & 99.98 & 0.35 & 0 & 64.44 \\  
    \cmidrule(lr){2-34}
    & {IF$\bigtriangledown$} &  100 & 67.85 & 100 & 74.19 & 100 & 82.35 & 99.33 & 74.37 & 99.88 & 45.26 & 93.78 & 73.74 & 99.98 & 15.09 & 49.78 & 70.77 & 99.88 & 55.90 & 98.22 & 73.76 & 99.58 & 27.43 & 72 & 72.09 & 99.98 & 67.85 & 99.33 & 73.61 & 100 & 62.12 & 98.89 & 73.67 \\ 
    & {IF$\bigtriangleup$}   &  76.44 & \textbf{3.85} & \textbf{5.56} & \textbf{62.13} & 15.28 & 34.3 & 45.33 & 63.17 & 0.12 & \textbf{0} & \textbf{0} & \textbf{63.05} & 1.95 & 0 & 0 & 13.17 & 0.32 & \textbf{0} & \textbf{0} & \textbf{57.27} & 0 & 0 & 0 & 33.01 & 57.7 & \textbf{0} & \textbf{0} & \textbf{60.84} & 86.05 & \textbf{0.03} & \textbf{0} & \textbf{61.13} \\ 
    \cmidrule(lr){2-34}
    & {FT$\bigtriangledown$} & 100 & 42.32 & 70.67 & 69.02 & 100 & 61.58 & 78.22 & 70.16 & 99.93 & 37.48 & 56.44 & 69.59 & 99.61 & 1.73 & 3.78 & 65.94 & 99.98 & 44.32 & 72.22 & 70.66 & 99.53 & 8.15 & 15.78 & 67.21 & 99.98  & 45.16 & 58.67 & 70.07 & 100 & 27.51 & 37.33 & 67.49 \\ 
    & {FT$\bigtriangleup$}  & \textbf{99.46} &\textbf{ 7.78} &  \textbf{11.56} & \textbf{64.06} & \textbf{99.06} & \textbf{6.96} & \textbf{7.56} & \textbf{61.07} & \textbf{99.56} & \textbf{5.65} & \textbf{3.78} & \textbf{63.57} & \textbf{97.95} & \textbf{0.07} & \textbf{0.44} & \textbf{58.73} & \textbf{99.36} & \textbf{8.25} & \textbf{9.56} & \textbf{64.29} & \textbf{99.38} & \textbf{1.11} & \textbf{1.78} & \textbf{63.89} & \textbf{99.90} & \textbf{12.64} & \textbf{12.89} & \textbf{63.05} & \textbf{99.68} & \textbf{0.62} & \textbf{0.67} &  \textbf{51.94} \\ 
    \cmidrule(lr){2-34}
    & {GA$\bigtriangledown$} & 100 & 72.62 & 100 & 74.45 & 100 & 84.54 & 99.78 & 74.60 & 99.98 & 57.78 & 98.89 & 74.34 & 100 & 20.20 & 68.22 & 71.78 & 99.98 & 65.53 & 99.78 & 74.52 & 99.90 & 41.11 & 91.56 & 73.22 & 100 & 72.94 & 100 & 74.02 & 100 & 66.67 & 99.56 & 74.02 \\ 
    & {GA$\bigtriangleup$}  &99.93 & 61.41 & 99.56 & 73.47 & 100 & 81.33 & 99.33 & 74.48 & 78.22 & \textbf{4.89} & \textbf{7.11} & \textbf{66.54} & 0 &  0 & 0 & 10 & 94.74 & 23.16  & 51.78 & 70.47 & 0 & 0 & 0 & 20.22 & 98.62 & 18.35 & 28.22 & 68.17 & \textbf{95.93} &\textbf{ 0.79} & \textbf{0.44} & \textbf{63.24}\\ 
    \midrule
    Before  &  -  & 100 & 76.99 & 100 & 74.59 & 100 & 85.93 & 100 & 74.59 & 100 & 68.42 & 99.56 & 74.59 & 100 & 58.15 & 100 & 74.59 & 100 & 74.62 & 100 & 74.59 & 99.98 & 67.78 & 100 & 74.59 & 100 & 80.52 & 100 & 74.59 & 100 & 76.69 & 100 & 74.59 \\ 
    \midrule[0.3pt]
    \bottomrule[1pt]
    \end{tabular}
    }
    
    \label{tab:unlearn_acc_on_UE}
\end{table*}

\begin{table*}[t]
    \centering
    \centering
    \caption{The accuracy (\%) of unlearned classifiers trained on TUE (subset-level).}
    \resizebox{\linewidth}{!}{%
    \begin{tabular}{cccccc|cccc|cccc|cccc|cccc|cccc|cccc|cccc}
    \toprule[1pt]
    \midrule[0.3pt]
    \multirow{2}{*}{Forgetting Set}&\multirow{2}{*}{Method}& \multicolumn{4}{c}{Class 0} & \multicolumn{4}{c}{Class 1} & \multicolumn{4}{c}{Class 2}& \multicolumn{4}{c}{Class 3}& \multicolumn{4}{c}{Class 4} & \multicolumn{4}{c}{Class 5} & \multicolumn{4}{c}{Class 6} & \multicolumn{4}{c}{Class 7} \\  
    \cmidrule(lr){3-34}
      &    & $D_u$ & $D_{uc}$ & $D_c$ & $D_t$ & $D_u$ & $D_{uc}$ & $D_c$ & $D_t$& $D_u$ & $D_{uc}$ & $D_c$ & $D_t$& $D_u$ & $D_{uc}$ & $D_c$ & $D_t$& $D_u$ & $D_{uc}$ & $D_c$ & $D_t$& $D_u$ & $D_{uc}$ & $D_c$ & $D_t$& $D_u$ & $D_{uc}$ & $D_c$ & $D_t$& $D_u$ & $D_{uc}$ & $D_c$ & $D_t$ \\     
    \midrule
    \multirow{7}{*}{$D_u$}
    & RT &  2.00  &  78.84  &  100  &  80.68  &  30.91  &  85.63  &  100  & 80.49 &  8.69  &  70.94  &  100  &  81.64  & 6.72   &  69.68  &  100  &  81.93  &  3.78  & 76.37   & 100   &  80.91  &  9.26  &  71.70  & 100   &  80.78  &  8.22  &  81.19  &  100  &  81.54  & 26.37 &80.62 & 100& 80.86 \\  
    \cmidrule(lr){2-34}
    & {IF$\bigtriangledown$} &  99.93  &  \textbf{75.78}  &  \textbf{99.78}  &  \textbf{81.37}  &  100  &  88.54  &  99.78  &  81.39  &  99.95  &  \textbf{66.35}  & \textbf{ 99.78}  &  \textbf{81.29}  &  100  &  \textbf{59.63}  & \textbf{100} &  \textbf{81.39}  &  100  &  \textbf{74.37}  &  \textbf{99.78}  & \textbf{ 81.50}  &  100  &  \textbf{66.74 } &  \textbf{100}  &  \textbf{81.43}  &  99.98  &   \textbf{77.85}  & \textbf{ 99.78 } & \textbf{ 81.15 } &  99.98 & \textbf{ 79.33 } & \textbf{ 99.78 } & \textbf{81.43} \\ 
    & {IF$\bigtriangleup$}   &  21.19  &  21.19  &  32  &  75.21  & 7.95 & 13.28 & 16.22 & 70.24 & 43.46 & 12.99 & 19.11 & 74.39 & 4.17 & 2.30 & 4 & 73.08 & 18.40 & 3.14 & 1.78 & 71.41 & 77.19 & 11.14  & 17.33 & 75.54 & 0 & 0 & 0 & 66.52 & 2.77 & 2.35 & 2.44 & 67.88 \\ 
    \cmidrule(lr){2-34}
    & {FT$\bigtriangledown$} & 100 & 86.91{\color{red} \large$\uparrow$} & 100 & 81.72 & 100 & 94.25{\color{red} \large$\uparrow$} & 100 & 81.33 & 100 & 74.94{\color{red} \large$\uparrow$} & 100 & 81.67 & 100 & 69.93{\color{red} \large$\uparrow$} & 100 & 81.92 & 100 & 81.61 & 100 & 81.64{\color{red} \large$\uparrow$} & 100 & 74.12 & 100 & 81.72 & 100 & 86.59 & 100 & 81.81 & 100 & 85.06 & 100 & 81.88 \\ 
    & {FT$\bigtriangleup$}   & \textbf{3.11} & \textbf{74} &  \textbf{83.78} &  \textbf{70.09} & 69.48 & 96.07{\color{red} \large$\uparrow$} & 99.11 & 67.58 & 25.11 & \textbf{69.95} & \textbf{81.78} & \textbf{69.69} & 2.27 & 76.42{\color{red} \large$\uparrow$} & 83.11 & 67.14 & 0.22 & 42.32 & 53.33 & 67.39 & 0 & 18.64 & 23.56 & 64.18 & 18 & 59.06 & 67.56 & 71.90 & 57.63 & \textbf{80.99} & \textbf{89.11} & \textbf{72.28}  \\ 
    \cmidrule(lr){2-34}
    & {GA$\bigtriangledown$} & 100 & \textbf{78.42} & \textbf{100} & \textbf{81.61} & 100 & 91.41 & 99.78 & 81.38 & 99.98 & \textbf{70.54} & \textbf{99.78} & \textbf{81.51} & 100 & \textbf{62.35} & \textbf{100} & \textbf{81.38} & 100 & \textbf{76.40} & \textbf{100} & \textbf{81.51} & 100 & \textbf{69.95} & \textbf{100} & \textbf{81.58} & 100 & \textbf{80.74} & \textbf{100} & \textbf{81.36} & 99.98 & \textbf{80.12} & \textbf{99.78} & \textbf{81.46} \\ 
    & {GA$\bigtriangleup$}   & 0 & 0 & 0 & 25.97 & 86.47 & 61.75 & 78 & 78.44 & 99.04 & \textbf{52.89} & \textbf{94.89} & \textbf{80.08} & 99.14 & \textbf{45.41} & \textbf{96.44} & \textbf{79.98} & 99.93 & \textbf{64.99} & \textbf{98.22} & \textbf{80.76} & 100 & \textbf{56.84} & \textbf{98.89} & \textbf{80.84} & 0 & 0 & 0 & 65.16 & 0 & 0 & 0 & 61.31 \\ 
    \midrule
    \multirow{7}{*}{$D_c$}
        & RT &  100 & 2.59 & 2.67 & 73.96 & 100 & 11.56 & 11.56 & 72.03 & 100 & 0.57 & 0.22 & 75.38 & 100 & 0.49 & 0.22 & 75.91 & 100 & 0.67 & 0.67 & 72.02 & 100 & 0.32 & 0.22 & 74.70 & 100 & 0.86 & 0.89 & 73.07 & 99.95 & 9.33 & 10.22 & 72.44\\  
    \cmidrule(lr){2-34}
    & {IF$\bigtriangledown$} & 99.93 & 73.33 & 99.78 & 81.24 & 100 & 88.86 & 99.78 & 81.30 & 99.88 & 58.44 & 98.44 & 80.54 & 99.93 & 38.07 & 92 & 80.01 & 99.98 & 66.77 & 98.44 & 81.11 & 99.83 & 37.21 & 79.78 & 79.23 & 100 & 79.28 & 100 & 81.37 & 99.98 & 76.22 & 99.56 & 81.17\\ 
    & {IF$\bigtriangleup$}   & 0.10 & \textbf{0.62} & \textbf{0.44} & \textbf{70.26} & 19.95 & \textbf{7.73} & \textbf{10.67} & \textbf{70.28} & 0 & \textbf{0.03} & \textbf{0} & \textbf{65.87} & 0 & 0 & 0 & 48.20 & 0.40 & \textbf{0} & \textbf{0} & \textbf{67.97} & 0 & 0 & 0 & 53.60 & 22 & \textbf{1.19} & \textbf{0.89} & \textbf{71.85} & 0.10 & \textbf{0.47} & \textbf{0.22} & \textbf{67.79} \\ 
    \cmidrule(lr){2-34}
    & {FT$\bigtriangledown$} & 100 & 71.51 & 99.11 & 80.79 & 100 & 71.24 & 82 & 77.83 & 100 & 23.43 & 38.44 & 75.98 & 100 & 40.82 & 88.89 & 78.92 & 100 & 62.77 & 97.56 & 80.39 & 100 & 49.33 & 84.67 & 77.07 & 100 & 54.35 & 70.89 & 77.56 & 99.95 & 63.75 & 89.11 & 77.87\\ 
    & {FT$\bigtriangleup$}   & \textbf{100} & \textbf{4.12} & \textbf{5.56} & \textbf{66.84} & \textbf{99.95} & \textbf{2.35} & \textbf{2.44} & \textbf{64.01} & \textbf{99.80} & \textbf{1.36} & \textbf{0.67} & \textbf{68.49} & \textbf{99.68} & \textbf{0.20} & \textbf{0.22} & \textbf{68.44} & \textbf{100} & \textbf{0.27} & \textbf{0} & \textbf{65.87} & \textbf{99.98} & \textbf{0.30} & \textbf{0.44} & \textbf{65.99} & \textbf{99.93} & \textbf{0.54} & \textbf{0.44} & \textbf{64.69} & \textbf{99.80} & \textbf{12.59} & \textbf{12.89} & \textbf{68.29}\\ 
    \cmidrule(lr){2-34}
    & {GA$\bigtriangledown$} & 100 & 81.41 & 100 & 81.63 & 100 & 92.91 & 99.78 & 81.56 & 100 & 68.94 & 99.78 & 81.39 & 100 & 58.30 & 100 & 81.38 & 100 & 76.15 & 100 & 81.53 & 100 & 68.77 & 100 & 81.50 & 100 & 84.44 & 100 & 81.57 & 100 & 83.63 & 100 & 81.58\\ 
    & {GA$\bigtriangleup$}   & 79.38 & 31.98 & 50 & 76.70 & 99.83 & 80.32 & 96.22 & 80.55 & 0.05 & \textbf{0.10} & \textbf{0.22} & \textbf{61.02} & 0 & 0 & 0 &10 & 83.38 & 8.03 & 11.78 & 73.29 & 0 & 0 & 0 & 18.97 & 88.07 & 22.32 & 27.78 & 75.12 & 58.47 & 14.32 & 18 & 73.14\\ 
    \midrule
    Before  &  -  &  100  &  84.10  &  100  &  81.59  &  100  &  93.98  & 100   & 81.59   &  100  &  74.20  &  100  &  81.59  &  100  &  67.73  &   100 &   81.59 &  100  &  81.36  &   100 & 81.59   &  100  & 74.30   &   100 &  81.59  &  100  &  86.72  & 100   &  81.59  & 100 & 86.07&  100 & 81.59 \\ 
    \midrule[0.3pt]
    \bottomrule[1pt]
    \end{tabular}
    }
    
    \label{tab:unlearn_acc_on_TUE}
\end{table*}


\begin{table*}[t]
    \centering
    \centering
    \caption{The accuracy (\%) of unlearned classifiers trained on PUE (subset-level).}
    \resizebox{\linewidth}{!}{%
    \begin{tabular}{cccccc|cccc|cccc|cccc|cccc|cccc|cccc|cccc}
    \toprule[1pt]
    \midrule[0.3pt]
    \multirow{2}{*}{Forgetting Set}&\multirow{2}{*}{Method}& \multicolumn{4}{c}{Class 0} & \multicolumn{4}{c}{Class 1} & \multicolumn{4}{c}{Class 2}& \multicolumn{4}{c}{Class 3}& \multicolumn{4}{c}{Class 4} & \multicolumn{4}{c}{Class 5} & \multicolumn{4}{c}{Class 6} & \multicolumn{4}{c}{Class 7} \\  
    \cmidrule(lr){3-34}
      &    & $D_u$ & $D_{uc}$ & $D_c$ & $D_t$ & $D_u$ & $D_{uc}$ & $D_c$ & $D_t$& $D_u$ & $D_{uc}$ & $D_c$ & $D_t$& $D_u$ & $D_{uc}$ & $D_c$ & $D_t$& $D_u$ & $D_{uc}$ & $D_c$ & $D_t$& $D_u$ & $D_{uc}$ & $D_c$ & $D_t$& $D_u$ & $D_{uc}$ & $D_c$ & $D_t$& $D_u$ & $D_{uc}$ & $D_c$ & $D_t$ \\     
    \midrule
    \multirow{7}{*}{$D_u$}
    & RT &  2.86& 77.56 & 100 & 80.73 & 3.75 & 88.47 & 100 & 81.49 & 0.05 & 70.96 & 100 & 81.79 & 0.25 & 66.89 & 100 & 83.71 & 0.07 & 77.51 & 100 & 84.37 & 0.03 & 72.40 & 100 & 82.28 & 0 & 82.84 & 99.78 & 81.78 & 75.26 & 83.98 & 100 & 82.32 \\  
    \cmidrule(lr){2-34}
    & {IF$\bigtriangledown$} & 100 & 80.17 & 99.78 & 82.61 & 100 & 90.4 & 100 & 82.67 & 100 & 72.35 & 99.78 & 82.57 & 100 & \textbf{66.52} & \textbf{100} & \textbf{82.62} & 100 & \textbf{74.1} & \textbf{99.33} & \textbf{82.45} & 100 & \textbf{68.67} & \textbf{99.56} & \textbf{82.75} & 100 & \textbf{81.98} & \textbf{100} & \textbf{82.42} & 100 & \textbf{84.47} & \textbf{99.78} & \textbf{82.74} \\ 
    & {IF$\bigtriangleup$}   & 42.30 & 32.42 & 49.78 & 72.61 & 26.96 & 44.49 & 48.22 & 71.28 & 87.80 & 39.80 & 74 & 75.84 & 0.03 & 10.77 & 21.56 & 71.35 & 17.01 & 28.49 & 47.11 & 72.29 & 0.67 & 15.53 & 24.22 & 73.47 & 0.25 & 8.40 & 9.11 & 70.35 & 93.38 & 34.47 & 50.22 & 72.72  \\ 
    \cmidrule(lr){2-34}
    & {FT$\bigtriangledown$} &  100 & 85.38{\color{red} \large$\uparrow$} & 100 & 82.33 & 100 & 92.27 & 100 & 82.67 & 100 & 73.70 & 100 & 82.44 & 100 & 74.07{\color{red} \large$\uparrow$} & 100 & 82.40 & 100 & 80.79{\color{red} \large$\uparrow$} & 100 & 82.53 & 100 & \textbf{70.07} & \textbf{100} & \textbf{81.82} & 100 & 87.95{\color{red} \large$\uparrow$} & 100 & 82.57 & 100 & 85.93 & 100 & 82.23 \\ 
    & {FT$\bigtriangleup$}  & \textbf{1.73} & \textbf{74.74} & \textbf{79.33} & \textbf{62.40} & \textbf{3.04} & \textbf{88.59} & \textbf{92} & \textbf{63.58} & 0.25 & 51.63 & 58.22 & 67.13 & 0.15 & 41.58 & 46.67 & 69.02 & 0 & 43.01 & 54 & 69.19 & 0 & 29.24 & 28.22 & 66.54 & 8.57 & 70.89 & 77.56 & 68.78 & 91.88 & \textbf{73.21} & \textbf{81.78} & \textbf{70.84}  \\ 
    \cmidrule(lr){2-34}
    & {GA$\bigtriangledown$} & 100 & 81.24 & 99.78 & 82.70 & 100 & 90.96 & 100 & 82.71 & 100 & 73.83 & 100 & 82.67 & 100 & 68.59 & 100 & 82.79 & 100 & \textbf{76.35} & \textbf{99.78} & \textbf{82.62} & 100 & \textbf{71.53} & \textbf{100} & \textbf{82.85} & 100 & \textbf{83.56} & \textbf{100} & \textbf{82.36} & 100 & 86.99 & 100 & 82.76  \\ 
    & {GA$\bigtriangleup$}  & 88.37& 49.46 & 82 & 78.71 & 95.70 & \textbf{66.62} & \textbf{84.89} & \textbf{80.33} & 100 & \textbf{68.64} & \textbf{99.78} & \textbf{82.30} & 0 & 12.30 & 23.78 & 68.90 & 99.93 & \textbf{67.11} & \textbf{98.89} & \textbf{81.78} & 99.98 & \textbf{63.28} & \textbf{99.33} & \textbf{82.50} & 0 & 0 & 0 & 10 & 100 & 85.98 & 100 & 82.79\\ 
    \midrule
    \multirow{7}{*}{$D_c$}
    & RT    &  100 & 4.32 & 4.89 & 75.30 & 100 & 4.72 & 5.78 & 72.84 & 100 & 0.07 & 0 & 74.72 & 99.95 & 0 & 0 & 74.05 & 100 & 0 & 0 & 74.13 & 100 & 0 & 0 & 74.49 & 100 & 0.35 & 0 & 72.91 & 100 & 0.10 & 0.22 & 72.55 \\  
    \cmidrule(lr){2-34}
    & {IF$\bigtriangledown$} & 100 & 75.31 & 99.11 & 82.36 & 100 & 87.80 & 100 & 82.56 & 100 & 62.25 & 98.67 & 81.98 & 99.70 & 30.86 & 78 & 79.86 & 99.98 & 63.73 & 98.89 & 81.70 & 100 & 51.65 & 96.67 & 81.84 & 100 & 79.88 & 99.78 & 82.37 & 100 & 79.58 & 99.56 & 82.37 \\ 
    & {IF$\bigtriangleup$}   & 10.84 & \textbf{2.40} & \textbf{2.67} & \textbf{70.79} & 14.35 & \textbf{5.73} & \textbf{6.89} & \textbf{66.83} & 0 & \textbf{0.64} & \textbf{0.22} & \textbf{71.88} & 0 & 0 & 0 & 52.22 & 1.19 & \textbf{0.05} & \textbf{0} & \textbf{66.89} & 0 & \textbf{0} & \textbf{0} & \textbf{64.37} & 76.69 & \textbf{1.93} & \textbf{2.44} & \textbf{70.40} & 37.31 & \textbf{1.68} & \textbf{2.89} & \textbf{71.87}  \\ 
    \cmidrule(lr){2-34}
    & {FT$\bigtriangledown$} & 100 & 61.80 & 88 & 79.99 & 99.88 & 87.24 & 97.11 & 76.29 & 100 & 9.04 & 14.44 & 74.20 & 99.95 & 28.82 & 60 & 79.21 & 100 & 31.63 & 51.56 & 76.87 & 99.51 & 7.21 & 12.67 & 73.94 & 99.98 & 42.82 & 54.44 & 76.08 & 100 & 62.99 & 89.33 & 79.65 \\ 
    & {FT$\bigtriangleup$}  & 99.88 & 12.32 & 14.67 & 65.73 & 99.58 & 37.21 & 38  & 63.86 & \textbf{99.93} & \textbf{0.03} & \textbf{0} & 53.49 & \textbf{99.31} & \textbf{0} & \textbf{0} & \textbf{64.64} & \textbf{99.19} & \textbf{0.03} & \textbf{0} & \textbf{65.95} & \textbf{99.85} & \textbf{0.59} & \textbf{0} & \textbf{62.24} & \textbf{99.98} & \textbf{1.68} & \textbf{0.89} & \textbf{62.77} & \textbf{100} & \textbf{0.07} & \textbf{0} & \textbf{61.10} \\ 
    \cmidrule(lr){2-34}
    & {GA$\bigtriangledown$} & 100 & 81.63 & 100 & 82.71 & 100 & 91.93 & 100 & 82.68 & 100 & 71.78 & 99.78 & 82.59 & 100 & 58.44 & 100 & 81.97 & 100 & 74.27 & 99.56 & 82.56 & 100 & 69.78 & 100 & 82.75 & 100 & 85.38 & 100 & 82.62 & 100 & 85.65 & 100 & 82.81 \\ 
    & {GA$\bigtriangleup$}  & 0.17 & \textbf{1.06} & \textbf{1.33} & \textbf{65.27} & 99.63 & 75.98 & 94.44 & 81.11 & 0 & \textbf{0} & \textbf{0} & \textbf{62.81} & 0 & 0 & 0 & 9.84 & 0 & 0 & 0 & 10.12 & 0 & 0 & 0 & 10.62 & 99.8 & 39.31 & 57.56 & 77.34 & 73.65 & 7.06 & 8.67 & 72.33 \\ 
    \midrule
    Before  &  -  & 100 & 84.15 & 100 & 82.76 & 100 & 92.72 & 100 & 82.76 & 100 & 75.46 & 100 & 82.76 & 100 & 71.11 & 100 & 82.76 & 100 & 78.44 & 99.78 & 82.76 & 100 & 73.53 & 100 & 82.76 & 100  & 86.86 & 100 & 82.76 & 100 & 87.46 & 100 & 82.76  \\ 
    \midrule[0.3pt]
    \bottomrule[1pt]
    \end{tabular}
    }
    
    \label{tab:unlearn_acc_on_PUE}
\end{table*}

\begin{table*}[t]
    \centering
    \centering
     \caption{The accuracy (\%) of unlearned classifiers trained on OPS (subset-level).}
    \resizebox{\linewidth}{!}{%
    \begin{tabular}{cccccc|cccc|cccc|cccc|cccc|cccc|cccc|cccc}
    \toprule[1pt]
    \midrule[0.3pt]
    \multirow{2}{*}{Forgetting Set}&\multirow{2}{*}{Method}& \multicolumn{4}{c}{Class 0} & \multicolumn{4}{c}{Class 1} & \multicolumn{4}{c}{Class 2}& \multicolumn{4}{c}{Class 3}& \multicolumn{4}{c}{Class 4} & \multicolumn{4}{c}{Class 5} & \multicolumn{4}{c}{Class 6} & \multicolumn{4}{c}{Class 7} \\  
    \cmidrule(lr){3-34}
      &    & $D_u$ & $D_{uc}$ & $D_c$ & $D_t$ & $D_u$ & $D_{uc}$ & $D_c$ & $D_t$& $D_u$ & $D_{uc}$ & $D_c$ & $D_t$& $D_u$ & $D_{uc}$ & $D_c$ & $D_t$& $D_u$ & $D_{uc}$ & $D_c$ & $D_t$& $D_u$ & $D_{uc}$ & $D_c$ & $D_t$& $D_u$ & $D_{uc}$ & $D_c$ & $D_t$& $D_u$ & $D_{uc}$ & $D_c$ & $D_t$ \\     
    \midrule
    \multirow{7}{*}{$D_u$}
    & RT &  60.96 & 71.14 & 100 & 76.88 & 4.89 & 84.86 & 100 & 76.59 & 5.61 & 68.89 & 100 & 78.06 & 10.22 & 59.53 & 100 & 79.1 & 7.06 & 71.31 & 100 & 79.12 & 43.04 & 66.07 & 99.78 & 77.58 & 39.95 & 79.58 & 100 & 79.29 & 23.75 & 79.51 & 100 & 79.16 \\  
    \cmidrule(lr){2-34}
    & {IF$\bigtriangledown$} & 100 & 77.06 & 100 & 78.54 & 88.99 & \textbf{82.30} & \textbf{91.78} & \textbf{76.79} & 100 & \textbf{62.72} & \textbf{100} & \textbf{78.78} & 100 & \textbf{58.79} & \textbf{100} & \textbf{78.77} & 100 & \textbf{68.10} & \textbf{100} & \textbf{78.48} & 100 & \textbf{61.90} & \textbf{99.78} & \textbf{78.28} & 100 & \textbf{80.10} & \textbf{100} & \textbf{78.56} & 100 & \textbf{78.67} & \textbf{100} & \textbf{78.39} \\ 
    & {IF$\bigtriangleup$}   & 99.38 & \textbf{52.35} & \textbf{92.67} & \textbf{76.6} & 0 & 0 & 0 & 13.68 & 99.63 & 33.93 & 82.89 & 75.45 & 99.14 & 22 & 74.67 & 76.64 & 99.73 & \textbf{40.2} & \textbf{92.89} & \textbf{74.84} & 52.47 & 1.85 & 3.33 & 64.56 & 100 & \textbf{66.67} & \textbf{98} & \textbf{76.74} & 98.96 & 42.64 & 75.11 & 73.47  \\ 
    \cmidrule(lr){2-34}
    & {FT$\bigtriangledown$} & 100 & 76.12 & 100 & 77.68 & 98.91 & 96.37 & 100 & 77.98 & 100 & 67.04 & 100 & 78.23 & 100 & \textbf{59.95} & \textbf{100} & \textbf{78.27} & 100 & 70.89 & 100 & 78.41 & 100 & 70.49{\color{red} \large$\uparrow$} & 100 & 78.69 & 100 & 82.59{\color{red} \large$\uparrow$} & 100 & 78.50 & 100 & \textbf{78.4} & \textbf{100} & \textbf{78.23}  \\ 
    & {FT$\bigtriangleup$}  & 74.94 & 79.48 & 90.44 & 68.43 & 36.49 & 89.61 & 94.67 & 64.88 & 1.93 & 40.15 & 46.44 & 69.48 & 56.57 & 62.12 & 76.22 & 67.79 & 7.31 & 65.28 & 78.67 & 71.82 & 45.68 & 44.05 & 56.67 & 71.21 & 29.46 & 46.22 & 50.44 & 70.29 & 40.05 & 82.05{\color{red} \large$\uparrow$} & 90.67 & 72.67  \\ 
    \cmidrule(lr){2-34}
    & {GA$\bigtriangledown$} & 100 & 77.68 & 100 & 78.72 & 0 & 0 & 0 & 10 & 100 & \textbf{63.43} & \textbf{100} & \textbf{78.82} & 100 & \textbf{58.37} & \textbf{100} & \textbf{78.84} & 100 & \textbf{68.62} & \textbf{100} & \textbf{78.56} & 100 & \textbf{62.30} & \textbf{99.78} & 78.44 & 100 & \textbf{80.96} & \textbf{100} & \textbf{78.64} & \textbf{100} & \textbf{78.74} & \textbf{100} & \textbf{78.48}\\ 
    & {GA$\bigtriangleup$}  & 99.63 & \textbf{54.57} & \textbf{93.56} & \textbf{76.77} & 0 & 0 & 0 & 10 & 99.93 & \textbf{44.84} & \textbf{96.44} & \textbf{77.8} & 98.96 & 20.42 & 70.89 & 76.65 & 99.98 & \textbf{56.35} & \textbf{99.11} & \textbf{77.51} & 0 & 0 & 0 & 10 & 100 & \textbf{79.21} & \textbf{100} & \textbf{78.61} & 99.16 & 48.4 & 84.22 & 74.81\\ 
    \midrule
    \multirow{7}{*}{$D_c$}
    & RT    & 100 & 2.20 & 2 & 71.71 & 99.85 & 70.64 & 69.33 & 75.67 & 100 & 1.53 & 1.56 & 73.74 & 100 & 3.88 & 2.44 & 74.21 & 100 & 2.07 & 3.78 & 70.42 & 100 & 7.63 & 5.78 & 73.04 & 100 & 0.22 & 0 & 71.28 & 99.98 & 4.27 & 4.44 & 70.65 \\  
    \cmidrule(lr){2-34}
    & {IF$\bigtriangledown$} &  100 & 70.99 & 99.78 & 78.24 & 100 & 96.30 & 100 & 78.76 & 100 & 54.69 & 99.78 & 78.37 & 100 & 42.91 & 98.44 & 78.24 & 100 & 63.93 & 100 & 78.20 & 100 & 57.16 & 99.56 & 78.01 & 100 & 76.62 & 100 & 78.49 & 100 & 73.75 & 99.78 & 78.08 \\ 
    & {IF$\bigtriangleup$}   & 71.58 & \textbf{3.11} & \textbf{5.56} & \textbf{64.72} & 96.72 & 85.70 & 97.11 & 77.8 & 56.89 & \textbf{0.27} & \textbf{0.44} & \textbf{64.33} & 10.57 & \textbf{0} & \textbf{0} & \textbf{71.77} & \textbf{94.64} & \textbf{6.59} & 16.44 & \textbf{71.22} & 48.07 & \textbf{0.15} & \textbf{0.44} & \textbf{66.22} & 96.86 & 12.25 & 14.89 & 70.73 & \textbf{90.17} & \textbf{2.59} & \textbf{4.22} & \textbf{68.08}  \\ 
    \cmidrule(lr){2-34}
    & {FT$\bigtriangledown$} & 100 & 63.36 & 98.22 & 77.45 & 99.98 & 92.37 & 99.11 & 77.78 & 99.98 & 25.14 & 42.89 & 74.09 & 99.98 & 15.68 & 27.56 & 71.90 & 100 & 43.21 & 86.44 & 75.55 & 99.95 & 39.46 & 73.56 & 75.37 & 100 & 3.80 & 5.33 & 68.50 & 100 & 37.85 & 54 & 72.96\\ 
    & {FT$\bigtriangleup$}  & \textbf{99.93} & \textbf{9.48} & \textbf{11.56} & \textbf{65.31} & 99.68 & 85.95 & 85.33 & 74.07 & \textbf{99.48} & \textbf{2.59} & \textbf{2} & \textbf{67.10} & \textbf{99.80} & \textbf{6.86} & \textbf{6.89} & \textbf{65.21} & \textbf{99.88} & \textbf{5.68} & \textbf{7.56} & \textbf{66.68} & \textbf{99.75} & \textbf{12.91} & \textbf{12} & \textbf{65.29} & \textbf{100} & \textbf{0} & \textbf{0} & \textbf{65.42} & \textbf{99.31} & \textbf{5.14} & \textbf{5.33} & \textbf{68.10} \\ 
    \cmidrule(lr){2-34}
    & {GA$\bigtriangledown$} &  100 & 76.79 & 100 & 78.62 & 100 & 96.44 & 100 & 78.75 & 100 & 61.51 & 100 & 78.79 & 100 & 49.16 & 99.78 & 78.44 & 100 & 68.25 & 100 & 78.46 & 100 & 63.41 & 99.78 & 78.37 & 100 & 79.24 & 100 & 78.47 & 100 & 77.63 & 100 & 78.39 \\ 
    & {GA$\bigtriangleup$}  & 93.46 & 12.27 & 22.89 & 68.25 & 99.80 & 93.10 & 100 & 78.68 & 77.90 & \textbf{0.89} & \textbf{0.89} & \textbf{68.16} & 0 & 0 & 0 & 10 & 99.95 & 40.10 & 93.33 & 75.77 & 0 & 0 & 0 & 33.35 & 99.70 & 29.38 & 46 & 73.31 & 13 & 0 & 0 & 58.14\\ 
    \midrule
    Before  &  -  & 100 & 80.37 & 100 & 78.68 & 100 & 96.77 & 100 & 78.68 & 100 & 66.47 & 100 & 78.68 & 100 & 62.52 & 100 & 78.68 & 100 & 70.64 & 100 & 78.68 & 100 & 69.61 & 100 & 78.68 & 100 & 81.65 & 100 & 78.68 & 100 & 81.56 & 100 & 78.68  \\ 
    \midrule[0.3pt]
    \bottomrule[1pt]
    \end{tabular}
    }
   
    \label{tab:unlearn_acc_on_OPS}
\end{table*}

\begin{table*}[t]
    \centering
    \centering
    \caption{The accuracy (\%) of unlearned classifiers trained on UE-s (class-level).}
    \resizebox{\linewidth}{!}{%
    \begin{tabular}{cccccc|cccc|cccc|cccc|ccc|ccc|ccc|ccc}
    \toprule[1pt]
    \midrule[0.3pt]
    \multirow{2}{*}&\multirow{2}{*}{Method}& \multicolumn{4}{c}{Class 0} & \multicolumn{4}{c}{Class 1} & \multicolumn{4}{c}{Class 2}& \multicolumn{4}{c}{Class 3}& \multicolumn{3}{c}{Class 4} & \multicolumn{3}{c}{Class 5} & \multicolumn{3}{c}{Class 6} & \multicolumn{3}{c}{Class 7} \\  
    \cmidrule(lr){3-30}
      &  &  $D_u$ & $D_{uc}$ & $D_{ut}$ & $D_{tr}$&  $D_u$ & $D_{uc}$ & $D_{ut}$ & $D_{tr}$&  $D_u$ & $D_{uc}$ & $D_{ut}$ & $D_{tr}$& $D_u$ & $D_{uc}$ & $D_{ut}$ & $D_{tr}$ & $D_{c}$ & $D_{ct}$ & $D_{tr}$ & $D_{c}$ & $D_{ct}$ & $D_{tr}$ & $D_{c}$ & $D_{ct}$ & $D_{tr}$ & $D_{c}$ & $D_{ct}$ & $D_{tr}$ \\     
    \midrule
    \multirow{7}{*}
    & RT & 0 & 0 & 0 & 97.35 & 0 & 0 & 0 & 97.70 & 0 & 0  & 0 & 97.05 & 0 &  0 & 0 & 97.20 & 0 & 0 & 97.15 & 0 & 0 & 96.85 & 0 & 0 & 97.20 & 0 & 0 & 97.20   \\  
    \cmidrule(lr){2-30}
    & {IF$\bigtriangledown$} & 39.22 & 0.04 & 0.20  & 96.35 & 85.67 & 0 & 0 & 97.05 & 5.29 & 0 & 0 &96.05  &98.33  & 0 & 0 & 97.35 & 0 & 0 & 36.05 & 0 & 0 & 78.50 & 0 & 0 & 95.10 & 0 & 0 &  93.50  \\ 
    & {IF$\bigtriangleup$}  & 0 & 0 & 0 & 79.60 & 0 & 0 & 0 & 89.15 & 0 & 0 & 0 & 45.60 & 0.31 & 0 & 0 & 95.80 & 0 & 0 & 10 &0 & 0 & 0 & 0 & 0 & 13.75 & 0 & 0 &  18.55  \\ 
    \cmidrule(lr){2-30}
    & {FT$\bigtriangledown$} & 8.44 & 0 & 0 & 95.50 & 16.22 & 0 & 0 & 95.85 & 24.76  & 0 & 0 & 97.55 & 0.38 & 0 & 0 & 96.25 & 3.36 & 2.70 & 96.25 & 2.80 & 3.40 & 96.45 & 0.80 & 1 & 94.95 & 3.53  & 3.90 & 95.80  \\ 
    & {FT$\bigtriangleup$}  & \textbf{0} & \textbf{0} & \textbf{0} & \textbf{92.55} & \textbf{0} & \textbf{0} & \textbf{0} & \textbf{88.80} & \textbf{0} & \textbf{0} & \textbf{0} & \textbf{96.80} & \textbf{0} & \textbf{0} & \textbf{0} & \textbf{93.25} & \textbf{0} & \textbf{0} & \textbf{96.10} & \textbf{0} & \textbf{0} & \textbf{93.80} & \textbf{0} & \textbf{0} &  \textbf{93.75} & \textbf{0} & \textbf{0} & \textbf{95.75} \\ 
    \cmidrule(lr){2-30}
    & {GA$\bigtriangledown$} & 100 & 8 & 8.70 & 97.55 & 100 & 1.89 & 1.60 &  97.50 & 99.98 & 3.98 & 3.20 & 97.55 & 100 & 0.33 & 0.20 & 97.55 & 24.76 & 23 & 94.50 & 0 & 0 & 90.25 & 98.98 & 89.60 & 97.70 & 97.56 & 83.40 &  97.55 \\ 
    & {GA$\bigtriangleup$}  & 1.24 & 0 & 0 & 83.20 & 100 & 1.24 & 1.10 & 97.50 & 0 & 0 & 0 & 10 & 100 & 0.27 & 0.20 & 97.55 & 0 & 0 & 10 & 0 & 0 & 10 &0 &  0 & 10 & 0 & 0 & 10  \\ 
    \midrule
    
     & Before & 100 & 9.80 & 9.90 & 97.55 & 100 & 2.40 & 2.40 & 97.55 & 100 & 5.93 & 5.30 & 97.55 & 100 & 0.36 & 0.20 & 97.55 & 100 & 93 & 97.55  & 99.98 & 92.80 & 97.55 & 100 & 96.30 & 97.55 & 99.98 & 94.20 & 97.55  \\ 
    \midrule[0.3pt]
    \bottomrule[1pt]
    \end{tabular}
    }
    
    \label{tab:unlearn_class_acc_on_UE}
\end{table*}


\begin{table*}[t]
    \centering
    \centering
     \caption{The accuracy (\%) of unlearned classifiers trained on TUE (class-level).}
    \resizebox{\linewidth}{!}{%
    \begin{tabular}{cccccc|cccc|cccc|cccc|ccc|ccc|ccc|ccc}
    \toprule[1pt]
    \midrule[0.3pt]
    \multirow{2}{*}&\multirow{2}{*}{Method}& \multicolumn{4}{c}{Class 0} & \multicolumn{4}{c}{Class 1} & \multicolumn{4}{c}{Class 2}& \multicolumn{4}{c}{Class 3}& \multicolumn{3}{c}{Class 4} & \multicolumn{3}{c}{Class 5} & \multicolumn{3}{c}{Class 6} & \multicolumn{3}{c}{Class 7} \\  
    \cmidrule(lr){3-30}
      &  &  $D_u$ & $D_{uc}$ & $D_{ut}$ & $D_{tr}$&  $D_u$ & $D_{uc}$ & $D_{ut}$ & $D_{tr}$&  $D_u$ & $D_{uc}$ & $D_{ut}$ & $D_{tr}$& $D_u$ & $D_{uc}$ & $D_{ut}$ & $D_{tr}$ & $D_{c}$ & $D_{ct}$ & $D_{tr}$ & $D_{c}$ & $D_{ct}$ & $D_{tr}$ & $D_{c}$ & $D_{ct}$ & $D_{tr}$ & $D_{c}$ & $D_{ct}$ & $D_{tr}$ \\     
    \midrule
    \multirow{7}{*}
    & RT & 0 & 0 & 0 & 97.75 & 0 & 0 & 0 & 97.65 & 0 & 0 & 0 & 97.60 & 0 & 0 & 0 & 97.35 & 0 & 0 & 97.40 & 0 & 0 & 97.50 & 0 & 0 & 97.70 & 0 & 0 & 97.05   \\  
    \cmidrule(lr){2-30}
    & {IF$\bigtriangledown$} & 99.80 & 0 & 0 & 97.20 & 99.93 & 0.02 & 0.10 & 97.25 & 99.09 & 0 & 0 & 97.65 & 99.27 & 0 & 0 & 97.56 & 0 & 0 & 78.80 & 0 & 0 & 87.9 & 0.93 & 0.40 & 97.50 & 2.11 & 2.50 & 91.55   \\ 
    & {IF$\bigtriangleup$}  & 0.18 & 0 & 0 & 91.55 & 3.53 & 0 & 0 & 89.20 & 0 & 0 & 0 & 97.40 & 0.07 & 0 & 0 & 95.90 & 0 &  0 & 10 & 0 & 0 & 46.50 & 0 & 0 & 94.90 & 0 & 0 & 0.50   \\ 
    \cmidrule(lr){2-30}
    & {FT$\bigtriangledown$} & 28.71 & 0 & 0 & 97.20 & 59.89 & 0 & 0 & 93.40 & 11.49 & 0 & 0 & 96.65 & 36.71 & 0 & 0 &  90.65& 15.56 & 12.50 & 96.45 & 9.11 & 8.80 & 96.95 & 6.22 & 6.50 & 97.55 & 5.20 &  4.40 & 95.90  \\ 
    & {FT$\bigtriangleup$}  & \textbf{0} & \textbf{0} & \textbf{0} & \textbf{96.60} & \textbf{0} & \textbf{0 }& \textbf{0} & \textbf{96.65} &\textbf{ 0} & \textbf{0 }& \textbf{0 }& \textbf{94.05} & \textbf{0 }& \textbf{0} &\textbf{ 0 }& \textbf{95.20} & \textbf{0} & \textbf{0 }& \textbf{96.95} & \textbf{0 }& \textbf{0} &\textbf{ 96.20} & \textbf{0 }& \textbf{0} & \textbf{97.20} & \textbf{ 0} &\textbf{ 0} & \textbf{96.50} \\ 
    \cmidrule(lr){2-30}
    & {GA$\bigtriangledown$} & 100 & 0 & 0 & 97.75 & 100 & 0.11 & 0.30 & 97.75 & 100 & 0.02 & 0 & 97.75 & 100 & 0.02 & 0 & 97.75 & 0 & 0 & 10 & 0 & 0 & 92 & 35.02 & 34.20 & 97.55 & 99.58 & 89.90 & 97.80  \\ 
    & {GA$\bigtriangleup$}  & 0 & 0 & 0 & 95.85 & 0 & 0 & 0 & 50 & 0 & 0 & 0 & 10 & 0 & 0 & 0 & 10 & 0 & 0 & 10 & 0 & 0 & 10 & 0 & 0 & 10 & 0 & 0 &  10 \\ 
    \midrule
    
     & Before & 100 & 0.02 & 0 & 97.75 & 100 & 0.13 & 0.30 &  97.75 & 100 & 0.02 & 0 & 97.75 & 100 & 0.04 & 0 & 97.75 & 100 & 93.40 & 97.75 & 100 & 93.60 & 97.75 & 100 & 96.20 & 97.75 & 100 & 95.20 &  97.75 \\ 
    \midrule[0.3pt]
    \bottomrule[1pt]
    \end{tabular}
    }
   
    \label{tab:unlearn_class_acc_on_TUE}
\end{table*}


\begin{table*}[t]
    \centering
    \centering
     \caption{The accuracy (\%) of unlearned classifiers trained on PUE (class-level).}
    \resizebox{\linewidth}{!}{%
    \begin{tabular}{cccccc|cccc|cccc|cccc|ccc|ccc|ccc|ccc}
    \toprule[1pt]
    \midrule[0.3pt]
    \multirow{2}{*}&\multirow{2}{*}{Method}& \multicolumn{4}{c}{Class 0} & \multicolumn{4}{c}{Class 1} & \multicolumn{4}{c}{Class 2}& \multicolumn{4}{c}{Class 3}& \multicolumn{3}{c}{Class 4} & \multicolumn{3}{c}{Class 5} & \multicolumn{3}{c}{Class 6} & \multicolumn{3}{c}{Class 7} \\  
    \cmidrule(lr){3-30}
      &  &  $D_u$ & $D_{uc}$ & $D_{ut}$ & $D_{tr}$&  $D_u$ & $D_{uc}$ & $D_{ut}$ & $D_{tr}$&  $D_u$ & $D_{uc}$ & $D_{ut}$ & $D_{tr}$& $D_u$ & $D_{uc}$ & $D_{ut}$ & $D_{tr}$ & $D_{c}$ & $D_{ct}$ & $D_{tr}$ & $D_{c}$ & $D_{ct}$ & $D_{tr}$ & $D_{c}$ & $D_{ct}$ & $D_{tr}$ & $D_{c}$ & $D_{ct}$ & $D_{tr}$ \\     
    \midrule
    \multirow{7}{*}
    & RT & 0 & 0 & 0 & 97.70 & 0 & 0 & 0 & 97.70 & 0 & 0 & 0 & 97.85 & 0 & 0 & 0 & 97.65 & 0 & 0 & 97.80 & 0 & 0 & 97.95 & 0 & 0 & 98 & 0 & 0 & 97.80   \\  
    \cmidrule(lr){2-30}
    & {IF$\bigtriangledown$} & 70.60 & 0.42 & 0.60 & 97.25 & 76.87 & 0.20 & 0.10 & 95.95 & 98.93 & 0.02 & 0.10 & 97.55 & 98.64 & 0 & 0 & 97.65 & 0.07 & 0 & 90.10 & 0 & 0 & 86.65 & 0.16 & 0.10 & 96.40 & 0 & 0 & 97.15   \\ 
    & {IF$\bigtriangleup$}  &0 & 0 & 0 & 84.85 & 0 & 0 & 0 & 72.15 & 0.02 & 0 & 0 & 91.80 & 0 & 0 & 0 & 95.80 & 0 & 0 & 0 & 0 & 0 & 2.90 & 0 & 0 & 62.90 & 0 & 0 & 63.90   \\ 
    \cmidrule(lr){2-30}
    & {FT$\bigtriangledown$} & 0.07 & 0 & 0 &96.40 & 1.31 & 0 & 0 & 94 & 0 &   0 & 0 & 94.15 & 18.62 & 0 & 0 & 97.35 & 7.51 & 8 & 96.70 & 4.93 & 4.30 & 96.70 & 5.67 & 4.70 & 97.60 & 6.96 & 7.40 & 97.05  \\ 
    & {FT$\bigtriangleup$}  & \textbf{0} & \textbf{0} & \textbf{0} & \textbf{97.10} & \textbf{0} & \textbf{0} & \textbf{0} & \textbf{95.45} & \textbf{0} & \textbf{0} & \textbf{0} & \textbf{97.85} & \textbf{0} & \textbf{0} & \textbf{0} & \textbf{92.35} & \textbf{0} & \textbf{0} & \textbf{96.45} & \textbf{0} & \textbf{0} & \textbf{96.20} & \textbf{0} & \textbf{0} & \textbf{96.85} & \textbf{0} & \textbf{0} & \textbf{96.55} \\ 
    \cmidrule(lr){2-30}
    & {GA$\bigtriangledown$} & 99.98 & 5.69 &  4.90 & 97.95 & 100 & 6.80 & 6.70 & 97.90 & 100 & 0.07 & 0.30 & 97.95& 100 & 0.13 & 0 & 97.95 & 26.58  & 26.40 & 96.15 & 0 & 0 & 10 &0 & 0 &95.75  & 0 &  0 & 10  \\ 
    & {GA$\bigtriangleup$}  & 94.69 & 2.36 & 2.30 & 97.85 & 0 & 0 & 0 &  10 & 100 & 0.04 & 0.30 & 98 & 0 & 0 & 0 & 95.85 & 0 & 0 & 10 & 0 & 0 & 10 & 0 & 0 & 10 & 0 & 0 & 10  \\ 
    \midrule
    
     & Before & 99.98 & 6.67 & 5.40 & 97.95  & 100 & 8.44 & 8.40 & 97.95 & 100 & 0.09 & 0.30 & 97.95 & 100 & 0.13 & 0 & 97.95 & 100 & 94.60 & 97.95 & 100 & 93 & 97.95 & 100 & 95.40 & 97.95 &100 & 94.70 & 97.95  \\ 
    \midrule[0.3pt]
    \bottomrule[1pt]
    \end{tabular}
    }
    
    \label{tab:unlearn_class_acc_on_PUE}
\end{table*}


\begin{table*}[t]
    \centering
    \centering
    \caption{The accuracy (\%) of unlearned classifiers trained on OPS (class-level).}
    \resizebox{\linewidth}{!}{%
    \begin{tabular}{cccccc|cccc|cccc|cccc|ccc|ccc|ccc|ccc}
    \toprule[1pt]
    \midrule[0.3pt]
    \multirow{2}{*}&\multirow{2}{*}{Method}& \multicolumn{4}{c}{Class 0} & \multicolumn{4}{c}{Class 1} & \multicolumn{4}{c}{Class 2}& \multicolumn{4}{c}{Class 3}& \multicolumn{3}{c}{Class 4} & \multicolumn{3}{c}{Class 5} & \multicolumn{3}{c}{Class 6} & \multicolumn{3}{c}{Class 7} \\  
    \cmidrule(lr){3-30}
      &  &  $D_u$ & $D_{uc}$ & $D_{ut}$ & $D_{tr}$&  $D_u$ & $D_{uc}$ & $D_{ut}$ & $D_{tr}$&  $D_u$ & $D_{uc}$ & $D_{ut}$ & $D_{tr}$& $D_u$ & $D_{uc}$ & $D_{ut}$ & $D_{tr}$ & $D_{c}$ & $D_{ct}$ & $D_{tr}$ & $D_{c}$ & $D_{ct}$ & $D_{tr}$ & $D_{c}$ & $D_{ct}$ & $D_{tr}$ & $D_{c}$ & $D_{ct}$ & $D_{tr}$ \\     
    \midrule
    \multirow{7}{*}
    & RT & 0 & 0 & 0  & 97.30 & 0 & 0 & 0 & 97.40 & 0 & 0 & 0 & 97.35 & 0 & 0 & 0 & 97.30 & 0 & 0 & 97.35 & 0 & 0 & 97.30 & 0 & 0 & 97.45 & 0 & 0 & 97.35   \\  
    \cmidrule(lr){2-30}
    & {IF$\bigtriangledown$} & 98.58 & 0.56 & 0.10 & 96.55 & 96.38 & 0.22 & 0.10 & 96.50 & 97.11 & 0.11 & 0 & 96 & 96.40 & 0.33 & 0 & 97 & 2.58 & 3.10 & 97 & 0  & 0 & 79.90 & 0.64 & 0.40 & 97.15 & 0 & 0 & 92.25   \\ 
    & {IF$\bigtriangleup$}  & 0 & 0 & 0 & 87.30 & 0 & 0 & 0 & 86.90 & 0 & 0 & 0 & 79.95 & 0 & 0 & 0 & 91.45 & 0 & 0 & 42.65 & 0 & 0 & 0 & 0 & 0 & 88.15 & 0 & 0 & 1.90   \\ 
    \cmidrule(lr){2-30}
    & {FT$\bigtriangledown$} & 40.33 & 0 & 0 & 95.30 & 64 & 0 & 0 &  96.25 & 15.64 & 0 & 0 & 95.30 & 7.24 & 0 & 0  & 95.70 & 12.33 & 11.40  & 96.05 & 3.98 & 5.10 & 95.25 & 6.04 & 5.40 & 94.95 & 13.82 & 14.50 & 95.95  \\ 
    & {FT$\bigtriangleup$}  & \textbf{0} & \textbf{0} & \textbf{0} & \textbf{94.45} & \textbf{2.04} & \textbf{0} & \textbf{0} & \textbf{95.50} & \textbf{0} & \textbf{0 }& \textbf{0} & \textbf{96.55} & \textbf{0} & \textbf{0 }& \textbf{0} & \textbf{89.40} & \textbf{0} & \textbf{0} & \textbf{94.80} & \textbf{0 }& \textbf{0} & \textbf{96.25} & \textbf{0} &\textbf{ 0} & \textbf{95.80 }& \textbf{0} &\textbf{ 0} & \textbf{95.70} \\ 
    \cmidrule(lr){2-30}
    & {GA$\bigtriangledown$} & 100 & 2.20 & 1.30 & 97.50 & 100 & 5.13 & 5 & 97.45  & 100 & 2.56 & 1.30 & 97.45 & 100 & 2.71 & 1.30 & 97.45 & 5.87 & 6.80 & 92.95 & 0 & 0 & 10 & 0 & 0 & 93.95 & 0 & 0  &  10 \\ 
    & {GA$\bigtriangleup$}  & 96.71 & 0.44 & 0 & 97.25 & 48.71 & 0 & 0 & 97 & 99.98 & 1.07 & 0.50 & 97.50 & 0 & 0 & 0 & 95.45 & 0 &  0 & 10 & 0 & 0 & 10 & 0 & 0 & 10 & 0 & 0 & 10  \\ 
    \midrule
    
     & Before & 100 & 2.71 & 1.50 & 97.45 & 100 & 6.27 & 7.40  & 97.45 & 100 & 3.02 & 1.80  & 97.45 & 100 & 3.09 & 1.60 & 97.45  &  100 & 94.50 & 97.45 & 100 & 92.80 & 97.45 & 99.98 & 96.80 &  97.45 & 99.98 & 94.70 & 97.45  \\ 
    \midrule[0.3pt]
    \bottomrule[1pt]
    \end{tabular}
    }
    
    \label{tab:unlearn_class_acc_on_OPS}
\end{table*}


\begin{table*}[t]
    \centering
    \centering
    \caption{The MIA success rate of unlearned classifiers trained on UE-s (subset-level). “--” indicates cases where the corresponding model in Table~\ref{tab:unlearn_acc_on_UE} exhibits extremely low $D_c$ and $D_t$ accuracy, suggesting that the model has been severely damaged; in such cases, the MIA results are not meaningful and are thus omitted. 
The red upward arrow ({\color{red} \large$\uparrow$}) denotes that the post-unlearning model achieves a higher MIA success rate than both the model before unlearning and the retrained model.}
    \resizebox{\linewidth}{!}{%
    \begin{tabular}{ccccc|ccc|ccc|ccc|ccc|ccc|ccc|ccc}
    \toprule[1pt]
    \midrule[0.3pt]
    \multirow{1}{*}&\multirow{2}{*}{Method}& \multicolumn{3}{c}{Class 0} & \multicolumn{3}{c}{Class 1} & \multicolumn{3}{c}{Class 2}& \multicolumn{3}{c}{Class 3}& \multicolumn{3}{c}{Class 4} & \multicolumn{3}{c}{Class 5} & \multicolumn{3}{c}{Class 6} & \multicolumn{3}{c}{Class 7} \\  
    \cmidrule(lr){3-26}
      &  &  corr & prob  & entro & corr & prob  & entro &  corr & prob  & entro &  corr & prob  & entro &  corr & prob  & entro &  corr & prob & entro &  corr & prob  & entro & corr & prob & entro \\     
    \midrule
    \multirow{7}{*}
    & RT & 0.6721 & 0.7198 & 0.4916 & 0.8388 & 0.8217 & 0.6519 & 0.6099 & 0.6469 & 0.3950 & 0.5867 & 0.5867 & 0.3121 & 0.6499 & 0.6662 & 0.4156 & 0.6348 & 0.6674 & 0.4188 & 0.7593 & 0.7511 & 0.5556 & 0.7306 & 0.7281 & 0.5138 \\ 
    \cmidrule(lr){2-26}
    & {IF$\bigtriangledown$} & 0.7178 & 0.7272 & 0.5012 & 0.8235 & 0.8319 & 0.6862 & 0.5316 & 0.6069 & 0.3454 & 0.5272 & 0.5427 & 0.2607 & 0.6143 & 0.6212 & 0.3494 & 0.5254 & 0.5854 & 0.3133 & 0.7230 & 0.7299 & 0.4936 & 0.7306 & 0.7136 & 0.5074 \\
    & {IF$\bigtriangleup$} & -&- &- & -& - & - & -& - &-  & - & - & -&- & - & -& - & - & - & -&- & - & -& - & -     \\  
    \cmidrule(lr){2-26}
    & {FT$\bigtriangledown$} & 0.7447&0.7733{\color{red} \large$\uparrow$} &0.5709{\color{red} \large$\uparrow$} & 0.8531& 0.8546 & 0.7156 & 0.6657& 0.7133{\color{red} \large$\uparrow$} & 0.4877{\color{red} \large$\uparrow$}  & 0.5832 & 0.6289{\color{red} \large$\uparrow$} & 0.3491{\color{red} \large$\uparrow$}& 0.7240 & 0.7242{\color{red} \large$\uparrow$} & 0.4933{\color{red} \large$\uparrow$}& 0.6044 & 0.6916{\color{red} \large$\uparrow$} & 0.4600{\color{red} \large$\uparrow$} & 0.8326{\color{red} \large$\uparrow$}&0.8341{\color{red} \large$\uparrow$} & 0.6800{\color{red} \large$\uparrow$} & 0.7704{\color{red} \large$\uparrow$}& 0.7627{\color{red} \large$\uparrow$} & 0.5694{\color{red} \large$\uparrow$}     \\
    & {FT$\bigtriangleup$}  & 0.5911 & 0.3121 & 0.0321 & 0.8916{\color{red} \large$\uparrow$} & 0.7420 & 0.4756 & - & - & - & - & - & - & - & - & - & - & - & - & 0.7398 & 0.4316 & 0.1457 & 0.7281 & 0.4565 & 0.1481 \\
    \cmidrule(lr){2-26}
    & {GA$\bigtriangledown$} & 0.7049 & 0.7210 & 0.4874 & 0.7815 & 0.8062 & 0.6412 & 0.2558 & 0.5385 & 0.2640 & 0.5126 & 0.5363 & 0.2514 & 0.4311 & 0.5422 & 0.2533 & 0.2810 & 0.5136 & 0.2306 & 0.6719 & 0.6877 & 0.4365 & 0.7398 & 0.7188 & 0.5128 \\
    & {GA$\bigtriangleup$}  & 0.6881 & 0.7094 & 0.4743 & 0.6896 & 0.7612 & 0.5573 & - & - & -  & 0.4938 & 0.5294 & 0.2484 & - & - & - & - & - & - & 0.5585 & 0.6252 & 0.3521 & 0.7311 & 0.7121 & 0.5057 \\ 
    \midrule
    
     & Before & 0.7699 & 0.7585 & 0.5531 & 0.8593 & 0.8610 & 0.7281 & 0.6842 & 0.6822 & 0.4467 & 0.5815 & 0.5746 & 0.2830 & 0.7462 & 0.7052 & 0.4723 & 0.6778 & 0.6706 & 0.4168 & 0.8052 & 0.7869 & 0.5923 & 0.7669 & 0.7437 & 0.5383 \\
    \midrule[0.3pt]
    \bottomrule[1pt]
    \end{tabular}
    }
    \label{tab:unlearn_subset_MIA_on_UE-s}
\end{table*}


\begin{table*}[t]
    \centering
    \centering
    \caption{The MIA success rate of unlearned classifiers trained on UE-s (class-level). “--” indicates cases where the corresponding model in Table~\ref{tab:unlearn_class_acc_on_UE} exhibits extremely low $D_{tr}$ accuracy, suggesting that the model has been severely damaged; in such cases, the MIA results are not meaningful and are thus omitted. 
The red upward arrow ({\color{red} \large$\uparrow$}) denotes that the post-unlearning model achieves a higher MIA success rate than both the model before unlearning and the retrained model.}
    \resizebox{\linewidth}{!}{%
    \begin{tabular}{ccccc|ccc|ccc|ccc|ccc|ccc|ccc|ccc}
    \toprule[1pt]
    \midrule[0.3pt]
    \multirow{1}{*}&\multirow{2}{*}{Method}& \multicolumn{3}{c}{Class 0} & \multicolumn{3}{c}{Class 1} & \multicolumn{3}{c}{Class 2}& \multicolumn{3}{c}{Class 3}& \multicolumn{3}{c}{Class 4} & \multicolumn{3}{c}{Class 5} & \multicolumn{3}{c}{Class 6} & \multicolumn{3}{c}{Class 7} \\  
    \cmidrule(lr){3-26}
      &  &  corr & prob  & entro & corr & prob  & entro &  corr & prob  & entro &  corr & prob  & entro &  corr & prob  & entro &  corr & prob & entro &  corr & prob  & entro & corr & prob & entro \\     
    \midrule
    \multirow{7}{*}
    & RT & 0 & 0.6271 & 0.3387 & 0 & 0.8664 & 0.7156 & 0 & 0.6351 & 0.3633 & 0 & 0.7380 & 0.4984 & 0 & 0.6753 & 0.3920 & 0 & 0.7216 & 0.4671 & 0 & 0.7000 & 0.4444 & 0 & 0.7138 & 0.4558 \\  
    \cmidrule(lr){2-26}
    & {IF$\bigtriangledown$} & 0.0004 & 0.6207 & 0.3316 & 0 & 0.8502 & 0.6978 & 0 & 0.6780{\color{red} \large$\uparrow$} & 0.4036{\color{red} \large$\uparrow$} & 0 & 0.7378 & 0.4987 & - & - & - & 0 & 0.6484 & 0.3318 & 0 & 0.4289 & 0.1302 & 0 & 0.6158 & 0.2993 \\
    & {IF$\bigtriangleup$} & 0 & 0.7377{\color{red} \large$\uparrow$} & 0.4987{\color{red} \large$\uparrow$} & 0 & 0.7607 & 0.5489 & - & - & - & 0 & 0.7673{\color{red} \large$\uparrow$} & 0.5478{\color{red} \large$\uparrow$} & - & - & - & - & - & - & - & - & - & - & - & - \\ 
    \cmidrule(lr){2-26}
    & {FT$\bigtriangledown$} & 0 & 0.5451 & 0.2389 & 0 & 0.8742{\color{red} \large$\uparrow$} & 0.7273{\color{red} \large$\uparrow$} & 0 & 0.5951 & 0.2962  & 0 & 0.7542 & 0.5315 & 0.0336 & 0.5144 & 0.2031 & 0.0280 & 0.4822 & 0.1711 & 0.0080 & 0.5229 & 0.1944 & 0.0353 & 0.5351 & 0.2218 \\
    
    & {FT$\bigtriangleup$}  & 0 & 0.5120 & 0.2053 & 0 & 0.9588{\color{red} \large$\uparrow$} & 0.8960{\color{red} \large$\uparrow$} & 0 & 0.4956 & 0.1980  & 0 & 0.5427 & 0.2360 & 0 & 0.4449 & 0.1413 & 0 & 0.5529 & 0.2340 & 0 & 0.4667 & 0.1782 & 0 & 0.5427 & 0.2129 \\
    \cmidrule(lr){2-26}
    & {GA$\bigtriangledown$} & 0.0800 & 0.5964 & 0.3031 & 0.0189 & 0.8309 & 0.6427 & 0.0398 & 0.6007 & 0.3256  & 0.0033 & 0.7082 & 0.4578 & 0.2476 & 0.6596 & 0.3893 & 0 & 0.5824 & 0.2696 & 0.9898 & 0.9540 & 0.8480 & 0.9756 & 0.9162 & 0.7653     \\
    
    & {GA$\bigtriangleup$}  & 0 & 0.5893 & 0.3004 & 0.0124 & 0.8407 & 0.6616 & - & - & -  & 0.0027 & 0.7058 & 0.4540 & - & - & - & - & - & - & - & - & - & - & - & -     \\ 
    \midrule
    
    & Before & 0.0980 & 0.5973 & 0.2989 & 0.0240 & 0.8222 & 0.6324 & 0.0593 & 0.5973 & 0.3142  & 0.0036 & 0.7064 & 0.4607 & 1.0 & 0.9969 & 0.9749 & 0.9998 & 0.9976 & 0.9738 & 1.0 & 0.9984 & 0.9847 & 0.9998 & 0.9987 & 0.9791 \\
    \midrule[0.3pt]
    \bottomrule[1pt]
    \end{tabular}
    }
    \label{tab:unlearn_class_MIA_on_UE-s}
\end{table*}

\begin{table*}[t]
    \centering
    \centering
    \caption{The MIA success rate of unlearned classifiers trained on TUE (subset-level). “--” indicates cases where the corresponding model in Table~\ref{tab:unlearn_acc_on_TUE} exhibits extremely low $D_c$ and $D_t$ accuracy, suggesting that the model has been severely damaged; in such cases, the MIA results are not meaningful and are thus omitted. 
The red upward arrow ({\color{red} \large$\uparrow$}) denotes that the post-unlearning model achieves a higher MIA success rate than both the model before unlearning and the retrained model.}
    \resizebox{\linewidth}{!}{%
    \begin{tabular}{ccccc|ccc|ccc|ccc|ccc|ccc|ccc|ccc}
    \toprule[1pt]
    \midrule[0.3pt]
    \multirow{1}{*}&\multirow{2}{*}{Method}& \multicolumn{3}{c}{Class 0} & \multicolumn{3}{c}{Class 1} & \multicolumn{3}{c}{Class 2}& \multicolumn{3}{c}{Class 3}& \multicolumn{3}{c}{Class 4} & \multicolumn{3}{c}{Class 5} & \multicolumn{3}{c}{Class 6} & \multicolumn{3}{c}{Class 7} \\  
    \cmidrule(lr){3-26}
      &  &  corr & prob  & entro & corr & prob  & entro &  corr & prob  & entro &  corr & prob  & entro &  corr & prob  & entro &  corr & prob  & entro &  corr & prob  & entro & corr & prob  & entro \\     
    \midrule
    \multirow{7}{*}
    & RT & 0.7883 & 0.7743 & 0.5698 & 0.8562 & 0.8276  & 0.6762 & 0.7093 & 0.7308 & 0.5192 & 0.6967 & 0.6755 & 0.4404 & 0.7637 & 0.7451 & 0.5353 & 0.7170 & 0.7165 & 0.4933 & 0.8118  & 0.7933 & 0.6148 & 0.8061 & 0.8009 & 0.6372     \\  
    \cmidrule(lr){2-26}
    & {IF$\bigtriangledown$} & 0.7577 & 0.7296 & 0.5311 & 0.8854 & 0.8614 & 0.7177 & 0.6634 & 0.6854 & 0.4649 & 0.5962 & 0.6032 & 0.3553 & 0.7437 & 0.7224 & 0.4955 & 0.6674 & 0.6819 & 0.4567 & 0.7785 & 0.7612 & 0.5570 & 0.7933 & 0.7666 &  0.5738    \\  
    & {IF$\bigtriangleup$} & - & - & - & - & - & - & - & -  & - & - & - & - & - & - & - & - & - & - & - & - & - & - & -  &  -    \\  
    \cmidrule(lr){2-26}
    & {FT$\bigtriangledown$} & 0.8691{\color{red} \large$\uparrow$} &0.8434{\color{red} \large$\uparrow$}  & 0.6903{\color{red} \large$\uparrow$} & 0.9424{\color{red} \large$\uparrow$} & 0.9261{\color{red} \large$\uparrow$} & 0.8296{\color{red} \large$\uparrow$} & 0.7493{\color{red} \large$\uparrow$} & 0.7558{\color{red} \large$\uparrow$} & 0.5590{\color{red} \large$\uparrow$} & 0.6992{\color{red} \large$\uparrow$} & 0.7012{\color{red} \large$\uparrow$} & 0.4730{\color{red} \large$\uparrow$} & 0.8160{\color{red} \large$\uparrow$} & 0.8039{\color{red} \large$\uparrow$} & 0.6111{\color{red} \large$\uparrow$} & 0.7412 & 0.7449{\color{red} \large$\uparrow$} & 0.5479{\color{red} \large$\uparrow$} & 0.8659 & 0.8503{\color{red} \large$\uparrow$} & 0.7066{\color{red} \large$\uparrow$} & 0.8506 & 0.8214 & 0.6634     \\  
    & {FT$\bigtriangleup$} & 0.7400 & 0.4081 & 0.0362 & 0.9607 & 0.8962  & 0.7358 & 0.6995 & 0.4039 & 0.0925 & 0.7641 & 0.2281 & 0.0054 & - & - & - & - & - & - & 0.5906 & 0.2651 & 0.0353 & 0.8098 & 0.5446 & 0.2195    \\  
    \cmidrule(lr){2-26}
    & {GA$\bigtriangledown$} & 0.7841 & 0.7567 & 0.5560 & 0.9140 & 0.8814 & 0.7545 & 0.7054 & 0.7106 & 0.4980 & 0.6234 & 0.6150  & 0.3745 & 0.7639 & 0.7350 & 0.5158 &0.6995 & 0.6985 & 0.4812 & 0.8074 & 0.7837 & 0.5980  & 0.8012 & 0.7666 & 0.5720    \\  
    & {GA$\bigtriangleup$} & - & - & - & 0.6175 & 0.6627 & 0.4054 & 0.5288 & 0.5896 & 0.3496 & 0.4540 & 0.5019 & 0.2316 & 0.6498 & 0.6303 & 0.3683 & 0.5683 & 0.6101  & 0.3639 & -  & - & - & - & - & -   \\  
    \midrule
    
     & Before & 0.8409 & 0.8044 & 0.6281 & 0.9397 & 0.9096 & 0.8039 & 0.7419 & 0.7407 & 0.5333 & 0.6772 & 0.6548 & 0.4224 & 0.8135 & 0.7767 & 0.5854 & 0.7429 & 0.7318 & 0.5264 & 0.8671 & 0.8382 & 0.6972 & 0.8607 & 0.8234  & 0.6723     \\  
    \midrule[0.3pt]
    \bottomrule[1pt]
    \end{tabular}
    }
    \label{tab:unlearn_subclass_MIA_on_TUE}
\end{table*}


\begin{table*}[t]
    \centering
    \centering
    \caption{The MIA success rate of unlearned classifiers trained on TUE (class-level). “--” indicates cases where the corresponding model in Table~\ref{tab:unlearn_class_acc_on_TUE} exhibits extremely low $D_{tr}$ accuracy, suggesting that the model has been severely damaged; in such cases, the MIA results are not meaningful and are thus omitted. 
The red upward arrow ({\color{red} \large$\uparrow$}) denotes that the post-unlearning model achieves a higher MIA success rate than both the model before unlearning and the retrained model.}
    \resizebox{\linewidth}{!}{%
    \begin{tabular}{ccccc|ccc|ccc|ccc|ccc|ccc|ccc|ccc}
    \toprule[1pt]
    \midrule[0.3pt]
    \multirow{1}{*}&\multirow{2}{*}{Method}& \multicolumn{3}{c}{Class 0} & \multicolumn{3}{c}{Class 1} & \multicolumn{3}{c}{Class 2}& \multicolumn{3}{c}{Class 3}& \multicolumn{3}{c}{Class 4} & \multicolumn{3}{c}{Class 5} & \multicolumn{3}{c}{Class 6} & \multicolumn{3}{c}{Class 7} \\  
    \cmidrule(lr){3-26}
      &  &  corr & prob  & entro & corr & prob  & entro &  corr & prob  & entro &  corr & prob  & entro &  corr & prob  & entro &  corr & prob  & entro &  corr & prob  & entro & corr & prob  & entro \\     
    \midrule
    \multirow{7}{*}
    & RT & 0 & 0.6315 & 0.3517 & 0 & 0.8620 & 0.7077 & 0 & 0.6460 & 0.3686 &0  & 0.7166 & 0.4706 & 0 & 0.6851 &0.4122  & 0 & 0.6953 & 0.4335 & 0 &0.7264  & 0.4760 & 0 & 0.7253 & 0.4713     \\  
    \cmidrule(lr){2-26}
    & {IF$\bigtriangledown$} & 0 & 0.6391 & 0.3551 & 0.0002 & 0.8171 & 0.6197 & 0 & 0.6471{\color{red} \large$\uparrow$} & 0.3822{\color{red} \large$\uparrow$} & 0 & 0.7017 & 0.4457 & 0 & 0.9315 & 0.7966 & 0 & 0.5451 & 0.2466 & 0.0093 &0.5380  & 0.2351 & 0.0211 & 0.7264 & 0.4624    \\  
    & {IF$\bigtriangleup$} & 0 & 0.6475{\color{red} \large$\uparrow$} &0.3657{\color{red} \large$\uparrow$}  & 0 & 0.6860 & 0.4177 & 0 & 0.7062{\color{red} \large$\uparrow$} & 0.4508{\color{red} \large$\uparrow$} & 0 & 0.7124 & 0.4626 & - & - & - & - & - & -  & 0 & 0.5504 & 0.1940 & - & - & -  \\  
    \cmidrule(lr){2-26}
    & {FT$\bigtriangledown$} & 0 & 0.6071 & 0.2848 & 0 & 0.9542{\color{red} \large$\uparrow$} & 0.8824{\color{red} \large$\uparrow$} & 0 & 0.5524 & 0.2597 & 0 & 0.7608{\color{red} \large$\uparrow$} & 0.5391{\color{red} \large$\uparrow$} &0.1555  & 0.3613 & 0.0871 & 0.0911& 0.4493 & 0.1724 & 0.0622 & 0.4002 & 0.1277 & 0.052 & 0.5655 & 0.2433   \\  
    & {FT$\bigtriangleup$}  & 0 & 0.7706 {\color{red} \large$\uparrow$} &0.5242{\color{red} \large$\uparrow$}  & 0 & 0.9064{\color{red} \large$\uparrow$} & 0.7795{\color{red} \large$\uparrow$} & 0 &  0.5013& 0.2044 & 0 &0.7706{\color{red} \large$\uparrow$}  & 0.5242{\color{red} \large$\uparrow$} &  0& 0.4551 & 0.1304 & 0 & 0.5200 & 0.1955 & 0 & 0.4991 & 0.1911 & 0 & 0.4871 &0.1608  \\ 
    \cmidrule(lr){2-26}
    & {GA$\bigtriangledown$} & 0 & 0.6451 & 0.3640{\color{red} \large$\uparrow$} & 0.0011 & 0.8444 & 0.6602 & 0.0002 & 0.6342 & 0.3400 & 0.0002 & 0.7106 & 0.4660 & - & - & - & 0 & 0.6320 &0.3477  &0.3502  & 0.6262 & 0.3395 & 0.9957 & 0.9737 &  0.8808  \\  
    & {GA$\bigtriangleup$} & 0 & 0.5562& 0.2646 & - &-  & -  & - & - & - & - &   - &-  & - & - & - & - & - &  & - & - & - &-  & - & -   \\  
    \midrule
    
     & Before & 0.0002 & 0.6455 & 0.3633 & 0.0013 & 0.8455 & 0.6602 & 0.0002 & 0.6340 & 0.3404 & 0.0004 & 0.7095 & 0.4671 & 1.0 & 0.9964 & 0.9708 & 1.0 &  0.9984& 0.9826 & 1.0 &0.9991  &0.9828  & 1.0 & 0.9988 & 0.9842   \\  
    \midrule[0.3pt]
    \bottomrule[1pt]
    \end{tabular}
    }
    \label{tab:unlearn_class_MIA_on_TUE}
\end{table*}

\begin{table*}[t]
    \centering
    \centering
    \caption{The MIA success rate of unlearned classifiers trained on PUE (subset-level). “--” indicates cases where the corresponding model in Table~\ref{tab:unlearn_acc_on_PUE} exhibits extremely low $D_c$ and $D_t$ accuracy, suggesting that the model has been severely damaged; in such cases, the MIA results are not meaningful and are thus omitted. 
The red upward arrow ({\color{red} \large$\uparrow$}) denotes that the post-unlearning model achieves a higher MIA success rate than both the model before unlearning and the retrained model.}
    \resizebox{\linewidth}{!}{%
    \begin{tabular}{ccccc|ccc|ccc|ccc|ccc|ccc|ccc|ccc}
    \toprule[1pt]
    \midrule[0.3pt]
    \multirow{1}{*}&\multirow{2}{*}{Method}& \multicolumn{3}{c}{Class 0} & \multicolumn{3}{c}{Class 1} & \multicolumn{3}{c}{Class 2}& \multicolumn{3}{c}{Class 3}& \multicolumn{3}{c}{Class 4} & \multicolumn{3}{c}{Class 5} & \multicolumn{3}{c}{Class 6} & \multicolumn{3}{c}{Class 7} \\  
    \cmidrule(lr){3-26}
      &  &  corr & prob  & entro & corr & prob  & entro &  corr & prob  & entro &  corr & prob  & entro &  corr & prob  & entro &  corr & prob  & entro &  corr & prob  & entro & corr & prob  & entro \\     
    \midrule
    \multirow{7}{*}
    & RT & 0.7755 & 0.7730 & 0.5639 & 0.8846 & 0.8728 &0.7395 & 0.7096 & 0.7187 & 0.5069 & 0.6688 & 0.6659 & 0.4111 & 0.7750 &0.7449  & 0.5362 & 0.7239 &0.7441  & 0.5392 & 0.8283 & 0.8096 & 0.6293 & 0.8397 & 0.8187 & 0.6602     \\  
    \cmidrule(lr){2-26}
    & {IF$\bigtriangledown$} & 0.8017 & 0.7740 & 0.5856 & 0.9039 & 0.8883 &0.7669 & 0.7234 &0.7286  & 0.5133 &0.6651  & 0.6427 & 0.3997 & 0.7409 & 0.7101 & 0.5 & 0.6866 & 0.7128 & 0.5014 & 0.8197 & 0.7938 & 0.6212 & 0.8446 & 0.7886 & 0.6049     \\  
    & {IF$\bigtriangleup$} & - & - &- & - & - &- & 0.3980 & 0.6338 & 0.3644 & - & - & - & - & - & - &- & - & - & - & - & - & - & - &-    \\  
    \cmidrule(lr){2-26}
    & {FT$\bigtriangledown$} & 0.8538{\color{red} \large$\uparrow$} & 0.8355{\color{red} \large$\uparrow$} & 0.6800{\color{red} \large$\uparrow$} & 0.9227 & 0.9061 &0.8019 &0.7370  & 0.7491{\color{red} \large$\uparrow$} & 0.5429 & 0.7407{\color{red} \large$\uparrow$} & 0.7348{\color{red} \large$\uparrow$} & 0.5261{\color{red} \large$\uparrow$} & 0.8079{\color{red} \large$\uparrow$} & 0.7792{\color{red} \large$\uparrow$} & 0.5903{\color{red} \large$\uparrow$} & 0.7007 & 0.7651{\color{red} \large$\uparrow$} & 0.5782{\color{red} \large$\uparrow$} & 0.8795{\color{red} \large$\uparrow$} & 0.8582{\color{red} \large$\uparrow$} & 0.7274{\color{red} \large$\uparrow$} & 0.8592 &0.8360{\color{red} \large$\uparrow$} & 0.6785{\color{red} \large$\uparrow$}     \\  
    & {FT$\bigtriangleup$} & 0.7474 & 0.3395 & 0.0056 & 0.8859 & 0.6918 &0.3735 & - & - & - & - & - & - & - & - & - & -  & - &-  & 0.7088 & 0.4464 & 0.1632 & 0.7320 & 0.4933 &  0.2004  \\  
    \cmidrule(lr){2-26}
    & {GA$\bigtriangledown$} & 0.8123 & 0.7844 & 0.6000 & 0.9096 & 0.8950 & 0.7792&0.7382  & 0.7271 & 0.5274 & 0.6859 & 0.6624 & 0.4261 & 0.7634 & 0.7271 & 0.5182 & 0.7153 & 0.7283 & 0.5234 &0.8355  &0.8024  & 0.6417 & 0.8698 & 0.8135 & 0.6412    \\  
    & {GA$\bigtriangleup$} & 0.4945 & 0.6108 & 0.3451 & 0.6661 & 0.7291 & 0.5037 & 0.6864 & 0.7012 & 0.4780 &- &-  &-  & 0.6711 & 0.6587 & 0.4237 & 0.6328 & 0.6767 & 0.4434 & - & - & - &0.8597  & 0.7982 & 0.6192    \\  
    \midrule
    
     & Before & 0.8414 & 0.8096 & 0.6335 & 0.9271 & 0.9106 &0.8106 & 0.7545 & 0.7437 & 0.5444 & 0.7111 & 0.6883 & 0.4565 & 0.7844 & 0.7427 & 0.5412 & 0.7353 & 0.7407 & 0.5441 & 0.8686 & 0.8355 & 0.6879 & 0.8745 & 0.8237 &  0.6518    \\  
    \midrule[0.3pt]
    \bottomrule[1pt]
    \end{tabular}
    }
    \label{tab:unlearn_subclass_MIA_on_PUE}
\end{table*}


\begin{table*}[t]
    \centering
    \centering
    \caption{The MIA success rate of unlearned classifiers trained on PUE (class-level).  “--” indicates cases where the corresponding model in Table~\ref{tab:unlearn_class_acc_on_PUE} exhibits extremely low $D_{tr}$ accuracy, suggesting that the model has been severely damaged; in such cases, the MIA results are not meaningful and are thus omitted. 
The red upward arrow ({\color{red} \large$\uparrow$}) denotes that the post-unlearning model achieves a higher MIA success rate than both the model before unlearning and the retrained model.}
    \resizebox{\linewidth}{!}{%
    \begin{tabular}{ccccc|ccc|ccc|ccc|ccc|ccc|ccc|ccc}
    \toprule[1pt]
    \midrule[0.3pt]
    \multirow{1}{*}&\multirow{2}{*}{Method}& \multicolumn{3}{c}{Class 0} & \multicolumn{3}{c}{Class 1} & \multicolumn{3}{c}{Class 2}& \multicolumn{3}{c}{Class 3}& \multicolumn{3}{c}{Class 4} & \multicolumn{3}{c}{Class 5} & \multicolumn{3}{c}{Class 6} & \multicolumn{3}{c}{Class 7} \\  
    \cmidrule(lr){3-26}
      &  &  corr & prob  & entro & corr & prob  & entro &  corr & prob  & entro &  corr & prob  & entro &  corr & prob  & entro &  corr & prob  & entro &  corr & prob  & entro & corr & prob  & entro \\     
    \midrule
    \multirow{7}{*}
    & RT & 0 & 0.3266 & 0.0511& 0 & 0.8644 & 0.7104  & 0 & 0.6251 & 0.3393 & 0 & 0.7097 &  0.4526  & 0  & 0.6891 & 0.4240 & 0 & 0.7082 & 0.4382 & 0 & 0.7013 &0.4562 & 0.0&0.7160 &0.4548    \\  
    \cmidrule(lr){2-26}
    & {IF$\bigtriangledown$} & 0.0042 & 0.5920{\color{red} \large$\uparrow$} & 0.2984{\color{red} \large$\uparrow$} & 0.0020 & 0.7897 &0.5877  & 0.0002 & 0.6548{\color{red} \large$\uparrow$} & 0.3906{\color{red} \large$\uparrow$} & 0.0002 & 0.7615{\color{red} \large$\uparrow$} & 0.5320{\color{red} \large$\uparrow$} & 0.0006 & 0.7513 & 0.4877 & 0 &0.5835  & 0.2486 & 0.0015 & 0.5057 & 0.2128 & 0 & 0.5626 &  0.2408   \\  
    & {IF$\bigtriangleup$} & 0 & 0.6406{\color{red} \large$\uparrow$} & 0.3708{\color{red} \large$\uparrow$} &  0&0.6324  & 0.3408 & 0 & 0.7782{\color{red} \large$\uparrow$} & 0.5648{\color{red} \large$\uparrow$} & 0 & 0.8475{\color{red} \large$\uparrow$} & 0.6891{\color{red} \large$\uparrow$} & - & - & - & - & - & - & 0 & 0.7048 & 0.3564 & 0 &0.4597  & 0.1168  \\  
    \cmidrule(lr){2-26}
    & {FT$\bigtriangledown$} & 0 & 0.6195{\color{red} \large$\uparrow$} & 0.3140{\color{red} \large$\uparrow$}  & 0 & 0.7868 & 0.5597 & 0 & 0.5340 & 0.2471 &  0& 0.7362{\color{red} \large$\uparrow$} & 0.4962{\color{red} \large$\uparrow$} & 0.0751 & 0.4760& 0.1591  &  0.0493& 0.5255 & 0.2200 & 0.0566 & 0.3535 & 0.0893 & 0.0695 & 0.5246 & 0.2191 \\  
    & {FT$\bigtriangleup$}  &0  & 0.5237 & 0.2095 & 0 &0.8953{\color{red} \large$\uparrow$}  & 0.7702{\color{red} \large$\uparrow$} & 0 & 0.4771 & 0.1813 &  0& 0.6042 & 0.3048 & 0 & 0.5068 & 0.1842 & 0 & 0.5562 &0.2408 &0 & 0.5802 & 0.2626 & 0 & 0.5573 & 0.2026 \\ 
    \cmidrule(lr){2-26}
    & {GA$\bigtriangledown$} & 0.0568 & 0.5760 & 0.2815{\color{red} \large$\uparrow$} & 0.068 &0.7857  & 0.5502 & 0.0006 &0.6000  &0.3202  & 0.0013 & 0.7004 & 0.4435 & 0.2657 & 0.6604 & 0.3766 & - & - & - & 0 & 0.7526 & 0.5028 & - & - &  -  \\  
    & {GA$\bigtriangleup$}  & 0.0235 & 0.5780{\color{red} \large$\uparrow$} & 0.2895{\color{red} \large$\uparrow$} & - &-  &-  & 0.0004 & 0.6028 & 0.3260 & 0 & 0.4037 & 0.1384 & - & - & - & - & - & - & - & - & - &-  & - & -   \\  
    \midrule
    
     & Before & 0.0666 & 0.5766 & 0.2795 & 0.0844 & 0.7815 & 0.5404 & 0.0008 & 0.5984 & 0.3197 & 0.0013 & 0.7024 & 0.4453 & 1.0 & 0.9984 & 0.9828 &1.0  & 0.9993 & 0.9817 & 1.0& 0.9982 & 0.9853 & 1.0 &0.9977  & 0.9828   \\  
    \midrule[0.3pt]
    \bottomrule[1pt]
    \end{tabular}
    }
    \label{tab:unlearn_class_MIA_on_PUE}
\end{table*}

\begin{table*}[t]
    \centering
    \centering
    \caption{The MIA success rate of unlearned classifiers trained on OPS (subset-level). “--” indicates cases where the corresponding model in Table~\ref{tab:unlearn_acc_on_OPS} exhibits extremely low $D_c$ and $D_t$ accuracy, suggesting that the model has been severely damaged; in such cases, the MIA results are not meaningful and are thus omitted. 
The red upward arrow ({\color{red} \large$\uparrow$}) denotes that the post-unlearning model achieves a higher MIA success rate than both the model before unlearning and the retrained model.}
    \resizebox{\linewidth}{!}{%
    \begin{tabular}{ccccc|ccc|ccc|ccc|ccc|ccc|ccc|ccc}
    \toprule[1pt]
    \midrule[0.3pt]
    \multirow{1}{*}&\multirow{2}{*}{Method}& \multicolumn{3}{c}{Class 0} & \multicolumn{3}{c}{Class 1} & \multicolumn{3}{c}{Class 2}& \multicolumn{3}{c}{Class 3}& \multicolumn{3}{c}{Class 4} & \multicolumn{3}{c}{Class 5} & \multicolumn{3}{c}{Class 6} & \multicolumn{3}{c}{Class 7} \\  
    \cmidrule(lr){3-26}
      &  &  corr & prob  & entro & corr & prob  & entro &  corr & prob  & entro &  corr & prob  & entro &  corr & prob  & entro &  corr & prob  & entro &  corr & prob  & entro & corr & prob  & entro \\     
    \midrule
    \multirow{7}{*}
    & RT & 0.7114 & 0.7190 & 0.4978 & 0.8486 & 0.8244 & 0.6454 & 0.6889 & 0.6938 & 0.4541 & 0.5953 & 0.5852 & 0.3202 & 0.7131 & 0.6958 & 0.4588 & 0.6607 & 0.6736 & 0.4279 & 0.7958 & 0.7704 & 0.5758 & 0.7951 & 0.7773 & 0.6069 \\
    \cmidrule(lr){2-26}
    & {IF$\bigtriangledown$} & 0.7706 & 0.7343 & 0.5188 & 0.8230 & 0.7931 & 0.6116 & 0.6272 & 0.6442 & 0.3951  & 0.5879 & 0.5647 & 0.2923 & 0.6810 & 0.6743 & 0.4311 & 0.6190 & 0.6173 & 0.3635 & 0.8010 & 0.7862 & 0.5849 & 0.7867 & 0.7407 & 0.5533 \\
    & {IF$\bigtriangleup$} & 0.5235 & 0.5891 & 0.3141 & - & - & - & 0.3393 & 0.5111 & 0.2452  & 0.2200 & 0.4509 & 0.1793 & 0.4020 & 0.5388 & 0.2583 & - & - & - & 0.6667 & 0.6837 & 0.4427 & 0.4264 & 0.5109 & 0.2484     \\  
    \cmidrule(lr){2-26}
    & {FT$\bigtriangledown$} & 0.7612 & 0.7459 & 0.5523 & 0.9637& 0.9314 & 0.8373 & 0.6704& 0.6857 & 0.4484  & 0.5995 & 0.6057{\color{red} \large$\uparrow$} & 0.3432& 0.7089 & 0.7012{\color{red} \large$\uparrow$}& 0.4743{\color{red} \large$\uparrow$}& 0.7049{\color{red} \large$\uparrow$} & 0.7054{\color{red} \large$\uparrow$} & 0.4795{\color{red} \large$\uparrow$} & 0.8259{\color{red} \large$\uparrow$}& 0.8188{\color{red} \large$\uparrow$} & 0.6469{\color{red} \large$\uparrow$} & 0.7840& 0.7627 &0.5583     \\
    
    & {FT$\bigtriangleup$}  & 0.7948& 0.5635 & 0.1393 & 0.8960& 0.7894 & 0.5783 & -& - & -  & 0.6212 & 0.2990 & 0.0459& 0.6528 & 0.3612 & 0.0531& - & - & - & -&- &- & 0.8205{\color{red} \large$\uparrow$}& 0.5906 &0.2469     \\
    \cmidrule(lr){2-26}
    & {GA$\bigtriangledown$} & 0.7768 & 0.7402 & 0.5227 & - & - & - & 0.6343 & 0.6439  & 0.4007 & 0.5837 & 0.5610 & 0.2881 & 0.6862 & 0.6781 & 0.4353 & 0.6230 & 0.6200 & 0.3657 & 0.8096 & 0.7879 & 0.5962 & 0.7874 & 0.7402 & 0.5516 \\
    
    & {GA$\bigtriangleup$}  & 0.5457 & 0.5914 & 0.3232 & - & - & - & 0.4484 & 0.5402 & 0.2807  & 0.2042 & 0.4333 & 0.1526 & 0.5635 & 0.6141 & 0.3326 & - & - & - & 0.7921 & 0.7716 & 0.5644 & 0.4840 & 0.5165 & 0.2546 \\ 
    \midrule
    
    & Before & 0.8037 & 0.7610 & 0.5523 & 0.9677 & 0.9358 & 0.8560 & 0.6647 & 0.6696 & 0.4281 & 0.6252 & 0.5909 & 0.3220 & 0.7064 & 0.6931 & 0.4612 & 0.6960 & 0.6672 & 0.4301 & 0.8165 & 0.7946 & 0.6080 & 0.8156 & 0.7731 & 0.5958 \\ 
    \midrule[0.3pt]
    \bottomrule[1pt]
    \end{tabular}
    }
    \label{tab:unlearn_subclass_MIA_on_OPS}
\end{table*}

\begin{table*}[t]
    \centering
    \centering
    \caption{The MIA success rate of unlearned classifiers trained on OPS (class-level).  “--” indicates cases where the corresponding model in Table~\ref{tab:unlearn_class_acc_on_OPS} exhibits extremely low $D_{tr}$ accuracy, suggesting that the model has been severely damaged; in such cases, the MIA results are not meaningful and are thus omitted. 
The red upward arrow ({\color{red} \large$\uparrow$}) denotes that the post-unlearning model achieves a higher MIA success rate than both the model before unlearning and the retrained model.}
    \resizebox{\linewidth}{!}{%
    \begin{tabular}{ccccc|ccc|ccc|ccc|ccc|ccc|ccc|ccc}
    \toprule[1pt]
    \midrule[0.3pt]
    \multirow{1}{*}&\multirow{2}{*}{Method}& \multicolumn{3}{c}{Class 0} & \multicolumn{3}{c}{Class 1} & \multicolumn{3}{c}{Class 2}& \multicolumn{3}{c}{Class 3}& \multicolumn{3}{c}{Class 4} & \multicolumn{3}{c}{Class 5} & \multicolumn{3}{c}{Class 6} & \multicolumn{3}{c}{Class 7} \\  
    \cmidrule(lr){3-26}
      &  &  corr & prob  & entro & corr & prob  & entro &  corr & prob  & entro &  corr & prob  & entro &  corr & prob  & entro &  corr & prob  & entro &  corr & prob  & entro & corr & prob  & entro \\     
    \midrule
    \multirow{7}{*}
    & RT & 0& 0.5982 & 0.3242 & 0& 0.8573 & 0.7020 & 0& 0.6102 & 0.3284  & 0 & 0.7084 & 0.4582& 0 & 0.6767 & 0.3929& 0 & 0.6622 & 0.3831 & 0& 0.7124 & 0.4762 & 0& 0.7231 & 0.4593     \\
    \cmidrule(lr){2-26}
    & {IF$\bigtriangledown$} & 0.0056 & 0.6091 & 0.3262 & 0.0022 & 0.8391 & 0.6556 & 0.0011 & 0.6622{\color{red} \large$\uparrow$} & 0.3922{\color{red} \large$\uparrow$}  & 0.0033 & 0.7496{\color{red} \large$\uparrow$} & 0.5356{\color{red} \large$\uparrow$} & 0.0258 & 0.6204 & 0.3133 & 0 & 0.6784 & 0.3744 & 0.0064 & 0.5449 & 0.2511 & 0 & 0.6642 & 0.3509     \\
    
    & {IF$\bigtriangleup$} & 0 & 0.5833 & 0.2858 & 0 & 0.7124 & 0.4702 & 0 & 0.7971{\color{red} \large$\uparrow$} & 0.6058{\color{red} \large$\uparrow$}  & 0 & 0.8504{\color{red} \large$\uparrow$} & 0.6831{\color{red} \large$\uparrow$} & - & - & - & - & - & - & 0 & 0.4887 & 0.1449 & - & - & -     \\
    \cmidrule(lr){2-26}
    & {FT$\bigtriangledown$} & 0 & 0.5513 & 0.2571 & 0 & 0.9218{\color{red} \large$\uparrow$} & 0.8151{\color{red} \large$\uparrow$} & 0 & 0.5602 & 0.2544  & 0 & 0.6310 & 0.3600 & 0.1233 & 0.4389 & 0.1602 & 0.0398 & 0.4587 & 0.1758 & 0.0604 & 0.4391 & 0.1731 & 0.1382 & 0.4122 & 0.1267     \\
    
    & {FT$\bigtriangleup$}  & 0 & 0.5520 & 0.2522 & 0 & 0.8680{\color{red} \large$\uparrow$} & 0.7193{\color{red} \large$\uparrow$} & 0 & 0.5482 & 0.2580 & 0 & 0.6687 & 0.3593 & 0 & 0.5220 & 0.2164 & 0 & 0.5073 & 0.1976 & 0 & 0.5953 & 0.2880 & 0 & 0.5376 & 0.2176  \\
    \cmidrule(lr){2-26}
    & {GA$\bigtriangledown$} & 0.0224 & 0.6262 & 0.3313 & 0.0538 & 0.7887 & 0.5293 & 0.0260 & 0.6080 & 0.3196  & 0.0273 & 0.6873 & 0.4367 & 0.0587 & 0.7289 & 0.4704 & - & - & - & 0 & 0.8253 & 0.6060 & 0 & 1.0000 & 1.0000     \\
    
    & {GA$\bigtriangleup$}  & 0.0044 & 0.6336{\color{red} \large$\uparrow$} & 0.3511{\color{red} \large$\uparrow$} & 0 & 0.9316{\color{red} \large$\uparrow$} & 0.8569{\color{red} \large$\uparrow$} & 0.0107 & 0.6142{\color{red} \large$\uparrow$} & 0.3344{\color{red} \large$\uparrow$}  & 0 & 0.3796 & 0.1120 & - & - & - & - & - & - & - & - & - & - & - & -     \\
    \midrule
    
    & Before & 0.0271 & 0.6249 & 0.3271 & 0.0627 & 0.7731 & 0.5020 & 0.0302 & 0.6053 & 0.3162 & 0.0309 & 0.6864 & 0.4389 & 1.0 & 0.9987 & 0.9862 & 1.0 & 0.9989 & 0.9731 & 0.9998 & 0.9987 & 0.9816 & 0.9998 & 0.9993 & 0.9842 \\
    \midrule[0.3pt]
    \bottomrule[1pt]
    \end{tabular}
    }
    \label{tab:unlearn_class_MIA_on_OPS}
\end{table*}

\section{Bridging \MethodName and Unlearning}\label{sec:interplay_theory}
A unified view of \MethodName and unlearning can be approached through analysis in the hypothesis (parameter) space. 
The learnability of \MethodName can be interpreted in terms of where models converge in this space.
According to Equation~\ref{eq:learnability}, a high probability of converging to parameters $\theta$ that yield low performance on $\sD_t$ corresponds to low learnability. 
Conversely, unlearning can be viewed as reducing $\gM(\sD_f)$ while maintaining $\gM(\sD\setminus\sD_f)$ (or $\gM(\sD_u\setminus\sD_f)$) with high probability. 
This connection suggests that analysis in the hypothesis space may be key to bridging \MethodName and unlearning.

\subsection{From Learnability to Unlearning}
In this section, we show how this bound characterizes intrinsic model behaviors during the unlearning process.
The bounds are restated as follows.
\lcert*
\noindent
The bounds characterize the collective change in model performance for models in the neighborhood of a model with weights $\hat{\theta}$ when perturbed by $\upsilon$.
Suppose $\hat{\theta}$ denotes the weights of the model prior to unlearning.
Unlearning involves adjusting the weights from $\hat{\theta}$ to a new set of parameters $\theta^*$ in the vicinity.
Successful unlearning requires that $A_{\sD_f}({\theta^*})\rightarrow 0$ and $A_{\sD_r}({\theta^*})$ remains largely unaffected.
Translated to the learnability bound, this means that $\hat{\theta}$ having a higher bound $\sup\ \{t\ |\ \Pr_{\kappa}[A_{\sD_r}(\hat{\theta}+\kappa) \leq t] \leq \underline{q}\}$ but lower $\inf\ \{t\ |\ \Pr_{\kappa}[A_{\sD_f}(\hat{\theta}+\kappa) \leq t] \geq \overline{q}\}$ are \emph{more unlearning-friendly}.
In this light, since models trained on different \MethodName variants may exhibit varying levels of learnability as quantified by the learnability bound, their synergy in supporting unlearning may also differ.

\subsection{From Unlearning to Learnability}
$(\epsilon, \zeta)$-certified unlearning supplies unlearning guarantees by stochastic postprocessing via noisy finetuning on $\sD_r$.
The core idea is that, after $T$ steps of Gaussian-perturbed and clipped gradient updates based on $\sD_r$, the model parameter distribution will be $(\epsilon, \zeta)$-indistinguishable to the distribution by training on $\sD_r$ from scratch.
Specifically, a recent analysis states that
\begin{restatable}
[\textit{$(\epsilon, \zeta)$-certified unlearning guarantee}]{theorem}{unlearn}\label{theorem:unlearning_DP}
Given the clipped initial parameter $\theta_0 = \prod_{C_0}(\hat{\theta})$, let the training update at step $t$ be $\theta_{t+1} = \theta_t -\gamma(\prod_{C_1}(g_t)) + \phi$, where $g_t$ is the set of gradients computed on $\sD_r$, $\phi\sim\N(0, \sigma^2I)$, and $\prod_{C_0}$ and $\prod_{C_1}$ denote clipping projections with radii $C_0 > 0$ and $C_1 > 0$, respectively.
Consider $T \geq 1$ update steps with $\gamma, \sigma > 0$. 
Suppose the privacy budgets are $\zeta \in (0,1)$ and $\epsilon \in \big(0, 3\log(1/\zeta)\big)$. 
Then an $(\epsilon,\zeta)$-unlearning guarantee is achieved if
\begin{equation}\small
    \sigma^2 = \frac{9\log(1/\zeta)}{\epsilon^2T} (C_0 + C_1 \gamma T)^2.
\end{equation}
\end{restatable}
\noindent
Please refer to the proof in the original paper~\cite{koloskova2025certified}.

To tie this to the learnability of \MethodName, consider a case where a \MethodName-protected dataset $\sD_u$ should be unlearned such that $\Pr[\gU(\Gamma(\sD_u \cup \sD^*), \sD_u \cup \sD^*, \sD_u) \in S]$ and $\Pr[\Gamma(\sD^*) \in S]$, with $\sD^*$ denoting a clean dataset, become indistinguishable. 
By fixing $\sigma$, $\gamma$, $C_0$, $\epsilon$, and $\zeta$, we obtain the proportional relationship $T \propto \tfrac{1}{C_1}$.
This implies that a smaller $C_1$ requires more training steps. 
At the same time, to avoid excessive performance degradation, $C_1$ should approximately match the gradient magnitude. 
If $\sD_u$ is less robust in the parameter space, the unlearning gradients $g_t$ will be large, allowing the model to forget $\sD_u$ with fewer updates.
However, large local gradients $g_t$ also make recovery attacks more effective. 
Conversely, unlearning tests with varying $T$ values can serve as a probe of the parametric robustness of \MethodName.

\end{document}